\documentclass{article}

\PassOptionsToPackage{hypertexnames=false}{hyperref}  %
\usepackage{parskip}

\usepackage[colorlinks=true, linkcolor=blue!70!black, citecolor=blue!70!black,urlcolor=black,breaklinks=true]{hyperref}
\usepackage{microtype}
\usepackage{hhline}

\makeatletter
\newcommand{\neutralize}[1]{\expandafter\let\csname c@#1\endcsname\count@}
\makeatother

\usepackage{algorithm}

\usepackage{amsthm}
\usepackage{mathtools}
\usepackage{amsmath}
\usepackage{bbm}
\usepackage{amsfonts}
\usepackage{amssymb}

\usepackage{xpatch}

\newtheorem*{theorem*}{Theorem}
\newtheorem{theorem}{Theorem}
\newtheorem{lemma}{Lemma}
\newtheorem{proposition}[theorem]{Proposition}
\newtheorem{assumption}{Assumption}
\newtheorem{corollary}{Corollary}

\theoremstyle{definition}

\newtheorem{example}{Example}

\theoremstyle{remark}

\AtBeginEnvironment{remark}{%
  \pushQED{\qed}%
}
\AtEndEnvironment{remark}{\popQED\endexample}

\makeatletter
  \renewenvironment{proof}[1][Proof]%
  {%
   \par\noindent{\bfseries\upshape {#1.}\ }%
  }%
  {\qed\newline}
  \makeatother

\xpatchcmd{\proof}{\itshape}{\normalfont\proofnameformat}{}{}
\newcommand{\proofnameformat}{\bfseries}

\usepackage[nameinlink,capitalize]{cleveref}

\crefformat{equation}{#2(#1)#3}
\Crefformat{equation}{#2(#1)#3}
\usepackage{mdframed}
\usepackage{lipsum}

\Crefformat{figure}{#2Figure #1#3}
\Crefname{assumption}{Assumption}{Assumptions}
\Crefformat{assumption}{#2Assumption #1#3}
\usepackage[customcolors]{hf-tikz}
\usepackage{crossreftools}
\usepackage{xparse}

\ExplSyntaxOn
\DeclareDocumentCommand{\XDeclarePairedDelimiter}{mm}
 {
  \__egreg_delimiter_clear_keys: %
  \keys_set:nn { egreg/delimiters } { #2 }
  \use:x %
   {
    \exp_not:n {\NewDocumentCommand{#1}{sO{}m} }
     {
      \exp_not:n { \IfBooleanTF{##1} }
       {
        \exp_not:N \egreg_paired_delimiter_expand:nnnn
         { \exp_not:V \l_egreg_delimiter_left_tl }
         { \exp_not:V \l_egreg_delimiter_right_tl }
         { \exp_not:n { ##3 } }
         { \exp_not:V \l_egreg_delimiter_subscript_tl }
       }
       {
        \exp_not:N \egreg_paired_delimiter_fixed:nnnnn 
         { \exp_not:n { ##2 } }
         { \exp_not:V \l_egreg_delimiter_left_tl }
         { \exp_not:V \l_egreg_delimiter_right_tl }
         { \exp_not:n { ##3 } }
         { \exp_not:V \l_egreg_delimiter_subscript_tl }
       }
     }
   }
 }

\keys_define:nn { egreg/delimiters }
 {
  left      .tl_set:N = \l_egreg_delimiter_left_tl,
  right     .tl_set:N = \l_egreg_delimiter_right_tl,
  subscript .tl_set:N = \l_egreg_delimiter_subscript_tl,
 }

\cs_new_protected:Npn \__egreg_delimiter_clear_keys:
 {
  \keys_set:nn { egreg/delimiters } { left=.,right=.,subscript={} }
 }

\cs_new_protected:Npn \egreg_paired_delimiter_expand:nnnn #1 #2 #3 #4
 {%
  \mathopen{}
  \mathclose\c_group_begin_token
   \left#1
   #3
   \group_insert_after:N \c_group_end_token
   \right#2
   \tl_if_empty:nF {#4} { \c_math_subscript_token {#4} }
 }
\cs_new_protected:Npn \egreg_paired_delimiter_fixed:nnnnn #1 #2 #3 #4 #5
 {
  \mathopen{#1#2}#4\mathclose{#1#3}
  \tl_if_empty:nF {#5} { \c_math_subscript_token {#5} }
 }
\ExplSyntaxOff

\XDeclarePairedDelimiter{\supnorm}{
  left=\lVert,
  right=\rVert,
  subscript=\infty
  }

\usepackage[utf8]{inputenc} %
\usepackage[T1]{fontenc}    %
\usepackage{url}            %
\usepackage{booktabs}       %
\usepackage{amsfonts}       %
\usepackage{nicefrac}       %
\usepackage{microtype}      %
\usepackage{authblk}
\usepackage{tocloft}            %

\usepackage{enumitem}

\usepackage{breakcites}
\usepackage{dsfont}
\usepackage{etoolbox}
\usepackage{comment}
\newtoggle{draft}
\togglefalse{draft}

\usepackage{color-edits}

\usepackage{mathrsfs}

\usepackage{algorithm}
\usepackage{verbatim}
\usepackage[noend]{algpseudocode}

\usepackage{multicol}

\usepackage{colortbl}
\usepackage{bbm}
\usepackage{setspace}

\usepackage{transparent}

\usepackage{inconsolata}
\usepackage[scaled=.90]{helvet}
\usepackage{xspace}

\usepackage{pifont}

\usepackage{tikz-cd}

    \newcommand{\E}{{\mathbb E}}
    
    \newcommand{\R}{{\mathbb R}}

    \newcommand{\PP}{\mathbb{P}}

    \newcounter{rcnt}[section]

    \def\argmin{\mathop{\rm argmin}}

\newcommand\blfootnote[1]{
    \begingroup
    \renewcommand\thefootnote{}\footnote{#1}
    \addtocounter{footnote}{-1}
    \endgroup
}

    \def\ddefloop#1{\ifx\ddefloop#1\else\ddef{#1}\expandafter\ddefloop\fi}
    \def\ddef#1{\expandafter\def\csname c#1\endcsname{\ensuremath{\mathcal{#1}}}}
    \ddefloop ABCDEFGHIJKLMNOPQRSTUVWXYZ\ddefloop

 \usepackage[main, final]{neurips_2026}

\usepackage[utf8]{inputenc} %
\usepackage[T1]{fontenc}    %
\usepackage{hyperref}       %
\usepackage{url}            %
\usepackage{booktabs}       %
\usepackage{amsfonts}       %
\usepackage{nicefrac}       %
\usepackage{microtype}      %
\usepackage{xcolor}         %
\usepackage{subcaption}
\usepackage{graphicx}
\usepackage{wrapfig}
\usepackage{amsmath}
\usepackage{tikz}
\usepackage{wrapfig}
\usetikzlibrary{shapes.geometric, positioning}
\usepackage{listings}
\usepackage{xcolor}

\AtBeginDocument{%
  \setlength{\abovedisplayskip}{5pt plus 2pt minus 2pt}%
  \setlength{\belowdisplayskip}{5pt plus 2pt minus 2pt}%
  \setlength{\abovedisplayshortskip}{2pt plus 2pt}%
  \setlength{\belowdisplayshortskip}{3pt plus 2pt minus 2pt}%
}

\IfFileExists{fontawesome5.sty}{%
  \usepackage{fontawesome5}%
}{%
  \providecommand{\faGithub}{\textsf{[GitHub]}}%
}

\title{Local Sparsity Enables \\ Unsupervised LLM Safety Detection}

\author{%
  \bf Xin Chen \quad Gil Kur\thanks{Equal contribution. Correspondence to: \texttt{xin.chen@inf.ethz.ch}} \quad  Alexander Shevchenko\textsuperscript{*}  \quad Andreas Krause \\
  \rule{0pt}{1.5em} ETH Z\"urich
}

\begin{document}

\hypersetup{linkcolor=black}
\maketitle
\hypersetup{linkcolor=blue!70!black}

\blfootnote{\faGithub~Code is available at
\url{https://github.com/lasgroup/unsupervised-llm-safety}.}

\begin{abstract}
Deployment-time safety methods for large language models (LLMs) are predominantly supervised and assume access to unsafe training data. Nevertheless, new attacks and harm categories regularly arise, not captured by models trained in such a supervised fashion. An alternative approach is to view this problem through the lens of anomaly detection, namely, to rely solely on modeling safe data and flagging out-of-distribution inputs. However, LLM activations lie in a high-dimensional space, raising concerns about whether anomaly detection is statistically feasible. We show that, under the linear representation hypothesis (LRH), there may indeed be hope. In the LRH concept space, which is typically recovered via a sparse autoencoder (SAE), nearby points share a small common active support. Using this local sparsity insight, we propose a framework for locally masked SAE-based anomaly detection, supported by theoretical justifications.  We validate it on various architectures and datasets, including both capability-testing datasets and safety-specific datasets. Finally, when we allow algorithms to use 1\% out-of-distribution data for calibration, locally sparse methods achieve near-optimal performance, demonstrating their ability to capture meaningful safety information while using only 1-2\% of SAE neurons for computation.
\end{abstract}

\newcommand{\ncluster}{\mathcal{N}_c}
\newcommand{\sparsity}{k}
\newcommand{\as}[1]{\textcolor{blue}{\textbf{[Alex]} #1}}

\section{Introduction}
\label{sec:introduction}

Supervised safety classifiers for Large Language Models (LLMs) \citep{sharma2025constitutional,
bai2022training, chen2025learning, zou2024improving} assume access to labeled unsafe examples. In practice, this turns out to be a significant limitation since unsafe distribution may change drastically during deployment. Examples of such distribution shift include: the emergence of new jailbreaks \citep{andriushchenko2024jailbreaking, nasr2025attacker}, harmful behaviors absent from the labeled set \citep{turner2025model}, and inputs from domains the model never encountered before \citep{oh2025understanding}. In this view, a supervised procedure fit to a fixed unsafe distribution might generalize poorly beyond it.

To address this limitation, we take an unsupervised view and frame the problem as anomaly detection \citep{scholkopf1999support, perera2021one}. Namely, given safe data $\{x_i\}_{i=1}^{n} {\sim} \mathbb{P}$, we aim to construct a score function  $s : \mathcal{X} \to \mathbb{R}$ that takes large values on inputs atypical under $\mathbb{P}$. Crucially, this construction does not require unsafe data, and is therefore not biased toward any particular harm category. Despite its conceptual appeal, this approach remains underexplored for modern LLMs. The reason is that LLM embeddings are high-dimensional, and unsupervised anomaly detection methods are known to perform poorly in this setting \citep{aggarwal2001surprising, zimek2012survey}.

\begin{figure*}[t]
  \centering
  \begin{tikzpicture}[
      font=\small,
      cluster/.style={draw, ellipse, line width=0.5pt, fill opacity=0.55, draw opacity=0.9},
      bar/.style={draw=black!20, fill=black!5, line width=0.4pt, rounded corners=1pt},
      tick/.style={line width=1.4pt},
      leader/.style={dashed, line width=0.4pt, opacity=0.7, dash pattern=on 2pt off 3pt},
      shared/.style={dashed, draw=black!50, line width=0.4pt, dash pattern=on 2pt off 2pt},
      pt/.style={circle, fill, inner sep=0pt, minimum size=2pt},
      bgpt/.style={circle, fill=black!45, opacity=0.35, inner sep=0pt, minimum size=1.4pt},
    ]

    \definecolor{purpA}{HTML}{CECBF6}
    \definecolor{purpAd}{HTML}{534AB7}
    \definecolor{purpAt}{HTML}{3C3489}
    \definecolor{tealB}{HTML}{9FE1CB}
    \definecolor{tealBd}{HTML}{1D9E75}
    \definecolor{tealBt}{HTML}{085041}
    \definecolor{coralC}{HTML}{F5C4B3}
    \definecolor{coralCd}{HTML}{D85A30}
    \definecolor{coralCt}{HTML}{712B13}
    \definecolor{barbg}{HTML}{F1EFE8}
    \definecolor{barfr}{HTML}{D3D1C7}
    \definecolor{unionfr}{HTML}{888780}
    \definecolor{uniontk}{HTML}{444441}
    \definecolor{axisgr}{HTML}{5F5E5A}

    \node[anchor=center, font=\small\bfseries] at (5.5, 4.10)
      {SAE feature space ($d = 16{,}384$)};

    \draw[barfr, line width=0.4pt] (0,0.50) -- (0,3.50);
    \draw[barfr, line width=0.4pt] (0,0.50) -- (4.0,0.50);
    \node[anchor=east, text=black!50, font=\scriptsize] at (-0.05,3.30) {PC$_2$};
    \node[anchor=east, text=black!50, font=\scriptsize] at (4.0,0.30) {PC$_1$};

    \foreach \x/\y in {%
      3.20/2.95, 3.32/2.80, 3.44/2.90, 3.56/2.75, 3.26/2.60, 3.46/2.55, 3.60/2.85,%
      0.48/0.85, 0.62/0.70, 0.68/0.85, 0.54/0.62, 0.74/0.58,%
      3.52/0.95, 3.68/0.78, 3.76/0.62, 3.60/0.58,%
      0.30/2.20, 0.44/2.05, 0.56/2.18, 0.36/1.92%
    }{\node[bgpt] at (\x,\y) {};}

    \node[cluster, fill=purpA, draw=purpAd, minimum width=1.40cm, minimum height=0.85cm]
      (clA) at (1.30,2.70) {};
    \foreach \x/\y in {1.00/2.80, 1.40/2.90, 1.66/2.75, 1.10/2.60, 1.50/2.50, 1.22/2.40, 1.70/2.55}
      {\node[pt, fill=purpAt] at (\x,\y) {};}
    \node[text=purpAt, font=\small\bfseries, anchor=west] at (1.75,3.15) {Cluster A};

    \node[cluster, fill=tealB, draw=tealBd, minimum width=1.30cm, minimum height=0.85cm]
      (clB) at (2.40,1.95) {};
    \foreach \x/\y in {2.10/2.05, 2.45/2.15, 2.70/2.00, 2.18/1.85, 2.52/1.75, 2.78/1.85, 2.30/1.65}
      {\node[pt, fill=tealBt] at (\x,\y) {};}
    \node[text=tealBt, font=\small\bfseries, anchor=west] at (3.05,1.95) {Cluster B};

    \node[cluster, fill=coralC, draw=coralCd, minimum width=1.40cm, minimum height=0.65cm]
      (clC) at (1.55,1.10) {};
    \foreach \x/\y in {1.20/1.20, 1.55/1.30, 1.90/1.18, 1.30/1.05, 1.70/0.95, 1.45/0.88, 2.00/1.00}
      {\node[pt, fill=coralCt] at (\x,\y) {};}
    \node[text=coralCt, font=\small\bfseries, anchor=west] at (2.40,1.10) {Cluster C};

    \def\barX{5.10}
    \def\barW{4.40}

    \newcommand{\rowbar}[3]{%
      \draw[bar] (\barX,#1) rectangle ++(\barW,0.32);%
      \foreach \tx in {#3}{\draw[tick, color=#2] (\barX+\tx,#1) -- ++(0,0.32);}%
    }

    \rowbar{2.95}{purpAd}{0.28, 0.48, 0.72, 0.92, 1.12, 1.32, 1.60, 1.84, 2.08, 2.32, 2.56, 2.84}
    \node[text=purpAt, anchor=west, font=\footnotesize] at (\barX+\barW+0.1, 3.11) {200 selected by mask};

    \rowbar{2.35}{tealBd}{1.60, 2.08, 2.32, 2.56, 2.84, 3.12, 3.36, 3.60, 3.84, 4.08}
    \node[text=tealBt, anchor=west, font=\footnotesize] at (\barX+\barW+0.1, 2.51) {\dots};

    \rowbar{1.75}{coralCd}{0.92, 1.32, 1.84, 3.12, 3.36, 3.92, 4.16, 4.32}
    \node[text=coralCt, anchor=west, font=\footnotesize] at (\barX+\barW+0.1, 1.91) {\dots};

    \draw[shared] (\barX+1.60, 1.75) -- (\barX+1.60, 3.27);
    \draw[shared] (\barX+3.12, 2.35) -- (\barX+3.12, 3.27);

    \draw[barfr, line width=0.4pt] (\barX,1.40) -- (\barX+\barW,1.40);

    \draw[draw=unionfr, fill=barbg, line width=0.4pt, rounded corners=1pt]
      (\barX,1.00) rectangle ++(\barW,0.32);
    \foreach \tx in {0.28, 0.48, 0.72, 0.92, 1.12, 1.32, 1.60, 1.84, 2.08, 2.32, 2.56, 2.84,%
                     3.12, 3.36, 3.60, 3.84, 3.92, 4.08, 4.16, 4.32}
      {\draw[tick, color=uniontk] (\barX+\tx,1.00) -- ++(0,0.32);}
    \node[text=uniontk, anchor=west, font=\footnotesize] at (\barX+\barW+0.1, 1.16)
      {$\approx 3{,}000$ across all masks};

    \draw[color=axisgr, line width=0.4pt] (\barX,0.65) -- (\barX+\barW,0.65);
    \node[text=axisgr, anchor=west, font=\footnotesize] at (\barX,0.40) {$0$};
    \node[text=axisgr, anchor=east, font=\footnotesize] at (\barX+\barW,0.40) {SAE dimension $d = 16{,}384$};

    \draw[leader, color=purpAd]  (clA.east) -- (\barX, 3.05);
    \draw[leader, color=tealBd]  (clB.east) -- (\barX, 2.45);
    \draw[leader, color=coralCd] (clC.east) -- (\barX, 1.85);

  \end{tikzpicture}
  \caption{Per-cluster masks in SAE's feature space.
  \emph{Left:} 2D projection of $\mathbb{R}^{d}$ with three highlighted clusters.
  \emph{Right:} the indices of features selected by the mask $m_c$ for each cluster.
  Each mask retains $200$ features, and the masks are partially overlapping; their union covers approximately $3{,}000$ features, still much smaller than $d$.}
  \label{fig:local-mask}
\vspace{-5mm}
\end{figure*}
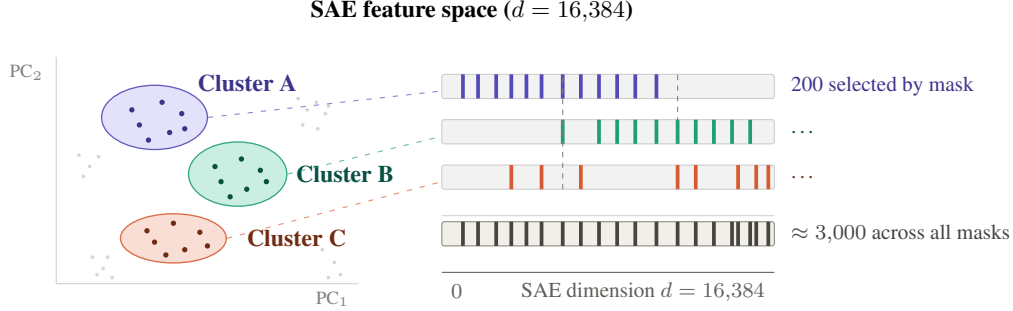

While this might appear to be the case at first sight, LLM activations exhibit far more structure than their nominal dimension suggests. The Linear Representation Hypothesis (LRH) \citep{park2023linear, park2025geometry} posits that each embedding is a \emph{sparse} combination of concept vectors. This concept space can be approximately recovered using a sparse autoencoder (SAE) \citep{cunningham2023sparse, paulo2024automatically, he2024llama}. \citet{hubotter2025specialization} further observe that within local ``neighborhoods'', the active features are largely shared (cf. \Cref{fig:local-mask}). Consequently, anomaly detection can in principle operate on a much lower-dimensional representation. However, this structure is \emph{local}: the relevant low-dimensional subspace varies across neighborhoods. Methods that instead impose a single, globally compressed representation, such as standard autoencoders, substantially underperform (cf. CVDD and Deep SVDD in ~\Cref{fig:overall_label}).
We leverage this local structure in the following contributions.

\begin{wrapfigure}{r}{0.48\textwidth}
    \centering
    \includegraphics[width=0.46\textwidth]{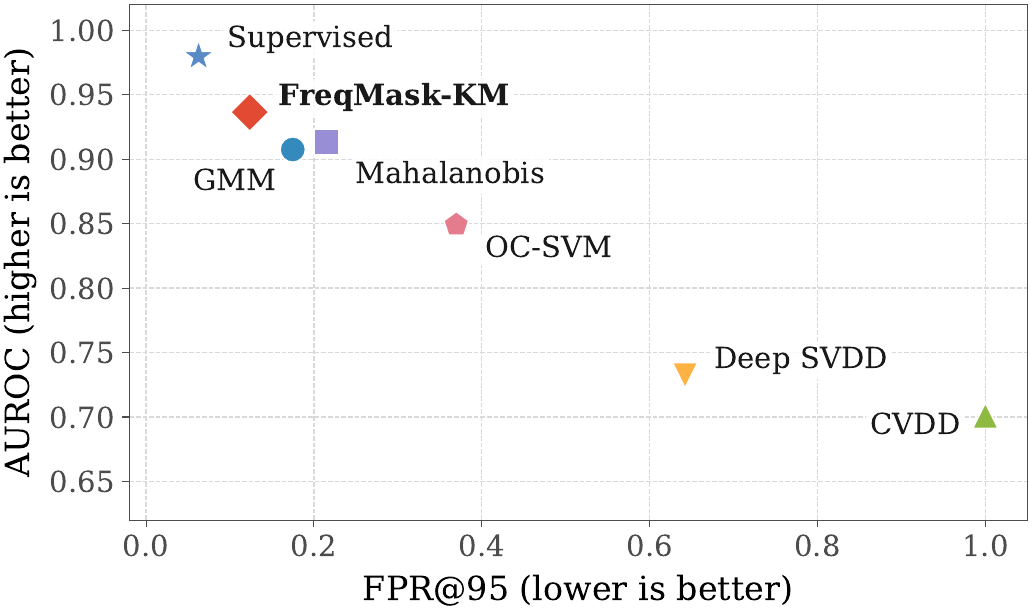}
    \caption{\textbf{Local sparsity in LLM activations enables one-class safety detection.}
    AUROC vs.\ FPR@95 on LLaMA3-8B for \textsc{FreqMask-KM} against one-class baselines
    (Mahalanobis, CVDD) and a supervised reference.}
    \label{fig:overall_label}
    \vspace{-5mm}
\end{wrapfigure}

\textbf{Modeling.} In \Cref{sec:framework}, we isolate four modeling choices through which local sparsity can be exploited for anomaly detection: \textbf{(i)} embedding space, \textbf{(ii)} clustering, \textbf{(iii)} choice of sparse subspace within each cluster, and \textbf{(iv)} scoring rule. We instantiate this modeling with two methods: \textsc{FreqMask-KM}, which is training-free at the cluster level, and \textsc{LearnedMask-LoRA}, which fits per-cluster low-rank adapters.

\textbf{Numerics.} In \Cref{sec:experiment}, we validate our findings across six
instruction-tuned models (Qwen2-1.5B \citep{bai2023qwen}, Ministral-8B \citep{liu2026ministral}, LLaMA3-8B \citep{grattafiori2024llama}, Qwen3-8B \citep{yang2025qwen3}, GPT-oss-20B \citep{agarwal2025gpt}, and Gemma 4-26B \citep{gemma_4_team_2026}) and
three categories of unsafe behavior. We show that several sparsity-aware algorithms perform on par or better than standard baselines. Ablations confirm that \emph{local sparsity} is essential for anomaly detection on LLM representations—any successful method must exploit it, either explicitly by design (as ours does) or implicitly through inductive bias. 

\textbf{Theoretical justification.} In \Cref{sec:theory}, we prove a sample-complexity bound for learning the corresponding anomaly detector. Notably, the bound scales only \emph{logarithmically} in the concept-space dimension $d_2$ (hidden dimension of SAE) and the ambient dimension $d_1$, and \emph{only linearly} in the local sparsity. This explains why anomaly detection does not fail in this high-dimensional regime.

\section{Background: Linear Representation Hypothesis and Local Sparsity}
\label{sec:lrh-background}

The Linear Representation Hypothesis posits that high-level concepts are encoded as linear directions in LLM activation space \citep{mikolov2013efficient, park2023linear}. Language models represent more concepts than their residual stream has dimensions by packing them into near-orthogonal directions, a.k.a. \emph{superposition} \citep{elhage2022toy}. 
Sparse autoencoders approximately recover the corresponding concept space as an overcomplete dictionary, such that each input activates only a small subset of hidden neurons \citep{cunningham2023sparse, bricken2023monosemanticity, templeton2024scaling, gao2024scaling}.  
\citet{hubotter2025specialization} puts this idea further by showing that the active features \emph{substantially overlap} across suitable ``neighbors''. In this section, we make these observations formal and define the set of notations used throughout the rest of the paper.

\paragraph{Residual stream and SAE.}
Let $\mathcal{X}$ denote the input space of token sequences. For a fixed layer $\ell$ of a transformer, we define the corresponding residual-stream activation map by
$a^{(\ell)} : \mathcal{X} \to \mathbb{R}^{d_1}$,
where $d_1$ stands for the dimension of residual-stream. A sparse autoencoder at layer $\ell$ is defined by its encoder-decoder pair:
$E : \mathbb{R}^{d_1} \to \mathbb{R}^{d_2}, D : \mathbb{R}^{d_2} \to \mathbb{R}^{d_1}, d_2 \gg d_1,$
where $E = \mathrm{ReLU}(W_Ex+b_E)$ with $D=W_Dx+b_D$, for $W_E,W^\top_D \in \mathbb{R}^{d_2\times d_1}$, $b_E \in \mathbb{R}^{d_2}$, $b_D \in \mathbb{R}^{d_1}$ and $\mathrm{ReLU}(a)=\max\left\{0,a\right\}$ is applied component-wise.
The optimal configuration of parameters $\Theta = (W_E,W_D,b_E,b_D)$ is selected to minimize \emph{regularized} $L_2$-reconstruction loss over activations $a^{\ell}(x)$ of safe inputs $x \sim \mathbb{P}$:
\begin{equation}
    \label{eq:sae-loss}
    \mathcal{L}_{\mathrm{SAE}}(E, D) \;:=\; \mathbb{E}_{x \sim \mathbb{P}} \left[\bigl\|a^{(\ell)}(x) - D\bigl(E(a^{(\ell)}(x))\bigr)\bigr\|_2^2\right] \;+\; \lambda \, \mathbb{E}_{x \sim \mathbb{P}} \left[\bigl\|E(a^{(\ell)}(x))\bigr\|_1\right] \rightarrow \min_{\Theta},
\end{equation}
where $\lambda > 0$ is a regularization coefficient. $\mathrm{ReLU}(\cdot)$ activation in conjunction with $L_1$-penalty on SAE features $z(x):=E(a^{(\ell)}(x)) \;\in\; \mathbb{R}^{d_2}$ ensures that the resulting representations $z(x)$ are indeed sparse. The corresponding active features of $z(x)$ form an active support $\phi(x) \;:=\; \{\, j \in [d_2] : z(x)_j \neq 0 \,\}$ of non-zero entries of cardinality $|\phi(x)| = \|z(x)\|_0$.

\paragraph{Global (average) sparsity.}
The key property of a well-trained SAE is that, despite extreme overparameterization $d_2 \gg d_1$, each input activates only a small subset of concepts. Formally: %
\begin{equation}
    \label{eq:global-sparsity}
    \mathbb{E}_{x \sim \mathbb{P}} \, \|z(x)\|_0 \;=:\; s \;\ll\; d_2.
\end{equation}
Across the SAEs we use, average sparsity $s$ is two to three orders of magnitude smaller than $d_2$.\footnote{For some of the models we consider the corresponding values are $(d_1, d_2, s) \approx (4096, 16384, 100)$.}

\paragraph{Local sparsity.}
Central to this work is a stronger property---the joint sparsity structure within a ``neighborhood''. For a fixed metric $\rho$ on the concept space (see examples in \Cref{sec:framework-clustering}), a safe input $x \in \cX_0 := \mathrm{supp}(\mathbb{P})$ and a radius $r > 0$, we denote the safe-input neighborhood of $x$ at scale $r$ by
$$\mathcal{N}_r(x) \;:=\; \{\, x' \in \mathcal{X}_0 : \rho(z(x), z(x')) \le r \,\}.$$
 The local-sparsity property posits that for some ``not small'' $r_0 \geq 0$, the \emph{union} of active supports over the corresponding neighborhood is significantly smaller than the dimension of the concept space:
\begin{equation}
    \label{eq:local-sparsity}
    \mathbb{E}_{x \sim \mathbb{P}} \, \Bigl|\, \bigcup_{x' \in \mathcal{N}_{r_0}(x)} \phi(x') \,\Bigr| \;=:\; s_{\mathrm{loc}} \;\ll\; d_2.
    \vspace{-2mm}
\end{equation}
In words, active features are largely \emph{shared} across nearby points rather than being distinct. This implies that our data lies in a manifold containing ``sparse patches'', which justifies the use of ``local sparsity'' terminology. As we will see, this assumption is motivated by the empirical validation on six model families in \Cref{sec:experiment}, and it is central to our theoretical analysis in \Cref{sec:theory}.

\paragraph{Implication for one-class anomaly detection.}
Equation~\eqref{eq:local-sparsity} is the structural property our modeling heavily relies on. If we design an anomaly detection algorithm that operates on sparse local patches implicitly (or by inductive bias), the effective 
dimension it faces amounts to $s_{\mathrm{loc}}$ rather than $d_2$. 
We make it precise in the upcoming sections:
\Cref{sec:framework}  focuses on possible modeling choices (embedding space, clustering, local subspace, scoring) to ensure \eqref{eq:local-sparsity} holds; 
\Cref{sec:experiment} confirms this picture empirically across six model families;
and \Cref{sec:theory} shows that, as long as \eqref{eq:local-sparsity} holds, the sample complexity of the resulting estimator scales linearly in $s_{\mathrm{loc}}$ and only logarithmically in $d_2$.

\section{Framework for Local Sparse Anomaly Detection}
\label{sec:framework}

We now describe a framework that poses anomaly detection on LLM activations as local problems for which their effective dimension is determined by $s_{\mathrm{loc}}$ rather than by $d_2$. The framework consists of four stages: the \emph{design of embedding space} in which the detector operates, the \emph{clustering procedure} $\mathcal{C} = \{C_1, \dots, C_{\ncluster}\}$ to recover safe neighborhoods, the \emph{local sparse mask} $m_c \in \{0,1\}^{d_2}$ identified within each cluster $C_c$, and the \emph{scoring} rule used to evaluate new test examples. Each stage admits several principled choices, detailed in \Cref{sec:embedding-space,sec:framework-clustering,sec:framework-subspace,sec:framework-scoring}.
We treat these steps as a design space rather than a single algorithm: different model geometries may favor different instantiations, and the experiments in \Cref{sec:experiment} support this modular view.

\subsection{Embedding Space}
\label{sec:embedding-space}

The first stage selects the space in which all subsequent phases operate.
Recall that the concepts are entangled in the \emph{dense} activations $a^{(\ell)} : \mathcal{X} \to \mathbb{R}^{d_1}$ for layer $\ell$. Therefore, searching for the local sparse structure of Equation \eqref{eq:local-sparsity} directly in $a^{(\ell)}$ is hopeless. We address this by lifting the activations to the sparse space induced by an SAE. By construction, the corresponding map $z : \mathcal{X} \to \mathbb{R}^{d_2}$ with $d_2 \gg d_1$ produces sparse representations, i.e., $\|z(x)\|_0 \ll d_2$. This structure enables the subsequent clustering to identify neighborhoods compatible with Equation \eqref{eq:local-sparsity}. 
We adopt a $\mathrm{ReLU}(\cdot)$ activation in the encoder $E(\cdot)$ with an $L_1$ penalty on $z(x)$, which enforces sparsity on average. 
\citet{kantamneni2025sparse} analyzes the role of SAE in supervised probing, which we discuss more in Appendix~\ref{sec:related-sae-probing}.

\subsection{Clustering}
\label{sec:framework-clustering}
The clustering stage partitions the safe distribution into regions
for which the local-sparsity property \eqref{eq:local-sparsity} is
expected to hold. Formally, given a finite set of safe examples
$\mathcal{D} = \{x_1, \dots, x_n\} \subset \mathcal{X}_0$ and a fixed choice of metric
$\rho$ on the embedding space, the outcome of this stage is a partition
$\mathcal{C} = \{C_1, \dots, C_{\ncluster}\}$ of $\mathcal{D}$ obtained
by running $K$-means with metric $\rho$ on the chosen embeddings of $\mathcal{D}$. Two
design choices govern this partition: the metric $\rho$ and the number
of clusters $\ncluster$.

\textbf{Distance metric.} The choice of metric set the trade-off between active support and activation ``magnitude'' when deciding which points are grouped together. Let $b(x) := \mathbf{1}[z(x) > 0] \in \{0,1\}^{d_2}$
be the active-feature indicator of sparse representations $z(x)$. We examine three points along this trade-off: \textbf{(i)}
the Hamming distance $\rho_{L_0}(x, x') = \|b(x) - b(x')\|_0$, which
depends on $z$ only through $b$ and groups the points \emph{exlusively} by their active
support $\phi(x)$, explicitly targeting small
values of $s_{\mathrm{loc}}$; \textbf{(ii-iii)} the magnitude-aware
$\rho_{L_p}(x, x') = \|z(x) - z(x')\|_p$ for $p \in \{1, 2\}$, which
retain activation influence, i.e., it groups points $\mathcal{D}$ by both which features fire and how strongly, at the cost of greater sensitivity to noise in the activation magnitudes.

\paragraph{Number of clusters.}
The number of clusters $\ncluster$ trades off cluster homogeneity
against per-cluster sample size. With too few clusters, the union of
inner-cluster supports $\bigcup_{x \in C_c} \phi(x)$ grows and \eqref{eq:local-sparsity}
degrades. With too many, per-cluster sample sizes
$|C_c| \approx N / \ncluster$ become small, and parametric per-cluster
estimators (e.g., per-cluster covariance for Mahalanobis scoring)
become ill-conditioned. An appropriate choice of $\ncluster$ takes into account both the
safe sample size $N$ and the global sparsity $s$.

\subsection{Local sparse subspace}
\label{sec:framework-subspace}
At this stage, for each cluster $C_c \in \mathcal{C}$, we identify a sparse subspaces of $\mathbb{R}^{d_2}$ on which inputs are scored.
The subspace is identified by its binary mask $m_c \in \{0,1\}^{d_2}$, applied
to SAE features $z(x)$ via the Hadamard product $m_c \odot z(x)$. We
consider two constructions below.

\paragraph{Frequency mask.}
For an index set $I \subseteq \mathcal{D}$, let
$f^{(I)} \in [0,1]^{d_2}$ be the per-feature activation frequency
over $I$,
$f^{(I)}_j \;=\; \frac{1}{|I|} \sum_{x \in I} \mathbf{1}\bigl[z(x)_j \neq 0\bigr], j \in [d_2],$
and define the corresponding top-$\sparsity$ mask by
$m^{(I)} \;:=\; \mathbf{1}\bigl[f^{(I)} \in \mathrm{top}_\sparsity(f^{(I)})\bigr]$,
where $\mathrm{top}_\sparsity(\cdot)$ returns the indices of the
$\sparsity$ largest entries. We distinguish between two choices for $I$:
\textbf{(i)} we set $I = \mathcal{D}$ and apply the same \emph{global mask}
in each cluster ($m_c := m^{(\mathcal{D})}$ for all
$c \in [\ncluster]$), leaving the clustering as the only source of local information;
\textbf{(ii)} alternatively, we set  $I = C_c$ and apply \emph{per-cluster mask} $m_c := m^{(C_c)}$.

\paragraph{Per-cluster learned subspace.}
An alternative is to let the SAE adjust locally via per-cluster low-rank adapters \citep{hu2022lora}: $E_c = E + \Delta E_c$, $D_c = D + \Delta D_c$, with $\mathrm{rank}(\Delta E_c), \mathrm{rank}(\Delta D_c) \le r \ll \min(d_1, d_2)$\footnote{With slight abuse of notation, we write $E$ for $W_E$ and $D$ for $W_D$.}. The mask $m_c$ and the adapters $(\Delta E_c, \Delta D_c)$ are then jointly optimized on $C_c$ against an auxiliary objective.
Fitting $E$ and $D$ directly on each cluster, without the rank
constraint above, was found in preliminary experiments to increase the
held-out reconstruction loss and the average density of $z(x)$ on safe
data, relative to the initial SAE. We therefore restricted per-cluster
updates to the \emph{low}s rank-$r$ adapters $(\Delta E_c, \Delta D_c)$.

\subsection{Scoring}
\label{sec:framework-scoring}

The final stage assigns each test point $x$ an anomaly score given its
cluster $c^\star(x) \in [\ncluster]$ and the local mask $m_{c^\star}$.
We use the nearest-centroid assignment
    $c^\star(x) \;=\; \arg\min_{c \in [\ncluster]} \rho\bigl(z(x), \mu_c\bigr), \mu_c \;=\; \frac{1}{|C_c|} \sum_{x' \in C_c} z(x'),$
where $\rho$ is the clustering metric
(\Cref{sec:framework-clustering}). We consider two scoring rules.

\textbf{Centroid distance} in the masked SAE feature space,
$s_{\mathrm{dist}}(x) = \bigl\lVert m_{c^\star} \odot \bigl(z(x) - \mu_{c^\star}\bigr) \bigr\rVert_p, p \in \{1, 2\},$
or, when per-cluster covariance estimates are reliable, its
Mahalanobis variant. This is the standard typical-set realization:
atypical points lie far from the local centroid in the active-feature
subspace.

\textbf{Reconstruction residual} under the cluster's local
encoder/decoder pair $(E_{c^\star}, D_{c^\star})$,
$s_{\mathrm{rec}}(x) \;=\; \bigl\lVert a^{(\ell)}(x) - D_{c^\star}\!\bigl(m_{c^\star} \odot E_{c^\star}(a^{(\ell)}(x))\bigr) \bigr\rVert_2,$
matched to the learned-subspace construction in
\Cref{sec:framework-subspace}  for reconstruction loss as the auxiliary objective.

\textbf{Per-cluster calibration.} Both centroid and residual scores use global thresholding, i.e., the cut-off parameter for a decision rule is independent of the cluster assigned to the evaluation point. Empirically, we notice that cluster assignment is highly non-homogeneous for unsafe data -- unsafe samples concentrate in a small subset of the safe clusters. Such discrepancy can be exploited for calibration at a cost of a small labeled set ($<1\%$ of data). We demonstrate that the scoring rule can be adjusted per cluster via a simple affine map. We further discuss it in  \Cref{sec:discussion} and Appendix~\ref{app:calibration}.

\begin{figure*}[t]
\centering
\begin{minipage}[t]{0.48\textwidth}
\begin{algorithm}[H]
\small
\caption{\textsc{FreqMask-KM}}
\label{alg:freqmask-km}
\begin{algorithmic}[1]
\Require Safe set $\mathcal{D}$, SAE encoder $E$,
  layer $\ell$, clusters number $\ncluster$, sparsity $\sparsity$
\Statex \textbf{Fit:}
\State $z_i \gets E(a^{(\ell)}(x_i))$ for $x_i \in \mathcal{D}$
  \Comment{SAE features}
\State $\{C_1, \ldots, C_{\ncluster}\} \gets \textsc{KMeans}_{L_2}(\{z_i\})$
\State $\mu_c \gets \text{mean}(\{z_i : i \in C_c\})$
  \Comment{cluster centroid}
\State $b_i \gets \mathbf{1}[z_i > 0]$
  \Comment{binarize}
\State $f_j \gets \tfrac{1}{N}\sum_i (b_{i})_j $
  \Comment{pooled frequency}
\State $m \gets \mathbf{1}[f \in \text{top-}\sparsity(f)]$
  \Comment{global mask}
\Statex \textbf{Scoring} $x$\textbf{:}
\State $z \gets E(a^{(\ell)}(x))$
\State $c^\star \gets \arg\min_c \lVert z - \mu_c \rVert_2$
  \Comment{assign cluster}
\Statex \textbf{Decision} (given threshold $\tau$)\textbf{:}
\State $s(x) \gets \lVert m \odot z - m \odot \mu_{c^\star} \rVert_1$
\State \Return \textsc{unsafe} if $s(x) > \tau$, else \textsc{safe}
\end{algorithmic}
\end{algorithm}
\end{minipage}%
\hfill
\begin{minipage}[t]{0.48\textwidth}
\begin{algorithm}[H]
\small
\caption{\textsc{LearnedMask-LoRA}}
\label{alg:learnedmask-lora}
\begin{algorithmic}[1]
\Require Safe set $\mathcal{D}$, SAE $(E, D)$,
  layer $\ell$, clusters $\ncluster$, sparsity $\lambda$
\Statex \textbf{Training:}
\State $z_i \gets E(a^{(\ell)}(x_i))$ for $x_i \in \mathcal{D}$
\State $\{C_1, \ldots, C_{\ncluster}\} \gets \textsc{KMeans}_{L_2}(\{z_i\})$
\For{each cluster $c$}
  \Comment{Stage 1: binary mask}
  \State $m_c \gets \arg\min_{m \in \{0,1\}^d} \mathcal{L}_{\text{rec}} + \lambda \lVert m \rVert_1$
\EndFor
\For{each cluster $c$}
  \Comment{Stage 2: LoRA adapters}
  \State $(\Delta E_c, \Delta D_c) \gets \arg\min \mathcal{L}_{\text{rec}}$
\EndFor
\Statex \textbf{Scoring} $x$\textbf{:}
\State $z \gets E(a^{(\ell)}(x))$, $\;c^\star \gets \arg\min_c \lVert z - \mu_c \rVert_2$
\Statex \textbf{Decision} (given threshold $\tau$)\textbf{:}
\State $s(x) \gets \lVert a^{(\ell)}(x) - D_{c^\star}(m_{c^\star} \odot E_{c^\star}(a^{(\ell)}(x))) \rVert_2$
\State \Return \textsc{unsafe} if $s(x) > \tau$, else \textsc{safe}
\end{algorithmic}
\end{algorithm}
\end{minipage}
\vspace{-6mm}
\end{figure*}

\subsection{Two concrete instantiations}
\label{sec:framework-instantiations}

We describe two instantiations of this framework, \textsc{FreqMask-KM}
and \textsc{LearnedMask-LoRA}, with no per-cluster learned parameters
and with per-cluster low-rank adaptation, respectively. Pseudocode is
given in \Cref{alg:freqmask-km,alg:learnedmask-lora}.

\textbf{\textsc{FreqMask-KM}.} The pipeline operates on SAE features
$z(x) \in \mathbb{R}^{d_2}$. Clustering uses Lloyd's algorithm with
uniform random centroid initialization (FAISS \citep{johnson2019billion} \texttt{Kmeans}
defaults), minimizing
$\sum_{c=1}^{\ncluster} \sum_{x \in C_c} \lVert z(x) - \mu_c \rVert_2^2$.
The local subspace uses the global frequency mask $m_{\mathrm{global}}$, shared across all clusters. Scoring is
the $\ell_1$ centroid distance:
$s_{\mathrm{dist}}(x) = \lVert m_{\mathrm{global}} \odot (z(x) - \mu_{c^\star}) \rVert_1$.

\textbf{\textsc{LearnedMask-LoRA}.} The pipeline operates on SAE
features $z(x)$ and uses the same clustering as
\textsc{FreqMask-KM}. The local subspace is induced by per-cluster LoRA
adapters $(\Delta E_c, \Delta D_c)$, fit
with the SAE backbone frozen and $m_c$ is learned by minimizing the per-cluster reconstruction loss
$\mathcal{L}_{\text{rec}} = \mathbb{E}_{x \in C_c} \lVert a^{(\ell)}(x) - D_c(m_c \odot E_c(a^{(\ell)}(x))) \rVert_2^2$. Scoring is done via the reconstruction residuals
$s_{\mathrm{rec}}(x)$ under the assigned
cluster's adapted encoder-decoder $(E_{c^\star}, D_{c^\star})$.

These two instances interpolate the design space defined in
\Cref{sec:framework-subspace}: the first fits no per-cluster
parameters, the second fits a low-rank adapter per cluster against the
downstream scoring objective. \Cref{sec:experiment} shows that both
reach competitive AUROC across six model families.
Detailed ablation studies can be found in Appendix~\ref{app:freqmask-km-ablations} and Appendix~\ref{app:learnedmask-lora-ablations}, where we vary \textsc{FreqMask-KM} and \textsc{LearnedMask-LoRA} one stage at a time: clustering metric, number of clusters $\ncluster$, mask size $\sparsity$, mask on versus off, scoring rule, within-cluster overlap of the safe and unsafe active supports, extraction layer, SAE sparsity penalty, etc., along with many more detailed analysis.

\section{Experiments}
\label{sec:experiment}

\subsection{Setup}
\label{sec:exp-setup}

\paragraph{Models.}
We evaluate on six instruction-tuned model families: \texttt{qwen2-1.5b-instruct} \citep{bai2023qwen}, \texttt{ministral-8b-instruct} \citep{liu2026ministral}, \texttt{llama3-8b-instruct} \citep{grattafiori2024llama}, \texttt{qwen3-8b}, \texttt{gpt-oss-20b}, and \texttt{gemma-4-26b}. Activations are extracted from the residual stream at layer $20$ for Qwen2-1.5B, layer $20$ for Ministral, layer $25$ for LLaMA3, layer $30$ for Qwen3, layer $20$ for GPT-oss, and layer $25$ for Gemma-4-26B. Layers were chosen at mid-to-late depth, where safety-relevant features have emerged but the representation has not yet narrowed toward the output vocabulary, and were not tuned against downstream AUROC. For every input, we record the residual-stream activation of the model's own generated response at the final response token.

\paragraph{Data.}
The safe set pools general-capability and instruction-following data with the safe portions of the evaluation benchmarks, totaling approximately $206$K examples per model; per-source composition, splits, and licenses are deferred to Appendix~\ref{app:datasets}. We evaluate on three unsafe categories: \textit{BeaverTails} \citep{ji2023beavertails} (requests for harmful content), \textit{ToxiGen} \citep{hartvigsen2022toxigen} (toxic and hateful language), and \textit{HarmBench} \citep{mazeika2024harmbench} adversarial attacks, pooled across eight attack algorithms. Reported AUROC and TPR@$5\%$FPR are means across these three categories.

\paragraph{Baselines, metrics, and method of focus.}
We compare against established one-class methods: global Mahalanobis distance \citep{nader2014mahalanobis}, Gaussian Mixture Model (GMM), One-Class SVM \citep{scholkopf1999support}, Deep SVDD \citep{ruff2018deep}, and CVDD \citep{ruff2019self}. We also report a supervised linear probe as a reference ceiling; the probe is trained with access to unsafe labels and is not a competitor to the unsupervised framework. For each method, we report its AUROC and TPR@$5\%$FPR (true positive rate at a $5\%$ false positive rate), with additional details in Appendix~\ref{app:hyperparameters}.

\subsection{Unsupervised Anomaly Detection on LLM Activations Is Feasible}
\label{sec:exp-main}

\begin{table}[t]
\centering
\footnotesize
\caption{AUROC and TPR@$5\%$FPR for two framework instantiations, five unsupervised baselines, and a supervised linear probe across six instruction-tuned model families. The ``Local'' column marks whether the method scores on a sparse local subspace ($\bullet$) or not ($-$). Results are obtained across $5$ random seeds. Local-sparsification instantiations are the best unsupervised method on 4 models, and for the rest, a close second (within $0.025$ AUROC). TPR@$5\%$FPR rankings are less uniform. In each column the best unsupervised result is in \textbf{bold} and the second best is \underline{underlined}; the supervised probe is excluded from this ranking. The five $\geq$8B models perform substantially better than Qwen 1.5B, indicating that representation quality is key to successful anomaly detection.}
\label{tab:main-results}
\setlength{\tabcolsep}{4pt}
\resizebox{\textwidth}{!}{%
\begin{tabular}{lcccccc!{\vrule}c}
\toprule
\multicolumn{8}{c}{\textit{AUROC} ($\uparrow$)} \\
\midrule
Method & Local & Ministral 8B & LLaMA3 8B & Qwen3 8B & GPT-oss 20B & Gemma 4 26B & Qwen 1.5B \\
\midrule
BCE (linear probe, supervised) & --     & $0.993 \pm 0.000$ & $0.972 \pm 0.003$ & $0.993 \pm 0.001$ & $0.990 \pm 0.001$ & $0.983 \pm 0.001$ & $0.982 \pm 0.002$ \\
\midrule
\textsc{FreqMask-KM}      & $\bullet$ & $\mathbf{0.914 \pm 0.000}$ & $\mathbf{0.937 \pm 0.000}$ & $0.915 \pm 0.000$ & $0.900 \pm 0.000$ & $0.856 \pm 0.000$ & $0.762 \pm 0.000$ \\
\textsc{LearnedMask-LoRA} & $\bullet$ & $\underline{0.896 \pm 0.001}$ & $0.912 \pm 0.001$ & $\underline{0.917 \pm 0.000}$ & $\underline{0.919 \pm 0.001}$ & $\mathbf{0.904 \pm 0.000}$ & $\mathbf{0.794 \pm 0.001}$ \\
Mahalanobis    & --     & $0.892 \pm 0.000$ & $\underline{0.913 \pm 0.000}$ & $0.914 \pm 0.000$ & $0.843 \pm 0.000$ & $\underline{0.901 \pm 0.000}$ & $0.756 \pm 0.000$ \\
GMM            & --     & $0.881 \pm 0.011$ & $0.906 \pm 0.002$ & $\mathbf{0.942 \pm 0.002}$ & $\mathbf{0.925 \pm 0.008}$ & $0.861 \pm 0.004$ & $\underline{0.780 \pm 0.005}$ \\
One-Class SVM  & --     & $0.829 \pm 0.002$ & $0.841 \pm 0.011$ & $0.885 \pm 0.007$ & $0.850 \pm 0.002$ & $0.886 \pm 0.003$ & $0.756 \pm 0.004$ \\
Deep SVDD      & --     & $0.389 \pm 0.152$ & $0.475 \pm 0.183$ & $0.389 \pm 0.118$ & $0.535 \pm 0.024$ & $0.665 \pm 0.026$ & $0.507 \pm 0.101$ \\
CVDD           & --     & $0.712 \pm 0.029$ & $0.665 \pm 0.037$ & $0.705 \pm 0.029$ & $0.706 \pm 0.059$ & $0.672 \pm 0.042$ & $0.545 \pm 0.021$ \\
\midrule
\multicolumn{8}{c}{\textit{TPR@}$5\%$\textit{FPR} ($\uparrow$)} \\
\midrule
Method & Local & Ministral 8B & LLaMA3 8B & Qwen3 8B & GPT-oss 20B & Gemma 4 26B & Qwen 1.5B \\
\midrule
BCE (linear probe, supervised) & --     & $0.964 \pm 0.001$ & $0.905 \pm 0.014$ & $0.957 \pm 0.004$ & $0.948 \pm 0.003$ & $0.942 \pm 0.003$ & $0.876 \pm 0.014$ \\
\midrule
\textsc{FreqMask-KM}      & $\bullet$ & $0.675 \pm 0.000$ & $\mathbf{0.457 \pm 0.000}$ & $0.429 \pm 0.000$ & $0.535 \pm 0.000$ & $0.264 \pm 0.000$ & $0.287 \pm 0.000$ \\
\textsc{LearnedMask-LoRA} & $\bullet$ & $\underline{0.685 \pm 0.001}$ & $\underline{0.450 \pm 0.009}$ & $\underline{0.581 \pm 0.001}$ & $\underline{0.719 \pm 0.002}$ & $\mathbf{0.642 \pm 0.003}$ & $\mathbf{0.359 \pm 0.006}$ \\
Mahalanobis    & --     & $0.458 \pm 0.000$ & $0.409 \pm 0.000$ & $0.458 \pm 0.000$ & $0.421 \pm 0.000$ & $\underline{0.641 \pm 0.000}$ & $0.264 \pm 0.000$ \\
GMM            & --     & $0.664 \pm 0.009$ & $0.381 \pm 0.010$ & $\mathbf{0.639 \pm 0.008}$ & $\mathbf{0.725 \pm 0.018}$ & $0.624 \pm 0.014$ & $\underline{0.343 \pm 0.039}$ \\
One-Class SVM  & --     & $\mathbf{0.694 \pm 0.004}$ & $0.185 \pm 0.020$ & $0.252 \pm 0.026$ & $0.567 \pm 0.006$ & $0.512 \pm 0.031$ & $0.217 \pm 0.004$ \\
Deep SVDD      & --     & $0.022 \pm 0.012$ & $0.036 \pm 0.015$ & $0.038 \pm 0.019$ & $0.078 \pm 0.033$ & $0.116 \pm 0.042$ & $0.063 \pm 0.018$ \\
CVDD           & --     & $0.358 \pm 0.080$ & $0.348 \pm 0.077$ & $0.416 \pm 0.056$ & $0.428 \pm 0.111$ & $0.352 \pm 0.087$ & $0.132 \pm 0.041$ \\
\bottomrule
\end{tabular}%
}
\vspace{-5mm}
\end{table}

Table~\ref{tab:main-results} reports 3-category AUROC and TPR@$5\%$FPR across the six models. On four of the five $\geq$8B models the two local-sparsification instantiations land in a tight band.
Gemma is the exception: its global-mask instantiation \textsc{FreqMask-KM} is weaker, while the learned instantiation \textsc{LearnedMask-LoRA} stays in band and is the best unsupervised method on Gemma. \textsc{FreqMask-KM} is the best unsupervised method on Ministral and LLaMA3; on Qwen3 and GPT-oss the GMM density baseline is marginally ahead ($0.942$ and $0.925$), with a local instantiation second within $0.025$ AUROC, and on Gemma \textsc{LearnedMask-LoRA} leads with Mahalanobis a close second ($0.904$ vs.\ $0.901$). The representation-learning baselines trail substantially,
indicating that their optimizations cannot find a useful representation in this setting. A roughly $0.04$--$0.08$ AUROC gap to the supervised linear probe remains on the $\geq$8B models, which we read as the cost of dropping the assumption that the unsafe distribution is known in advance (Section~\ref{sec:discussion}).

TPR@$5\%$FPR is less uniform than AUROC: the best method varies by model: One-Class SVM on Ministral, \textsc{FreqMask-KM} on LLaMA3, GMM on Qwen3 and GPT-oss, and \textsc{LearnedMask-LoRA} on Gemma. Meanwhile, \textsc{LearnedMask-LoRA} stays within a few points of the best on every $\geq$8B model. 
No single method dominates across the $\geq$8B models, consistent with different LLMs having different internal representation geometries that may favor different algorithmic choices. On Qwen 1.5B, all unsupervised methods score below $0.80$ AUROC with low TPR@$5\%$FPR, indicating that representation quality bounds what unsupervised detection can recover. We defer the structural reading of these results to Section~\ref{sec:exp-decomposition}, where the decomposition table licenses it directly.

\subsection{Decomposition: Which Design Choices Most Impact the Performance?}
\label{sec:exp-decomposition}

\begin{table}[t]
\centering
\footnotesize
\caption{Decomposition of framework instantiations and baselines along three design axes: clustering, sparsification, and scoring. Performance is averaged across the five $\geq$8B models and $5$ random seeds. The ``Local'' column marks whether scoring uses a sparse local subspace ($\bullet$) or not ($-$). We use the mean AUROC from \Cref{tab:main-results} across those five models. Methods in the upper block operate in the raw LLM activation space; methods in the lower block operate in an SAE feature space. Moving from activation space to SAE space in CVDD contributes a large share of the gain; once in SAE space, local-subspace constructions \textsc{FreqMask-KM} and \textsc{LearnedMask-LoRA} land in the same AUROC band despite differing substantially.}
\label{tab:decomposition}
\begin{tabular}{lclllc}
\toprule
Method & Local & Clustering & Sparsification & Scoring & AUROC \\
\midrule
Mahalanobis & -- & None & Full covariance & Mahalanobis distance & 0.893 \\
GMM & -- & Mixture models & None & Mixture likelihood & 0.903 \\
CVDD & -- & KMeans & None & $L_2$ to cluster centroid & 0.692 \\
\midrule
CVDD (SAE space) & -- & KMeans & None & $L_2$ to cluster centroid & 0.806 \\
\textsc{FreqMask-KM} & $\bullet$ & KMeans & Global top-$k$ & $L_1$ to cluster centroid & 0.904 \\
Local \textsc{FreqMask-KM} & $\bullet$ & KMeans & per-cluster top-$k$ & $L_1$ to cluster centroid & 0.790 \\
\textsc{LearnedMask-LoRA} & $\bullet$ & KMeans & per-cluster learned & SAE residual & \textbf{0.909} \\
\bottomrule
\end{tabular}
\vspace{-5mm}
\end{table}

Table~\ref{tab:decomposition} decomposes framework instantiations and baselines along representation space, clustering, sparsification, and scoring. The lower block isolates the role of the SAE embedding: \textit{CVDD (SAE space)} is exactly CVDD with the SAE encoder substituted for the activation embedding and training. The three framework instantiations share KMeans clustering and nearest-centroid cluster assignment, and differ in how the sparse subspace is identified: globally by activation frequency (\textsc{FreqMask-KM}), per-cluster by frequency (Local \textsc{FreqMask-KM}), or per-cluster by a learned mask trained jointly with a LoRA adapter against SAE reconstruction (\textsc{LearnedMask-LoRA}).

\paragraph{Representation space carries a large share of the gain.}
\looseness -1 \textit{CVDD (SAE space)} differs from \textit{CVDD} only in the encoder. Substituting the SAE encoder for the activation embedding lifts mean AUROC from $0.692$ to $0.806$. Adding a sparse local subspace on top lifts it further: \textsc{FreqMask-KM} and \textsc{LearnedMask-LoRA} reach $0.904$ and $0.909$ respectively, on par with or above the global density baselines (Mahalanobis $0.893$, GMM $0.903$) while operating in a much lower effective dimension.

\begin{table}[t]
\centering
\footnotesize
\caption{Comparison between the mask disabled against the deployed masked
configuration. The masks improve the bigger models more significantly. 
}
\label{tab:mask-ablation}
\setlength{\tabcolsep}{5pt}
\begin{tabular}{lcccccc}
\toprule
& Qwen 1.5B & Ministral 8B & LLaMA3 8B & Qwen3 8B & GPT-oss 20B & Gemma 4 26B \\
\midrule
No mask  & $0.776$ & $0.904$ & $0.937$ & $0.914$ & $0.781$ & $0.766$ \\
Mask     & $0.762$ & $0.914$ & $0.937$ & $0.915$ & $0.900$ & $0.856$ \\
$\Delta$ & $-0.014$ & $+0.010$ & $\phantom{+}0.000$ & $+0.001$ & $+0.119$ & $+0.090$ \\
\bottomrule
\end{tabular}
\vspace{-5mm}
\end{table}

\paragraph{The mask is necessary for the largest models.}
We study the necessity of the mask by disabling it, scoring $\ell_1$ over all SAE features, and re-sweeping the
cluster count $\ncluster$ for the no-mask condition alone, so it is given
its own best case (\Cref{tab:mask-ablation}; full sweep and setup in
Appendix~\ref{app:freq-abl-K}). The masked configuration matches or
exceeds the best no-mask result on all five $\geq$8B models, and the gap
grows with model size.
Larger models
pack more concepts into a given activation, so the local subspace is
harder to recover from the dense SAE representation and the mask
contributes more.

\paragraph{Within the SAE space, the mechanism for finding the subspace does not dominate.}
\textsc{FreqMask-KM} and \textsc{LearnedMask-LoRA} differ substantially in cost: global frequency selection is essentially free, whereas \textsc{LearnedMask-LoRA} requires fitting $\ncluster$ masks and per-cluster low-rank adapters. Both are competitive against the global density baselines while operating in a lower effective dimension. Their agreement is at the aggregate level rather than model by model: \textsc{FreqMask-KM} leads by $0.018$ on Ministral and $0.025$ on LLaMA3, the two are tied on Qwen3, and \textsc{LearnedMask-LoRA} leads by $0.019$ on GPT-oss and $0.048$ on Gemma (\Cref{tab:main-results}). Which construction wins is therefore model-dependent, but neither is far from the other on any $\geq$8B model. We read this as evidence for the structural claim: if local low effective dimension is the operative property, any reasonable procedure for surfacing that subspace should work. The comparability across instantiations supports the claim about structure rather than about any particular algorithm.

\subsection{Local Effective Dimension Is Genuinely Low}
\label{sec:exp-local-dim}

\begin{figure}[t]
\centering
\includegraphics[width=\textwidth]{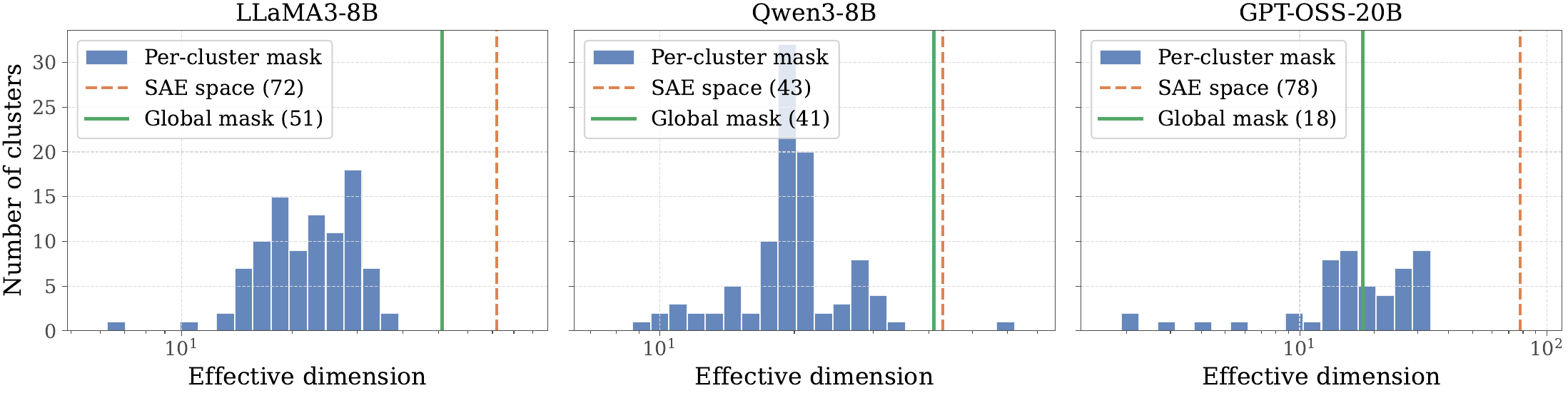}
\caption{Effective dimension ($d_{90}$) across representation spaces for three of the model families; two more are shown in Appendix~\ref{app:effective-dim} (Figure~\ref{fig:effective-dim-appendix}), and Gemma-4-26B is omitted because its selected mask has only $k=25$ features, which already bounds its masked $d_{90}$ values by $25$. Blue bars show the distribution of per-cluster $d_{90}$ values, where each cluster's $d_{90}$ is computed on that cluster's members within its own top-$k$ mask $\mathcal{M}_c$ (the $k$ SAE features that fire most often within the cluster, with $k$ as for the global mask; Appendix~\ref{app:effective-dim}). Vertical lines mark $d_{90}$ in the raw SAE feature space (dashed orange) and the global-mask subspace (solid green).}
\label{fig:effective-dim}
\vspace{-5mm}
\end{figure}

The framework's feasibility and its robustness to design choice both rest on the premise that safe cluster members occupy low-dimensional subspaces of the SAE feature space. We test this by computing $d_{90}$, the smallest number of PCA components explaining $90\%$ of variance, in three spaces: SAE space, the global-mask subspace spanned by the $k$ features selected by \textsc{FreqMask-KM}, and, for each cluster separately, a subspace of the same size $k$ spanned by the features most frequently active within that cluster (the mask used by Local \textsc{FreqMask-KM}). Definitions are deferred to Appendix~\ref{app:effective-dim} with extended results; Figure~\ref{fig:effective-dim} shows the per-cluster distribution for three model families.

Across the three models, the per-cluster $d_{90}$ is small in absolute terms (median $22$ for LLaMA3-8B, $20$ for Qwen3-8B, and $17.5$ for GPT-OSS-20B), a tiny fraction of the ambient SAE dimension ($d_2 = 16{,}384$) and well below the raw SAE-space values of $72$, $43$, and $78$ respectively. The per-cluster masks are not subsets of the global mask, so the three values describe different subspaces rather than successive restrictions of one another. Their relation is model-dependent: on LLaMA3-8B both masked subspaces are lower-dimensional than raw SAE space, with the per-cluster masks lowest (raw $72$, global $51$, per-cluster $22$); on Qwen3-8B the raw and global-mask dimensions are already close ($43$ and $41$), while the per-cluster masks give $20$; and on GPT-OSS-20B the global mask alone brings the dimension from $78$ to $18$, and the per-cluster masks give a similar median ($17.5$). The two additional model families in Appendix~\ref{app:effective-dim} (Figure~\ref{fig:effective-dim-appendix}) follow the LLaMA3 pattern, with per-cluster medians of $36$ (Qwen2-1.5B) and $40.5$ (Ministral-8B). Across these five families the local effective dimension is consistently a small fraction of the ambient SAE dimension, which is what makes per-cluster scoring well posed.

\section{Discussion and Limitations}
\label{sec:discussion}

\paragraph{Limits of the supervised assumption.} Supervised methods work only insofar as the labels cover the unsafe distribution at deployment, and the evidence across safety subfields suggests this coverage is rarely complete \citep{andriushchenko2024jailbreaking, mazeika2024harmbench}. The one-class framing trades a $0.04$--$0.08$ AUROC gap to the supervised probe on the $\geq$8B models (\Cref{tab:main-results}) for independence from this assumption, which is a reasonable trade in regimes where it is unreliable.

\begin{table}[tb]
\centering
\small
\caption{AUROC after per-cluster calibration fit from $1\%$ of the unsafe data and $1\%$ of the benign out-of-distribution data. Negatives are the held-out in-distribution safe points together with two novel benign domains absent from training (HellaSwag \citep{zellers2019hellaswag} and HH-RLHF-helpful \citep{bai2022training}); positives are the held-out unsafe points.
In each column, the best result among the calibrated unsupervised methods is in \textbf{bold} and the second best is \underline{underlined}. Means over $5$ seeds. Methods that score on a local low-dimensional subspace obtain near-optimal performance.
}

\label{tab:calib-auroc}
\setlength{\tabcolsep}{4pt}
\resizebox{\textwidth}{!}{%
\begin{tabular}{lcccccc}
\toprule
Method & Qwen2-1.5B & Ministral-8B & LLaMA3-8B & Qwen3-8B & GPT-OSS-20B & Gemma-4-26B \\
\midrule
\textsc{FreqMask-KM}, uncalibrated & 0.265 & 0.522 & 0.712 & 0.386 & 0.750 & 0.454 \\
\midrule
\textsc{FreqMask-KM}      & \textbf{0.986} & \underline{0.973} & \underline{0.978} & \textbf{0.985} & 0.969 & \underline{0.962} \\
\textsc{LearnedMask-LoRA} & \underline{0.985} & \textbf{0.990} & \textbf{0.987} & \underline{0.981} & \underline{0.982} & \textbf{0.976} \\
GMM                       & 0.953 & 0.969 & 0.935 & 0.967 & \textbf{0.985} & 0.901 \\
CVDD                      & 0.890 & 0.920 & 0.960 & 0.964 & 0.789 & 0.317 \\
BCE (supervised)          & 0.977 & 0.981 & 0.967 & 0.975 & 0.975 & 0.953 \\
\bottomrule
\end{tabular}%
}
\end{table}

\begin{figure}[!tb]
\centering
\includegraphics[width=0.7\textwidth]{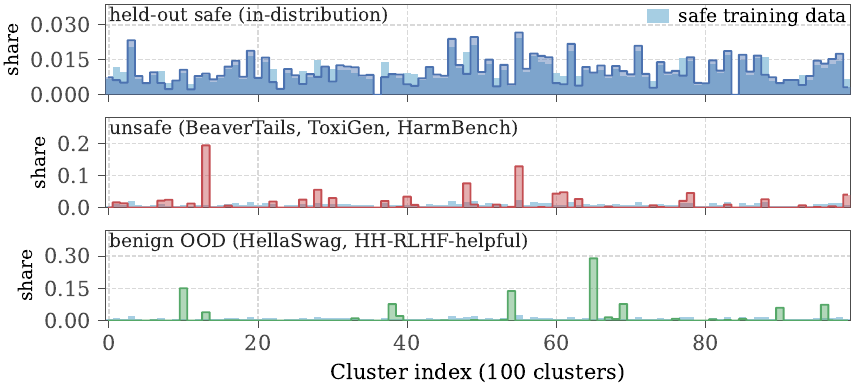}
\caption{Data distribution over clusters across held-out safe, unsafe,
and benign OOD data on LLaMA3-8B. Each panel
overlays the safe training data (light blue) with one other data source. The held-out safe data follows the
training occupancy closely, while unsafe data and novel benign data each
concentrate in a few clusters and in different clusters from one another.
}
\label{fig:cluster-occupancy}
\vspace{-3mm}
\end{figure}

\paragraph{Sensitivity to the safe distribution and calibration.}
Unsupervised detector assigns typicality score based on training representations of safe data. Consequently, benign regions of the input domain that are underrepresented during training may receive high anomaly scores regardless of their safety \citep{ruff2018deep}. Supervised detectors are also sensitive to such benign distribution shifts, while additionally relying on the unsafe training distribution remaining representative at deployment. The second requirement is arguably harder to satisfy, since unsafe inputs are open-ended and may evolve adversarially as new attack strategies emerge.

Importantly, the cluster structure learned from safe data makes it possible to correct the remaining sensitivity to benign shift using only a small labeled subset. This relies on empirical observation that unsafe data cluster assignment is highly non-homogeneous. That is, unsafe and novel benign data concentrate in small subsets of clusters, while safe data is spread more evenly (\Cref{fig:cluster-occupancy}). We employ this observation to design a simple correction method. Instead of modifying the learned representations (i.e., retraining), we adjust the scoring within each cluster via an affine transformation calibrated on a small labeled subset.
In particular, with $1\%$ of the unsafe data and $1\%$ of the benign OOD data, AUROC rises from $0.265$--$0.750$ to $0.962$--$0.990$ across the six models, and benign false-positive rate displays substantial improvement from $0.12$--$0.92$ to at most $0.041$ (\Cref{tab:calib-auroc,tab:calib-benign-fpr}). We note that this low-sample regime reflects a realistic scenario: while full unsafe and benign distributions are unavailable during deployment, it is feasible to obtain a small portion of each for calibration. We defer the remaining details and additional discussion to Appendix~\ref{app:calibration}.

\section{Theory: Generalization under Local Sparsity}
\label{sec:theory}

We study two questions separately: whether \textsc{FreqMask-KM} estimates
the safe geometry accurately, and whether this geometry retains the
information needed to distinguish safe from unsafe inputs. The first
question concerns estimation from safe data. The second also depends on
the unsafe distribution and cannot be answered from local sparsity alone. We defer the full elaborated proof to Appendix \ref{ss:Pf}.

\paragraph{Population and empirical detectors.}
Let $\PP_0$ and $\PP_1$ be the distributions of safe and unsafe inputs.
Fix the representation $z(x):=E(a^{(\ell)}(x))\in\R_+^{d_2}$ and let
$P_y$ be the distribution of $z(X)$ for $X\sim\PP_y$, $y\in\{0,1\}$.
We analyze an independent sample
$\mathcal D_n=(Z_1,\ldots,Z_n)\sim P_0^n$, with $K$ and $k$ fixed.
Thus, the result treats the SAE as given; it also applies conditionally
on an SAE fitted using an independent safe sample. Assume $\E_{P_0}\|Z\|_2^2<\infty$, and fix population $K$-means centers
\begin{equation}
\label{eq:population-kmeans}
 (\mu_1^\star,\ldots,\mu_K^\star)
 \in\argmin_{\mu_1,\ldots,\mu_K}
       \E_{P_0}\min_{c\in[K]}\|Z-\mu_c\|_2^2.
\end{equation}
Let $c^\star(z)$ be the nearest-center index, breaking ties by the
smallest index, and let $\pi_c=P_0(c^\star(Z)=c)$.
We assume $\pi_{\min}:=\min_c\pi_c>0$; then
$\mu_c^\star=\E[Z\mid c^\star(Z)=c]$.

The frequency mask is \emph{global}, as in
Algorithm~\ref{alg:freqmask-km}. Set $p_j=P_0(Z_j>0)$, let
$S^\star\subset[d_2]$ contain the $k$ largest $p_j$, and put
$M^\star=\operatorname{diag}(\mathbf 1_{S^\star})$.
For a target safe false-positive rate $q\in(0,1)$, define
\begin{align}
\label{eq:population-score}
 s^\star(z)&=\|M^\star(z-\mu_{c^\star(z)}^\star)\|_1,\\
\label{eq:population-threshold}
 \tau_q&=\inf\{t:P_0(s^\star(Z)\leq t)\geq1-q\},
 \qquad f_q^\star(z)=\mathbf 1\{s^\star(z)>\tau_q\}.
\end{align}
This is the population one-class detector, not the optimal safe/unsafe
classifier. Write $\widehat\mu_c$, $\widehat c_n$, and $\widehat M$ for
the fitted centers, nearest-center assignment, and empirical global
frequency mask. Their score is
\begin{equation}
\label{eq:empirical-score}
 \widehat s_n(z)=
 \|\widehat M(z-\widehat\mu_{\widehat c_n(z)})\|_1,
 \qquad
 \widehat f_{n,q}(z)=\mathbf 1\{\widehat s_n(z)>\tau_q\}.
\end{equation}
We first compare the scores at the same population threshold. Independent
safe-quantile calibration is treated separately in
Appendix~\ref{app:theory-calibration}, as a deployment procedure rather
than a claim about the experimental protocol.

\paragraph{Local geometry and clustering error.}
For each population cluster, let $J_c\subseteq S^\star$ be the coordinates
that have nonzero conditional variance. Put $s_{\max}=\max_c|J_c|\leq k$.
This is a uniform clusterwise measure of local sparsity, distinct from
the average neighborhood quantity $s_{\mathrm{loc}}$ in
\eqref{eq:local-sparsity}.
\begin{assumption}\label{Ass:local-safe-geometry}
For every $c\in[K]$, conditionally on $c^\star(Z)=c$,
\begin{equation}
\label{eq:local-energy}
 \E\|M^\star(Z-\mu_c^\star)\|_2^2\leq v^2,
 \qquad
 \|M^\star(Z-\mu_c^\star)\|_2\leq B\quad\text{almost surely}.
\end{equation}
\end{assumption}
Here $v^2$ bounds the masked within-cluster covariance trace.  Let $I_c=\{i:c^\star(Z_i)=c\}$, $N_c=|I_c|$, and let
$\overline\mu_c$ be the mean over $I_c$, with value zero when $N_c=0$.
After matching empirical and population cluster labels, define
\begin{equation}
\label{eq:clustering-errors}
 \rho_n=\max_c\|M^\star(\widehat\mu_c-\overline\mu_c)\|_1,
 \qquad
 \eta_{0,n}=P_0(\widehat c_n(Z)\neq c^\star(Z)\mid\mathcal D_n).
\end{equation}
The matching is any fixed measurable rule, chosen from the training
sample. These are clustering errors, not assumptions that the errors
vanish: $\rho_n$ measures the effect of fitting the partition on the
centroids, and $\eta_{0,n}$ measures changed assignments of fresh safe
points. The theorem keeps both terms explicit.

\begin{assumption}\label{Ass:frequency-and-margin}
For $1\leq k<d_2$, the ordered frequencies satisfy
$\gamma:=p_{(k)}-p_{(k+1)}>0$. Moreover, for some $L_0<\infty$,
\begin{equation}
\label{eq:safe-anticoncentration}
 P_0(|s^\star(Z)-\tau_q|\leq t)\leq L_0t,
 \qquad t\geq0.
\end{equation}
\end{assumption}
The frequency gap permits recovery of the mask. The second condition
limits the mass of safe points whose decisions can change after a small
perturbation of the score.

\begin{theorem}\label{Thm:One}
Suppose \Cref{Ass:local-safe-geometry,Ass:frequency-and-margin} hold.
Let $\delta\in(0,1)$ and let
\begin{equation}
\label{eq:mask-recovery}
 n\pi_{\min}\geq8\log\frac{4K}{\delta},
 \qquad n\geq8\gamma^{-2}\log\frac{4d_2}{\delta} \qquad \varepsilon_n=\rho_n+
 \sqrt{\frac{2s_{\max}}{n\pi_{\min}}}
 \left(v+B\sqrt{2\log\frac{4K}{\delta}}\right).
\end{equation}
Then, with probability at least $1-\delta$ over $\mathcal D_n$, it holds
\begin{equation}
\label{eq:main-estimation-bound}
 P_0(\widehat f_{n,q}(Z)\neq f_q^\star(Z)\mid\mathcal D_n)
 \leq\eta_{0,n}+L_0\varepsilon_n.
\end{equation}
\end{theorem}

\paragraph{Proof idea.}
Each population cluster contains at least $n\pi_c/2$ observations with
high probability. Its masked sample mean concentrates in $\ell_2$;
converting to $\ell_1$ costs only $\sqrt{s_{\max}}$ because the error
is supported on $J_c$. Uniform frequency concentration identifies the
global mask. Outside changed assignments, the scores then differ by
at most $\varepsilon_n$, and \eqref{eq:safe-anticoncentration} bounds
the probability of a changed decision.  For balanced clusters, the sampling term is of order
$\widetilde O(\sqrt{Ks_{\max}/n})$ when $L_0v,L_0B$ are bounded.
In particular, if $K,s_{\max}\lesssim s_{\mathrm{loc}}$, this becomes
$\widetilde O(s_{\mathrm{loc}}/\sqrt n)$. These are additional scaling
conditions, not consequences of pointwise sparsity. Mask recovery has
logarithmic dependence on $d_2$; the theorem does not bound SAE training
or the dimension dependence of $\rho_n$ and $\eta_{0,n}$.

\section*{Acknowledgment}
Xin Chen is supported by the Open Philanthropy AI Fellowship and the Vitalik Buterin Fellowship from the Future of Life Institute. The research received further support through ELSA (European Lighthouse on Secure and Safe AI) funded by the European Union under grant agreement No. 101070617 and the Swiss National Science Foundation under NCCR Automation, grant agreement 51NF40 180545.  Alexander Shevchenko was supported by the Swiss National Science Foundation (SNSF) Grant 204439. Gil Kur conducted the part of the work during his visit to the IDEAL Institute, hosted by Lev Reyzin, which was supported by NSF ECCS-2217023.

\bibliographystyle{plainnat}
\bibliography{references}

\newpage
\appendix
\section{Broader Impacts}
\label{app:broader-impacts}

The motivation for a one-class framing is that supervised safety methods
depend on having labels for the unsafe distribution at deployment, a
condition that is rarely fully met in practice
(Section~\ref{sec:discussion}). A detector that scores typicality with
respect to safe behavior, rather than similarity to a fixed catalog of
harms, can in principle flag previously unseen attack strategies and
out-of-domain inputs without requiring those specific examples to be
labeled in advance. To the extent the framework in this paper extends, it
offers a route to safety methods that degrade more gracefully as
adversarial techniques evolve, complementing rather than replacing
existing supervised pipelines.

The main risk we see runs in the other direction. The sensitivity to the
safe distribution discussed in Section~\ref{sec:discussion} can produce
false positives on legitimate benign inputs whose distribution differs
from the safe training set, and an operator may respond by loosening the
threshold, which simultaneously reduces detection of genuinely unsafe
inputs. Operators relying on this kind of detector in isolation, without
the calibration of Section~\ref{sec:discussion} or a comparable
safeguard against benign distribution shift, risk worse safety outcomes
than a well-calibrated supervised classifier on a stable distribution. We explicitly do not advocate for unsupervised
methods as standalone replacements for supervised safety systems.

\section{Per-Cluster Calibration from a Small Labeled Budget}
\label{app:calibration}

This section details the recalibration procedure used in
Section~\ref{sec:discussion}, gives the observation that motivates it,
and reports the full results across all six models.

\paragraph{Motivation.}
The framework of \Cref{sec:framework} is fit without supervision,
whereas the supervised probe of \Cref{sec:exp-main} requires enough
labels to train a classifier. Many deployments fall between these two
regimes: they can label a small sample, but not one large enough to
characterize the unsafe distribution. We use such a sample to rescale
the scores of an already-trained detector instead of training a
classifier on it. The procedure reads only precomputed scores and
cluster assignments, so the detector and its training are left
untouched and no GPU is required.

The rescaling is fit per cluster because the clusters differ from one
another in three respects, all measured on the deployed detector and
shown in
\Cref{fig:cluster-distributions,fig:cluster-scales} and
\Cref{fig:cluster-occupancy}.
First, the score is on a different scale in each cluster. On LLaMA3-8B
the single global threshold that costs $5\%$ false positives overall is
higher than every safe score in $71$ of the $98$ clusters that contain
at least ten safe points, and lower than the median safe score in $7$ of
them, so that one threshold flags no safe point in the first group and
more than half of the safe points in the second. Second, the clusters
differ in how well the within-cluster distance separates safe from
unsafe points, from near-perfect separation to none. Third, unsafe
points and benign out-of-distribution points are concentrated in
different clusters. To make this precise, we count for each cluster how
many unsafe evaluation points it contains, and then take clusters in
decreasing order of that count until they account for half of all unsafe
points. Between $2$ and $9$ clusters are enough, out of the $50$ to
$200$ the model has, and a large fraction of the remaining clusters
contain no unsafe point at all. Repeating the same count for the benign
out-of-distribution points takes even fewer clusters, between $1$ and
$3$. On every one of the six models, these two groups of clusters have
no cluster in common. This describes where each kind of point is
concentrated and not where it is allowed to occur: individual clusters,
including some in the two groups, do contain both kinds of point.
\Cref{fig:cluster-occupancy} shows this per cluster on LLaMA3-8B, and
\Cref{fig:cluster-occupancy-qwen2,fig:cluster-occupancy-ministral,fig:cluster-occupancy-qwen3,fig:cluster-occupancy-gptoss,fig:cluster-occupancy-gemma}
repeat the same view on the other five: held-out safe data follows the
training occupancy, while the two atypical groups each collapse onto a
few clusters. Lowering the score in the clusters that hold most of the
benign out-of-distribution points therefore need not lower it in the
clusters that hold most of the unsafe points.

The rescaling has to vary across clusters to have any effect on the
metrics we report. Both AUROC and recall at a fixed FPR depend on the
scores only through the ranking they induce. A single global map
$s \mapsto a_0 s + b_0$ with $a_0 > 0$ is strictly increasing, so it
leaves that ranking, and therefore both metrics, unchanged. Coefficients
that differ across clusters are what allow inputs assigned to different
clusters to be reordered relative to one another. A global fit is still
computed, but it enters the procedure as the shrinkage target
$(a_0, b_0)$ defined below, not as a calibration in its own right.

\begin{figure}[tb]
\centering
\includegraphics[width=\textwidth]{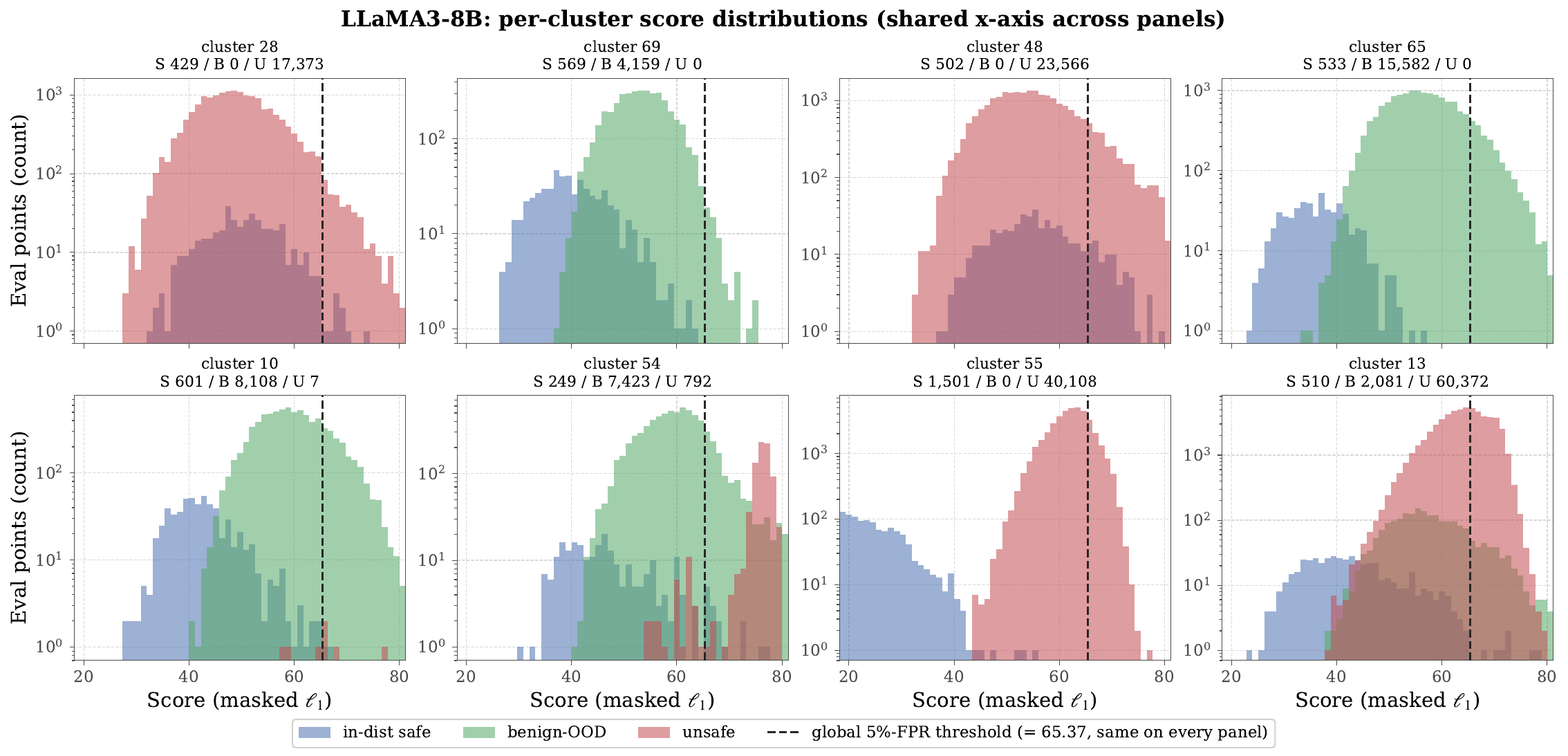}
\caption{Per-cluster score distributions on LLaMA3-8B for the three
groups: in-distribution safe, benign out-of-distribution, and unsafe.
Counts are on a log axis, since the smaller group is invisible at linear
scale (cluster~13 holds $60{,}372$ unsafe points against $510$ safe).
All panels use the same $x$-range, so scores are directly comparable
across panels. The dashed line is the \emph{same} global threshold in
every panel, the one costing $5\%$ false positives over all
in-distribution safe points. Its position within each cluster's score
distribution differs: in cluster~55 it lies in the middle of the unsafe
distribution, while in clusters~65 and~69 most benign
out-of-distribution points fall below it. Clusters are selected by a
fixed rule (the three largest by benign-OOD count, the three largest by
unsafe count, and the two largest containing no unsafe point,
deduplicated) and ordered by median score. The other five models behave
similarly; \Cref{fig:cluster-scales} reports the per-cluster score
spread for all six.}
\label{fig:cluster-distributions}
\end{figure}

\begin{figure}[tb]
\centering
\includegraphics[width=\textwidth]{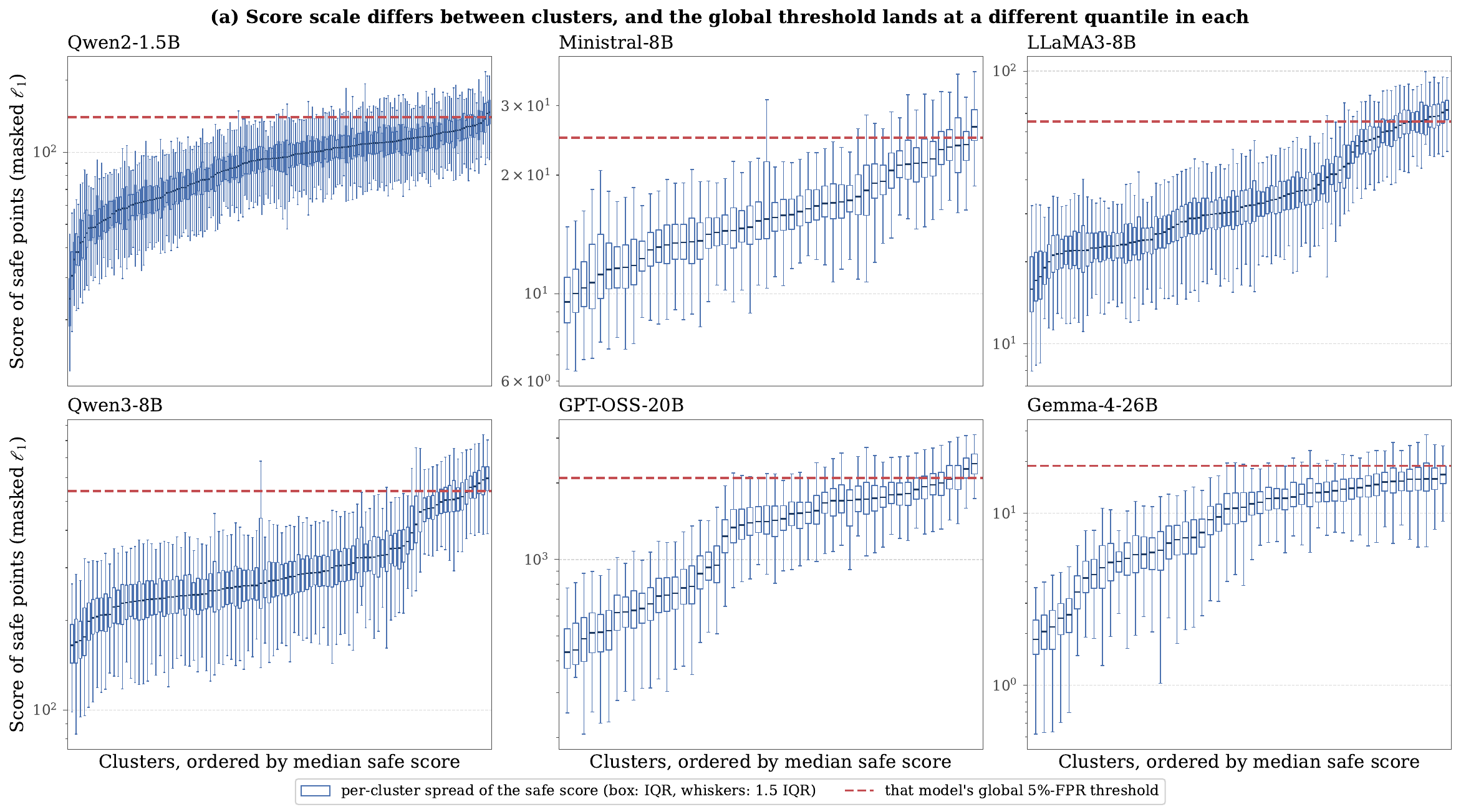}
\caption{Per-cluster spread of the score over in-distribution safe
points, for all six models, clusters ordered by median, with each
model's own global $5\%$-FPR threshold as a dashed line. The
per-cluster medians span roughly an order of magnitude within every
model, and the thresholds themselves differ by two orders of magnitude
across models ($18.9$ on Gemma-4-26B to $2094$ on GPT-OSS-20B), so no
score scale is shared either within or across models.}
\label{fig:cluster-scales}
\end{figure}

\begin{figure}[tb]
\centering
\includegraphics[width=0.92\textwidth]{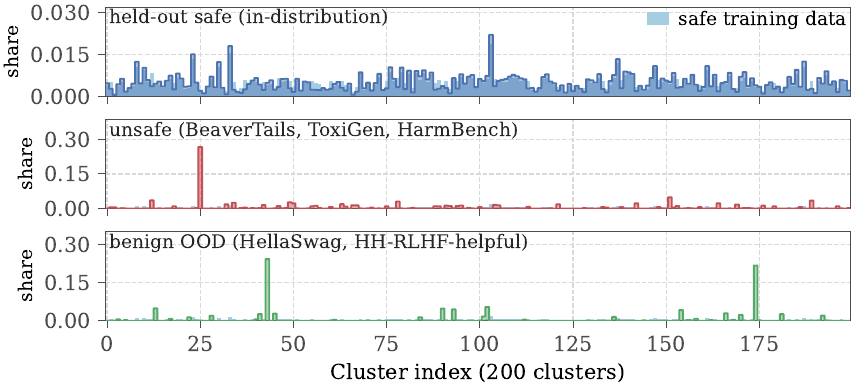}
\caption{Per-cluster occupancy on Qwen2-1.5B (200 clusters). Axes,
colors and panel order are as in \Cref{fig:cluster-occupancy}.
Total variation distance from the training occupancy is
$0.11$ for held-out safe, $0.72$ for unsafe and $0.86$ for benign
out-of-distribution.}
\label{fig:cluster-occupancy-qwen2}
\end{figure}

\begin{figure}[tb]
\centering
\includegraphics[width=0.92\textwidth]{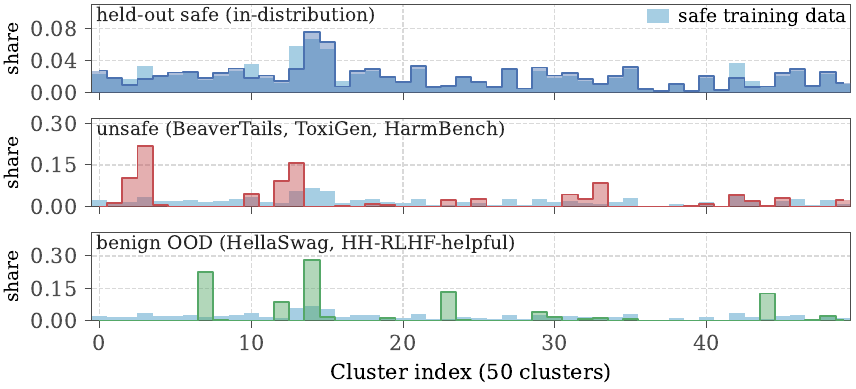}
\caption{Per-cluster occupancy on Ministral-8B (50 clusters). Axes,
colors and panel order are as in \Cref{fig:cluster-occupancy}.
Total variation distance from the training occupancy is
$0.10$ for held-out safe, $0.64$ for unsafe and $0.76$ for benign
out-of-distribution.}
\label{fig:cluster-occupancy-ministral}
\end{figure}

\begin{figure}[tb]
\centering
\includegraphics[width=0.92\textwidth]{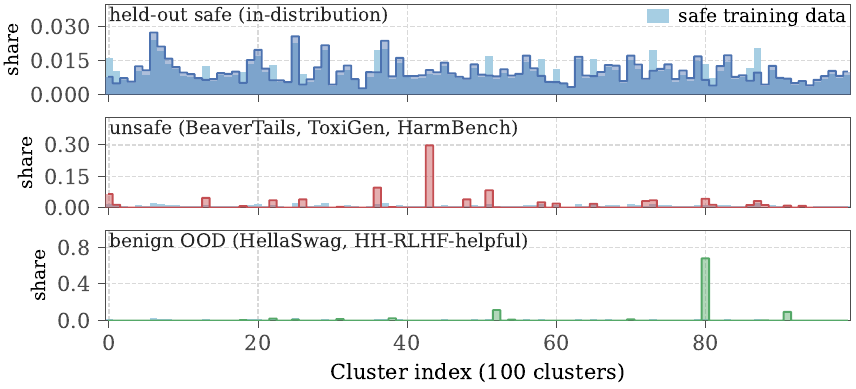}
\caption{Per-cluster occupancy on Qwen3-8B (100 clusters). Axes,
colors and panel order are as in \Cref{fig:cluster-occupancy}.
Total variation distance from the training occupancy is
$0.11$ for held-out safe, $0.73$ for unsafe and $0.90$ for benign
out-of-distribution.}
\label{fig:cluster-occupancy-qwen3}
\end{figure}

\begin{figure}[tb]
\centering
\includegraphics[width=0.92\textwidth]{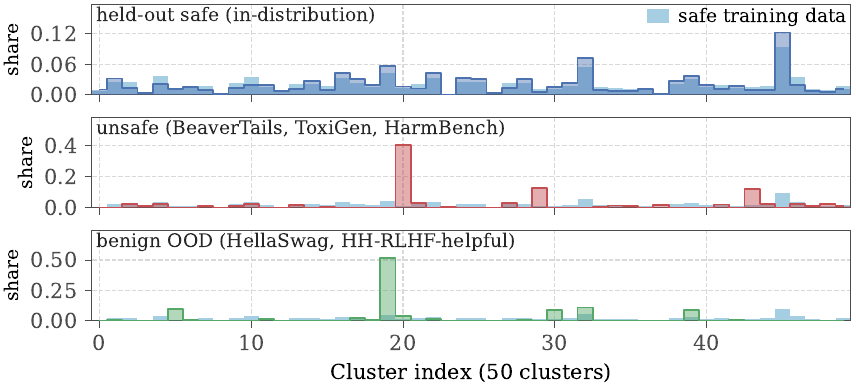}
\caption{Per-cluster occupancy on GPT-OSS-20B (50 clusters). Axes,
colors and panel order are as in \Cref{fig:cluster-occupancy}.
Total variation distance from the training occupancy is
$0.17$ for held-out safe, $0.67$ for unsafe and $0.78$ for benign
out-of-distribution.}
\label{fig:cluster-occupancy-gptoss}
\end{figure}

\begin{figure}[tb]
\centering
\includegraphics[width=0.92\textwidth]{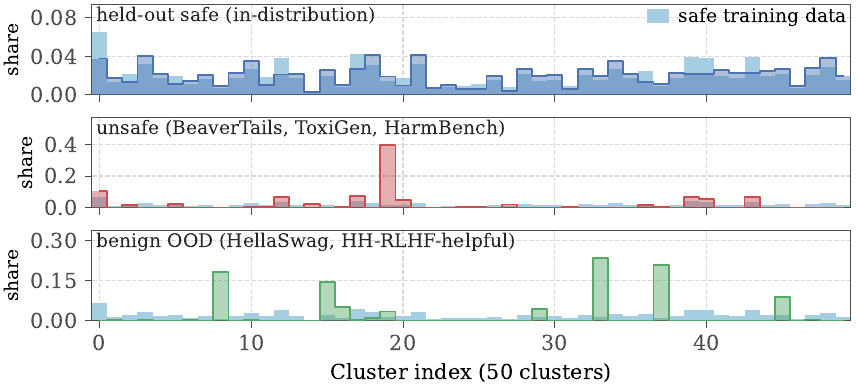}
\caption{Per-cluster occupancy on Gemma-4-26B (50 clusters). Axes,
colors and panel order are as in \Cref{fig:cluster-occupancy}.
Total variation distance from the training occupancy is
$0.17$ for held-out safe, $0.62$ for unsafe and $0.87$ for benign
out-of-distribution.}
\label{fig:cluster-occupancy-gemma}
\end{figure}

\paragraph{Setup.} A one-class detector routes each input $x$ to its
nearest latent cluster $c(x)$ and produces a score $s(x)$, the (masked)
distance to that cluster's centroid. At deployment a single global
threshold on $s(x)$ yields the safe/unsafe decision.

\paragraph{Labeled sample.} The calibration sample is drawn once per
run and contains three groups: a fraction $p$ of the unsafe inputs, the
same fraction $p$ of the benign out-of-distribution inputs, and half of
the in-distribution safe inputs. The safe pool is not budget-limited
because in-distribution safe data is abundant by
assumption, whereas
labeled unsafe examples and labeled examples of a novel benign domain
are equally costly to obtain, so they are drawn at a common rate. We refer to $p$ as the coupled budget; it is the horizontal axis of \Cref{fig:calib-benign-a,fig:calib-benign-b,fig:calib-benign-c}.
The inputs not drawn form the evaluation set, so the two are disjoint
and no calibration point is scored at evaluation time.

\paragraph{Calibration map.} We replace the raw score with a per-cluster
affine map of it,
\begin{equation}
g(x) = a_{c(x)}\, s(x) + b_{c(x)}, \qquad
\Pr(\text{unsafe}\mid x) = \sigma\!\left(g(x)\right),
\qquad \sigma(z) = \frac{1}{1 + e^{-z}},
\end{equation}
so that every cluster $c$ has its own coefficient pair
$(a_c, b_c) \in \mathbb{R}^2$. Benign out-of-distribution points enter
the fit with the safe label. Scoring uses $g(x)$ with a single global
threshold, and all reported metrics are threshold-free except the
false-positive rate defined below.

\paragraph{Importance weights.} The three groups are sampled at
different rates: a fraction $p$ of the unsafe points, the same fraction
$p$ of the benign out-of-distribution points, and one half of the
in-distribution safe points. The calibration sample therefore has a
different composition than the data the detector is applied to. At
$p = 1\%$ it holds one labeled unsafe point for every $100$ in the
unsafe pool, but one labeled safe point for every $2$ in the safe pool,
so unsafe points are under-represented in it by a factor of $50$
relative to safe points. An unweighted fit would place the coefficients,
and in particular the offsets $b_c$, where they suit that composition
rather than the full data. We therefore weight each point by the
reciprocal of its group's sampling fraction,
\begin{equation}
w_i =
\begin{cases}
1/p & \text{if $i$ is unsafe},\\
1/p & \text{if $i$ is benign out-of-distribution},\\
2   & \text{if $i$ is in-distribution safe},
\end{cases}
\end{equation}
so that each labeled point counts once for every point of its group that
it stands for. The weight $2$ on in-distribution safe points is that
same rule applied to a sampling fraction of one half: each labeled safe
point stands for itself and for one safe point that was not sampled.
With these weights the weighted group sizes in the calibration sample
equal the true group sizes, and the weighted sum in the objectives below
is an unbiased estimate of the same sum taken over the full data.

\paragraph{Global anchor.} The labeled sample is small and is divided
over $50$ to $200$ clusters, so many clusters receive few labeled points
and some receive none. Fitting every cluster independently on its own
points is therefore not possible: a cluster whose labeled points are all
safe or all unsafe is linearly separable and has no finite logistic fit,
and a cluster with no labeled point has nothing to fit at all. We
introduce $(a_0, b_0)$ as the estimate to fall back on in these cases.
It is obtained by weighted logistic regression over \emph{all} labeled
points pooled together, ignoring cluster membership:
\begin{equation}
(a_0, b_0) = \argmin_{a, b} \;
\sum_{i} w_i \, \ell\!\left(y_i,\, a s_i + b\right)
+ \varepsilon \left( a^2 + b^2 \right),
\end{equation}
where $y_i \in \{0,1\}$ labels input $i$ as safe (including benign-OOD)
or unsafe and $\ell$ is the logistic loss,
\begin{equation}
\ell(y, z) = -\Bigl[\, y \log \sigma(z) + (1 - y) \log\bigl(1 - \sigma(z)\bigr) \Bigr].
\end{equation}
The term $\varepsilon\,(a^2 + b^2)$ with $\varepsilon = 10^{-3}$ is a
small ridge that keeps the minimizer finite when the labeled sample is
close to linearly separable.

Two properties make this pair a suitable fallback. It is estimated from
every labeled point at once, so it remains well determined at budgets at
which individual clusters hold almost no labeled data. It is also the
best affine map available to a procedure that may not use cluster
membership, so taking it as the target of the
penalty in the next paragraph states the assumption that a cluster
behaves like the pooled data unless its own labeled points show
otherwise. Shrinking toward $(0, 0)$ instead, which is what a plain
ridge penalty would do, would drive $a_c$ and $b_c$ toward zero and
discard the score entirely in exactly those clusters that hold too
little labeled data to estimate anything. Note that $(a_0, b_0)$ is
never used to score inputs on its own: as shown above, a single global
affine map leaves AUROC and recall unchanged.

\paragraph{Per-cluster estimation.} Each cluster then minimizes the same
weighted loss, restricted to the points assigned to that cluster, with
an $\ell_2$ penalty pulling its coefficients toward the anchor,
\begin{equation}
(a_c, b_c) = \argmin_{a, b} \;
\sum_{i \,:\, c(x_i) = c} w_i \, \ell\!\left(y_i,\, a s_i + b\right)
+ \lambda \left\lVert (a, b) - (a_0, b_0) \right\rVert_2^2 ,
\end{equation}
initialized at $(a_0, b_0)$. The likelihood term is a sum over the
labeled points assigned to cluster $c$ while $\lambda$ is fixed, so the
penalty dominates the objective for clusters holding few labeled points
and the likelihood dominates for clusters holding many; a cluster with
no labeled points retains the anchor exactly. The estimator therefore
interpolates between the global map and an unconstrained per-cluster fit
as a function of how many labeled points each cluster receives. We
select $\lambda \in \{1, 10, 100\}$ by cross-validation within the
labeled sample, scoring folds by AUROC. Note that a negative $a_c$ is
permitted, which is what allows the procedure to repair clusters in
which the raw score ranks unsafe points as \emph{closer} to the centroid
than safe ones; such clusters exist on every model
(\Cref{fig:cluster-scales}).

\paragraph{Evaluation protocol.} The test pool is fixed and always
contains all three groups, so the numbers before and after calibration
are computed on identical data. Negatives are the held-out
in-distribution safe points
together with the held-out benign out-of-distribution points
(HellaSwag, $10{,}042$ examples, and HH-RLHF-helpful, $43{,}835$
examples); positives are the held-out unsafe pool. A detector that flags a novel benign domain is therefore
penalized by this metric.
Two quantities are reported. AUROC uses this pooled negative set, so it
is a strictly harder quantity than the AUROC of \Cref{tab:main-results},
whose negatives are in-distribution safe only; the two are not
comparable. The benign false-positive rate is the fraction of held-out
benign-OOD points scored above the threshold that costs $5\%$ false
positives on held-out in-distribution safe points alone, which is the
same operating point at which recall is reported, so recall and the
benign false-positive rate are computed at the same threshold. All
numbers are means over five seeds.

\paragraph{Results across models.}
\Cref{fig:calib-benign-a,fig:calib-benign-b,fig:calib-benign-c} report
every method on all six models as a function of the coupled budget. At
the largest budget we measure ($p = 0.02$), per-cluster calibration
raises AUROC from $0.267$--$0.750$ to $0.962$--$0.987$ for
\textsc{FreqMask-KM}, and from $0.308$--$0.476$ to $0.978$--$0.992$ for
\textsc{LearnedMask-LoRA}. Unlike the unsafe-only budget studied in
earlier versions of this experiment, the gain here is monotone in the
budget at the low end: even at the smallest budget we measure
($p = 3\times10^{-5}$, of the order of ten labeled examples of each
kind) the calibrated score is above the uncalibrated one on every model
and both detectors, reaching $0.730$--$0.892$ for \textsc{FreqMask-KM}
and $0.645$--$0.824$ for \textsc{LearnedMask-LoRA}. The reason is that
the uncalibrated detector is very poor on this test pool to begin with,
so even a coarse per-cluster offset estimated from a handful of points
is an improvement, whereas against in-distribution negatives alone the
uncalibrated detector is already strong and a noisy calibration can
reduce AUROC at the smallest budgets.

\Cref{tab:calib-recall} reports recall at the $5\%$ operating point,
the third quantity reported for this
experiment, alongside the AUROC of Section~\ref{sec:discussion}.

\begin{table}[tb]
\centering
\small
\caption{Recall at the threshold costing $5\%$ false positives on the
in-distribution safe points, at the coupled budget $p = 1\%$. Higher is
better. The first row is the deployed detector before calibration.}
\label{tab:calib-recall}
\setlength{\tabcolsep}{4pt}
\resizebox{\textwidth}{!}{%
\begin{tabular}{lcccccc}
\toprule
Method & Qwen2-1.5B & Ministral-8B & LLaMA3-8B & Qwen3-8B & GPT-OSS-20B & Gemma-4-26B \\
\midrule
\textsc{FreqMask-KM}, uncalibrated & 0.000 & 0.051 & 0.137 & 0.000 & 0.295 & 0.004 \\
\midrule
\textsc{FreqMask-KM}      & 0.970 & 0.840 & 0.913 & 0.978 & 0.777 & 0.681 \\
\textsc{LearnedMask-LoRA} & 0.957 & 0.977 & 0.981 & 0.956 & 0.902 & 0.837 \\
GMM                       & 0.531 & 0.882 & 0.743 & 0.819 & 0.894 & 0.402 \\
CVDD                      & 0.620 & 0.527 & 0.638 & 0.783 & 0.117 & 0.002 \\
BCE (supervised)          & 0.910 & 0.947 & 0.905 & 0.926 & 0.911 & 0.800 \\
\bottomrule
\end{tabular}%
}
\end{table}

\begin{table}[tb]
\centering
\small
\caption{False-positive rate on the held-out benign out-of-distribution domains, measured at the threshold costing $5\%$ false positives on the in-distribution safe points, at the coupled budget $p = 1\%$. Lower is better. The first row is the deployed detector before calibration. Means over $5$ seeds.}
\label{tab:calib-benign-fpr}
\setlength{\tabcolsep}{4pt}
\resizebox{\textwidth}{!}{%
\begin{tabular}{lcccccc}
\toprule
Method & Qwen2-1.5B & Ministral-8B & LLaMA3-8B & Qwen3-8B & GPT-OSS-20B & Gemma-4-26B \\
\midrule
\textsc{FreqMask-KM}, uncalibrated & 0.920 & 0.730 & 0.155 & 0.832 & 0.123 & 0.622 \\
\midrule
\textsc{FreqMask-KM}      & 0.002 & 0.012 & 0.015 & 0.003 & 0.041 & 0.026 \\
\textsc{LearnedMask-LoRA} & 0.000 & 0.002 & 0.005 & 0.007 & 0.007 & 0.004 \\
GMM                       & 0.068 & 0.001 & 0.011 & 0.010 & 0.002 & 0.009 \\
CVDD                      & 0.001 & 0.014 & 0.038 & 0.014 & 0.000 & 0.831 \\
BCE (supervised)          & 0.070 & 0.018 & 0.045 & 0.032 & 0.017 & 0.035 \\
\bottomrule
\end{tabular}%
}
\end{table}

\begin{figure}[tb]
\centering
\includegraphics[width=\textwidth]{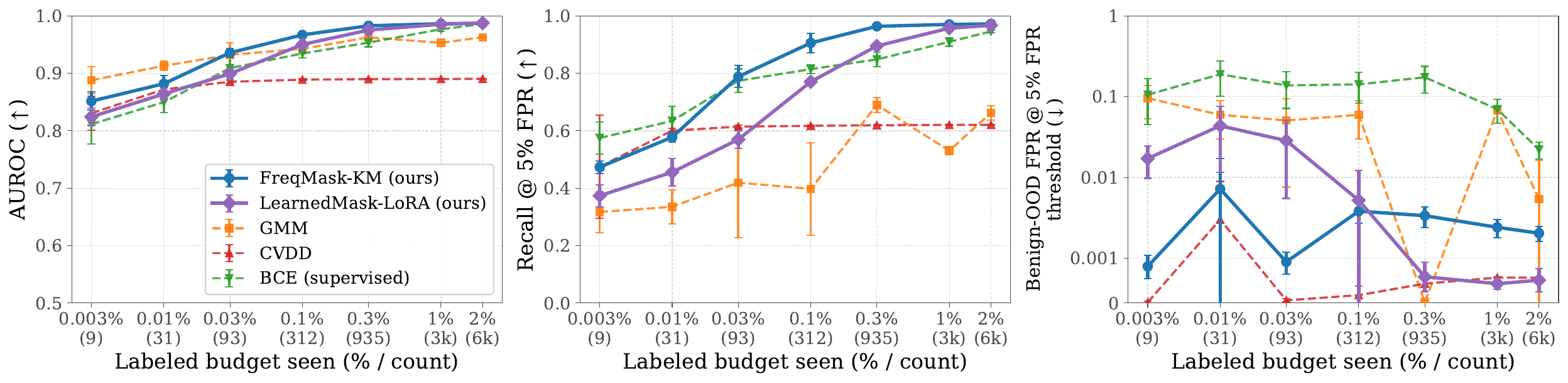}\\[0.4em]
\includegraphics[width=\textwidth]{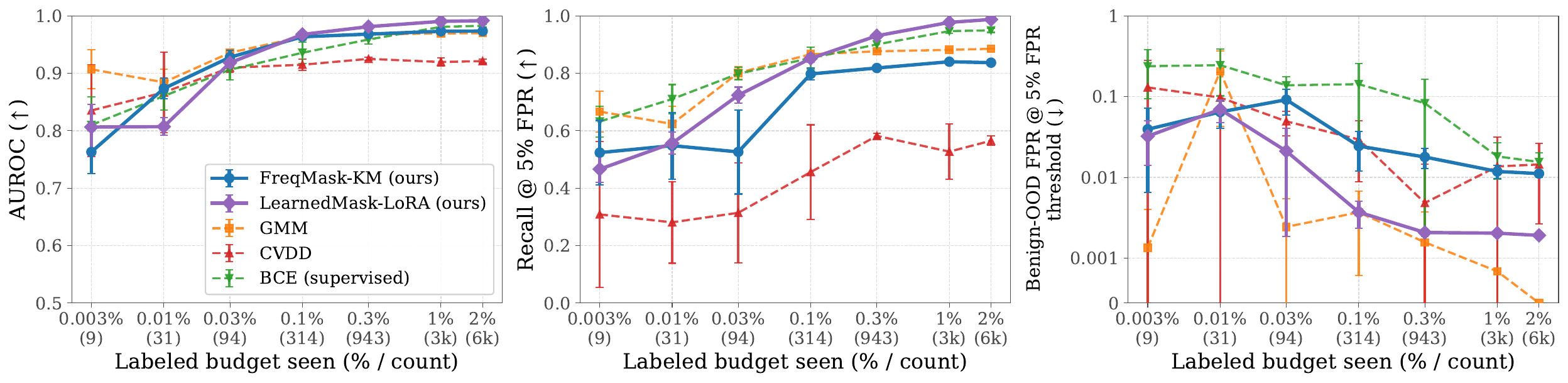}
\caption{Calibration against the coupled budget for Qwen2-1.5B (top) and
Ministral-8B (bottom). AUROC (left, negatives are in-distribution safe
together with the benign domains), recall at $5\%$ FPR (middle), and
benign out-of-distribution FPR at that same threshold (right, lower is
better, symmetric-log axis so exact zeros remain visible). Our methods
are solid, baselines dashed. Markers are means over $5$ seeds and bars
are $\pm$std; at the smallest budgets the sample is of the order of ten
labeled examples, so several baselines have a spread comparable to their
mean.}
\label{fig:calib-benign-a}
\end{figure}

\begin{figure}[tb]
\centering
\includegraphics[width=\textwidth]{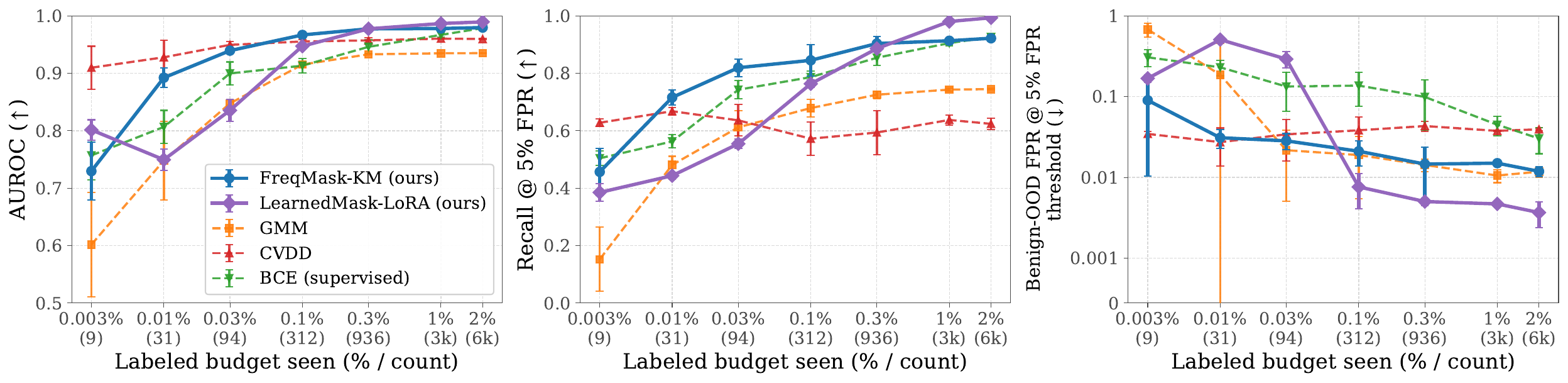}\\[0.4em]
\includegraphics[width=\textwidth]{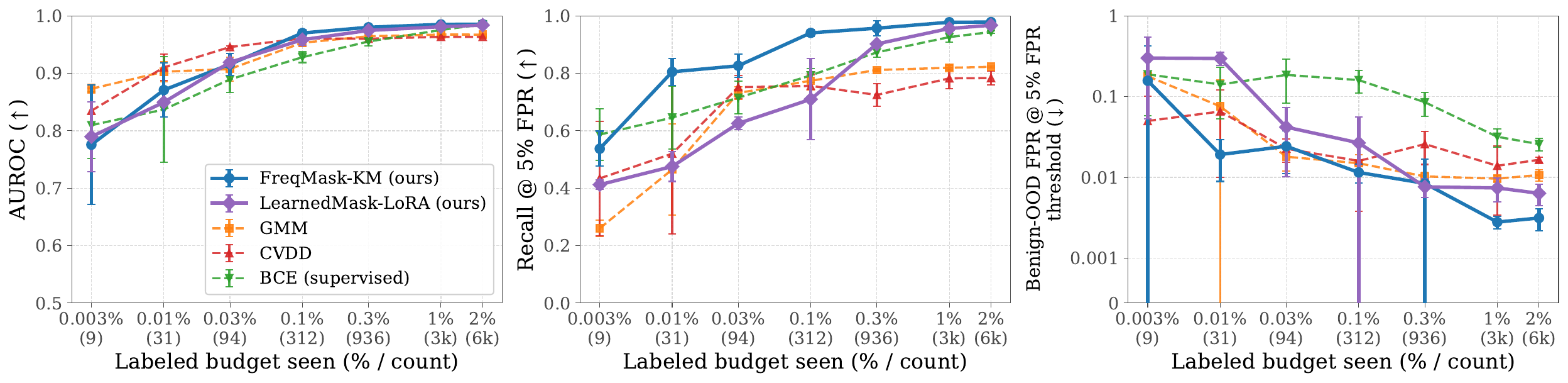}
\caption{Calibration against the coupled budget for LLaMA3-8B (top) and
Qwen3-8B (bottom). Axes, markers, and line styles are as in
\Cref{fig:calib-benign-a}.}
\label{fig:calib-benign-b}
\end{figure}

\begin{figure}[tb]
\centering
\includegraphics[width=\textwidth]{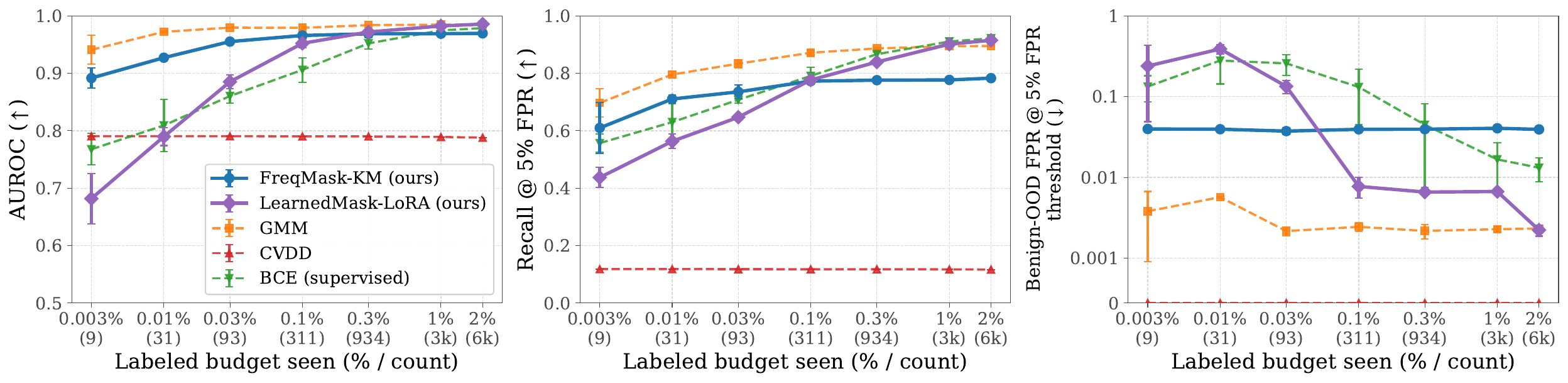}\\[0.4em]
\includegraphics[width=\textwidth]{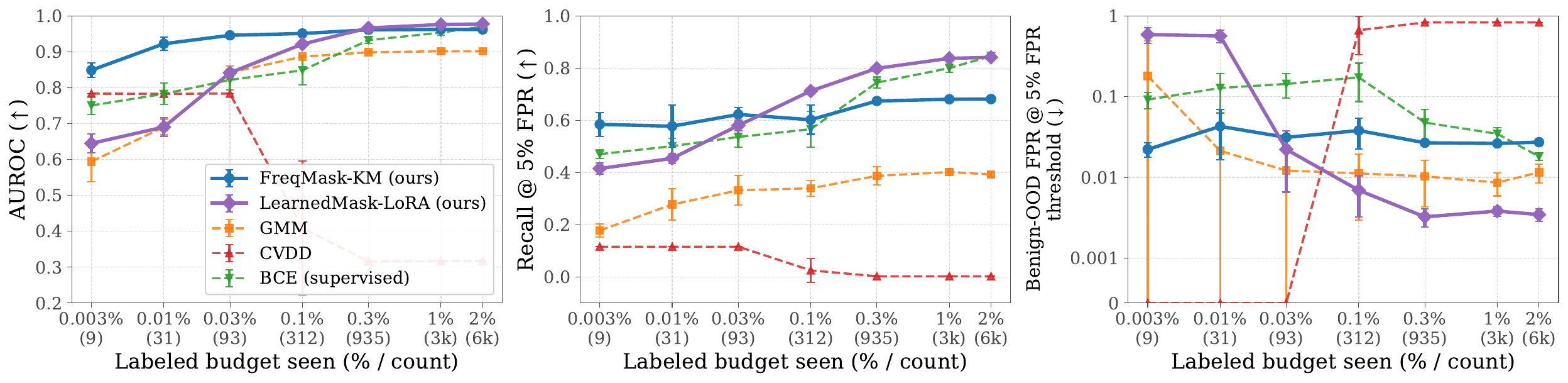}
\caption{Calibration against the coupled budget for GPT-OSS-20B (top)
and Gemma-4-26B (bottom). Axes, markers, and line styles are as in
\Cref{fig:calib-benign-a}. The CVDD curve on Gemma-4-26B falls below
the AUROC range of the other methods once the budget exceeds
$0.03\%$: its raw score is essentially uninformative on that model, so
the calibration has no signal to sharpen.}
\label{fig:calib-benign-c}
\end{figure}

\section{Related Work}
\label{sec:related-work}

\subsection{LLM Safety}
\label{sec:safety-background}

The term ``safety'' is used across several research communities to
refer to related but distinct concerns. The concerns
range from \emph{harmful content generation}, addressed by RLHF-based alignment \citep{dai2023safe,
wachi2024stepwise} and content-filtering systems
\citep{sharma2025constitutional}; to \emph{adversarial robustness}
against jailbreaks and prompt injections \citep{
mazeika2024harmbench,chen2026learning}; to
\emph{alignment properties} such as deception \citep{chen2025ai} and sycophancy \citep{sharma2023towards}; and to
\emph{application-domain boundaries}, where an input outside a
deployed system's intended domain is treated as unsafe even when it
is not harmful in absolute terms \citep{hong2024out}.

Each of these definitions has developed its own methodology, and in each case, the dominant paradigm is supervised or otherwise commits to a specific hypothesis about what unsafe inputs look like, whether through labeled training data \citep{park2025steer}, identified concept directions \citep{zhang2025jbshield}, or assumed brittleness under perturbation \citep{robey2023smoothllm}. Our contribution is to show that the same structure used by representation-based safety methods \citep{chen2025learning, zou2024improving} is exploitable under a framing that does not commit to any single definition of safety.

\subsection{One-Class Anomaly Detection and  Discussions in LLM}
\label{sec:ad-background-one-class}

One-class anomaly detection addresses a different problem from classification: Given samples from a distribution $p_{\text{safe}}$, the task is to score a new input by how typical it is under $p_{\text{safe}}$, without reference to any unsafe distribution. Classical approaches fall into different categories, such as density estimation (Gaussian mixtures \citep{nader2014mahalanobis}, kernel density estimation \citep{goyal2020drocc}) and distance-based scoring ($k$-nearest-neighbors \citep{nizan2024k}, variance reduction \citep{ruff2018deep, ruff2019self}). Applied to LLM safety, this framing has a natural appeal: an unsafe input is one that is distributionally atypical with respect to safe behaviors, regardless of the harm category.

An existing line of work in LLMs proposes scoring functions without labeled negative data, but targets hallucination and reliability rather than safety, and it is not a priori clear that scores tuned for the former transfer to the latter. Examples include Gaussian models on embedding distances for OOD detection \citep{ren2022out}, semantic entropy over sampled outputs for uncertainty \citep{kuhn2023semantic}, intrinsic-dimension estimators on activation neighborhoods for truthfulness \citep{yin2024characterizing}, and identifying hallucination subspaces from unlabeled generations \citep{du2024haloscope}.  \citet{feng2025maad} takes a first step toward bridging these literatures with safety by introducing an anomaly detection benchmark for LLM misalignment. However, it does not propose new algorithms or frameworks. By far, no existing work provides theoretical insight into what makes anomaly detection tractable in high-dimensional LLM representations, which is the question our work seeks to answer.

\subsection{SAE Features for Downstream Tasks}
\label{sec:related-sae-probing}

\citet{kantamneni2025sparse} study whether SAE features are a useful representation for supervised binary probing. Their probe uses labels on two classes and selects a single global set of features from class-conditioned differences. They find that SAE probes rarely outperform probes trained on the dense activations, except in specific regimes such as data scarcity, class imbalance, label noise, and covariate shift. Our setting is complementary. We use no unsafe labels at any point, so features cannot be selected from differences between the two classes; the mask in \textsc{FreqMask-KM} is chosen by activation frequency on safe data alone. Our claim is also different. We do not claim that SAE features are a good representation. We claim that the sparse support is \emph{local}, in the sense that it varies across neighborhoods of the concept space, and that this locality is what makes anomaly detection feasible in a $16{,}384$-dimensional feature space, both in the algorithms of \Cref{sec:framework} and in the sample-complexity bound of \Cref{sec:theory}, which scales only logarithmically in $d_2$. On the representation question, our unsupervised setting differs from their supervised one: moving to the SAE space provides a large gain, raising CVDD from $0.692$ to $0.806$ mean AUROC on the $\geq$8B models (\Cref{tab:decomposition}). Masking provides substantial additional gains on the largest models, adding $0.119$ AUROC on GPT-oss-20B and $0.090$ on Gemma-4-26B, while changing AUROC by at most $0.014$ on the other four models.

\section{Experimental Setup: Additional Details}
\label{app:exp-setup}

This section provides extended details on the SAE training procedure
(Appendix~\ref{app:sae-training}), datasets (Appendix~\ref{app:datasets}), per-algorithm
hyperparameters (Appendix~\ref{app:hyperparameters}), and compute budget
(Appendix~\ref{app:compute}).

\subsection{SAE Training}
\label{app:sae-training}

For each model, we train a Sparse Autoencoder on residual-stream activations
extracted at the analysis layer specified in Section~\ref{sec:exp-setup}
(layer~$20$ for Qwen~1.5B and Ministral, layer~$25$ for LLaMA3, layer~$30$
for Qwen3, layer~$20$ for GPT-oss, and layer~$25$ for Gemma-4-26B). Training data is
the same pooled safe distribution used by the downstream anomaly-detection
methods, comprising approximately $206$K last-token activations per model.
SAEs are trained independently per model; no parameters or features are shared
across models. For GPT-oss we train several SAEs at different sparsity weights
(\Cref{tab:sae-hparams}), as the SAE-space methods select different operating
points along the sparsity spectrum.

The SAE follows a standard architecture
mapping a residual-stream activation $a \in \mathbb{R}^{d_{\text{model}}}$
to a sparse feature vector $z = E(a) = \mathrm{ReLU}(W_E a + b_E) \in
\mathbb{R}^{d_{\text{SAE}}}$ and back via $\hat{a} = D(z) = W_D z + b_D$,
with a single hidden layer of size $d_{\text{SAE}} = 16384$. Training
minimizes a reconstruction loss with an $\ell_1$ sparsity penalty on the
hidden activations:
\begin{equation}
\mathcal{L}_{\text{rec}}(a) = \lVert a - D(E(a)) \rVert_2^2
+ \lambda_{\ell_1} \lVert E(a) \rVert_1.
\end{equation}
The value of $\lambda_{\ell_1}$ is not comparable across models, because
the penalty is applied to activations of different magnitude. We measure
that magnitude by the root mean square (RMS) of a residual-stream
activation, $\mathrm{RMS}(a) = \bigl(\tfrac{1}{d_{\text{model}}}
\sum_{j=1}^{d_{\text{model}}} a_j^2\bigr)^{1/2}$, which differs across
the six models by up to a factor of about $34$. Since the reconstruction term in $\mathcal{L}_{\text{rec}}$ is
quadratic in the scale of $a$ while the penalty term is only linear in
it, the penalty carries relatively less weight the larger that scale
is, so a model with larger RMS needs a larger $\lambda_{\ell_1}$ to
reach the same feature density. Training
hyperparameters are summarized in Table~\ref{tab:sae-hparams}.
\begin{table}[h]
\centering
\small
\caption{SAE training hyperparameters. Input dimension is the residual-stream
width of the corresponding base model; the SAE feature dimension is shared
across all models. The sparsity weight $\lambda_{\ell_1}$ is shared across the
original three models; GPT-oss and Gemma require larger values because their
residual activations have larger RMS ($\approx\!34\times$ for GPT-oss). For
GPT-oss we train at two operating points ($\lambda_{\ell_1}=300$ and $700$),
the $\lambda_{\ell_1}=700$ encoder being the one used by \textsc{FreqMask-KM};
Gemma uses a single intermediate value ($\lambda_{\ell_1}=3$).}
\label{tab:sae-hparams}
\setlength{\tabcolsep}{4pt}
\resizebox{\textwidth}{!}{%
\begin{tabular}{lcccccc}
\toprule
 & Qwen 1.5B & Ministral 8B & LLaMA3 8B & Qwen3 8B & GPT-oss 20B & Gemma 4 26B \\
\midrule
Input dimension $d_{\text{model}}$ & $1536$ & $4096$ & $4096$ & $4096$ & $2880$ & $2816$ \\
SAE feature dimension $d_{\text{SAE}}$ & $16384$ & $16384$ & $16384$ & $16384$ & $16384$ & $16384$ \\
Expansion factor & $\approx 10.7\times$ & $4\times$ & $4\times$ & $4\times$ & $\approx 5.7\times$ & $\approx 5.8\times$ \\
Sparsity weight $\lambda_{\ell_1}$ & $0.01$ & $0.01$ & $0.01$ & $0.01$ & $300$ / $700$ & $3$ \\
Epochs & $10$ & $5$ & $5$ & $20$ & $20$ & $10$ \\
Batch size & $256$ & $256$ & $256$ & $256$ & $256$ & $256$ \\
Optimizer & Adam & Adam & Adam & Adam & Adam & Adam \\
Learning rate & 1e-3 & 1e-3 & 1e-3 & 1e-3 & 1e-3 & 1e-3 \\
\bottomrule
\end{tabular}%
}
\end{table}

\subsection{Datasets}
\label{app:datasets}

\paragraph{Safe distribution.}
The safe training set pools two groups of data: (i) general capability and
instruction-following data (GSM8K, MathSFT40K, HumanEval, HumanEval+, MBPP,
Alpaca, MT-Bench, MMLU, ARC), and (ii) the safe portions of the evaluation
benchmarks (benign examples from BeaverTails and ToxiGen, and the model's
correct refusals on HarmBench attacks). The total is approximately
$206{,}000$ last-token activations per model. Counts are approximately the
same across models; minor variation arises from generation failures on
individual prompts.

\paragraph{Unsafe evaluation.}
We evaluate on three categories: BeaverTails (harmful-content requests),
ToxiGen (toxic and hateful language), and pooled HarmBench adversarial
attacks generated by eight algorithms (AutoDAN \citep{liu2023autodan}, AutoPrompt \citep{shin2020autoprompt}, DirectRequest \citep{mazeika2024harmbench},
GBDA \citep{guo2021gradient}, GCG \citep{zou2023universal}, HumanJailbreaks \citep{mazeika2024harmbench}, PEZ \citep{wen2023hard}, UAT \citep{wallace2019universal}). HarmBench suffixes are optimized
per-target-model, so the adversarial inputs differ across models;
the eight attacks are pooled into a single category for reporting. The
total unsafe evaluation set comprises approximately 314,000 samples
per model.

\begin{table}[t]
\centering
\small
\caption{Licenses for datasets used in this work. All datasets are publicly
available and cited at first mention in Section~\ref{sec:exp-setup}. For the
three benchmarks in the lower block, the safe portion contributes to the safe
training pool and the unsafe portion is used only for evaluation.}
\label{tab:dataset-licenses}
\begin{tabular}{lll}
\toprule
Dataset & License & Use \\
\midrule
GSM8K \citep{cobbe2021training}                & MIT          & Safe training \\
MathSFT40K \citep{mathsft40k}        & Not stated   & Safe training \\
HumanEval \citep{chen2021evaluating}         & MIT          & Safe training \\
HumanEval+ (EvalPlus) \citep{evalplus} & Apache 2.0 & Safe training \\
MBPP \citep{austin2021program}                 & CC-BY 4.0    & Safe training \\
Alpaca \citep{alpaca}              & CC-BY-NC 4.0 & Safe training \\
MT-Bench \citep{zheng2023judging}           & Apache 2.0   & Safe training \\
MMLU \citep{hendrycks2020measuring}              & MIT          & Safe training \\
ARC \citep{clark2018think}               & CC-BY-SA 4.0 & Safe training \\
\midrule
BeaverTails \citep{ji2023beavertails}       & CC-BY-NC 4.0 & Safe training \& unsafe eval \\
ToxiGen \citep{hartvigsen2022toxigen}       & MIT          & Safe training \& unsafe eval \\
HarmBench \citep{mazeika2024harmbench}      & MIT          & Safe training \& unsafe eval \\
\bottomrule
\end{tabular}
\end{table}

\paragraph{Splits.}
The safe set is split $0.8/0.2$: the $0.8$ portion is used to fit each method
(cluster centroids, mask selection, covariance estimation, etc.), and the
remaining $0.2$ is held out from fitting. Both AUROC and TPR@$5\%$FPR are
computed from the same ROC curve on the evaluation set, where the safe
evaluation samples supply the negatives and each unsafe category supplies the
positives; TPR@$5\%$FPR is the true-positive rate at the operating point whose
false-positive rate equals $5\%$. No separate threshold-calibration step is
performed.

\subsection{Activation-space baselines}
\label{app:hp-baselines}

All activation-space baselines operate directly on the residual-stream
activations extracted at the analysis layer (Section~\ref{sec:exp-setup}),
without any SAE encoding. We use our own implementations to keep the
training and scoring pipelines uniform across methods. The hyperparameters that were swept and the values selected for
the main results are listed in Tables~\ref{tab:hp-baselines-qwen},~\ref{tab:hp-baselines-ministral},~\ref{tab:hp-baselines-llama3},~\ref{tab:hp-baselines-qwen3},~\ref{tab:hp-baselines-gpt},~and~\ref{tab:hp-baselines-gemma}.

\paragraph{Mahalanobis.}
A single Gaussian is fit to the safe training set: the global mean
$\mu$ is estimated on safe samples, and the covariance $\Sigma$ is
accumulated on the full training pool, regularized by adding $\lambda I$
before inversion. The anomaly score is the Mahalanobis distance to the
global mean, $\sqrt{(x-\mu)^\top \Sigma^{-1} (x-\mu)}$, mapped to
$[0,1]$ either by linear interpolation between the 5th and 95th
percentiles of training distances (\textit{percentile}) or by a sigmoid
on standardized distances (\textit{sigmoid}). The two free choices are
the regularization strength $\lambda$ and the score type.

\paragraph{Gaussian Mixture Model (GMM).}
A diagonal-covariance Gaussian mixture is fit on standardized safe
features by EM, with components initialized via $k$-means++ or
uniformly at random. The anomaly score is the negative log-likelihood
$-\log p(x \mid \theta)$, normalized to $[0,1]$ by the same percentile
or sigmoid mapping used for Mahalanobis. The free choices are the
number of mixture components $n_{\text{components}}$, the
initialization, and the score normalization.

\paragraph{One-Class SVM (OC-SVM).}
We use a PyTorch implementation that solves the standard one-class
SVM dual on standardized safe features, with $\nu$ controlling the
upper bound on the training-error fraction (and equivalently the
lower bound on the support-vector fraction), and $\gamma$ controlling
the RBF bandwidth. The anomaly score is the negative decision
function (distance from the separating hyperplane in feature space).
The free choices are $\nu$, $\gamma$, and the kernel (RBF or linear).

\paragraph{Deep SVDD.}
Following \citet{ruff2018deep}, a small MLP encoder
$\phi: \mathbb{R}^{d_{\text{model}}} \to \mathbb{R}^{d_{\text{rep}}}$
is autoencoder-pretrained and then fine-tuned to map safe inputs near
a fixed center $c$ in representation space, by minimizing the mean
squared distance $\|\phi(x) - c\|_2^2$. The anomaly score is this
squared distance. The free choices are the representation dimension
$d_{\text{rep}}$, the deep one-class hyperparameter $\nu$ (used by
the soft-boundary objective), and the optimizer learning rate.

\paragraph{Context Vector Data Description (CVDD).}
CVDD generalizes Deep SVDD to multiple centers: an MLP encoder is
trained jointly with $n_{\text{contexts}}$ context vectors
$\{c_k\}_{k=1}^{n_{\text{contexts}}}$, with the per-sample loss given
by a temperature-annealed soft-minimum over distances to all centers,
$-T \log \sum_k \exp(-\|\phi(x) - c_k\|_2^2 / T)$. An orthogonality
regularizer $\|C C^\top - I\|_F^2$ encourages the centers to span
distinct directions. The anomaly score is the distance to the nearest
context vector. The free choices are the number of contexts
$n_{\text{contexts}}$ and the orthogonality weight; the temperature
schedule and pretraining flag are fixed across runs.

\subsection{Hyperparameters}
\label{app:hyperparameters}

For each method we list the hyperparameters that were swept and the values
selected for the main results. Selection is performed by AUROC.

\subsubsection{Activation-space baselines}
Hyperparameters are tuned per model rather than shared, since the analysis layer, residual-stream geometry, and sample sizes vary across models. Tables~\ref{tab:hp-baselines-qwen}, \ref{tab:hp-baselines-ministral}, \ref{tab:hp-baselines-llama3}, \ref{tab:hp-baselines-qwen3}, \ref{tab:hp-baselines-gpt}, and \ref{tab:hp-baselines-gemma} report swept ranges and selected values for each model.

\begin{table}[h]
\centering
\small
\caption{Hyperparameter sweeps and selected values for activation-space baselines on Qwen 1.5B.}
\label{tab:hp-baselines-qwen}
\begin{tabular}{llll}
\toprule
Method & Hyperparameter & Sweep range & Selected \\
\midrule
Mahalanobis
 & Regularization & $\{10^{-5}, 10^{-4}, 10^{-3}, 10^{-2}, 10^{-1}\}$ & $10^{-5}$ \\
 & Score type & \{percentile, sigmoid\} & sigmoid \\
\midrule
GMM
 & $n_{\text{components}}$ & $\{3, 5, 10, 50, 100\}$ & $50$ \\
 & Initialization & \{k-means++, random\} & k-means++ \\
 & Score normalization & \{percentile, sigmoid\} & percentile \\
\midrule
OC-SVM
 & $\nu$ & $\{0.01, 0.05, 0.1, 0.2\}$ & $0.01$ \\
 & $\gamma$ & $\{10^{-4}, 10^{-3}, 10^{-2}, 10^{-1}, 1\}$ & $10^{-2}$ \\
 & Kernel & \{RBF, linear\} & RBF \\
\midrule
Deep SVDD
 & Representation dim. & $\{16, 32, 64, 128\}$ & $128$ \\
 & $\nu$ & $\{0.05, 0.1, 0.2\}$ & $0.2$ \\
 & Learning rate & $\{10^{-4}, 10^{-3}, 10^{-2}\}$ & $10^{-3}$ \\
\midrule
CVDD
 & $n_{\text{contexts}}$ & $\{10, 50, 200, 500\}$ & $500$ \\
 & Orthogonality weight & $\{10^{-4}, 10^{-3}, 10^{-2}\}$ & $10^{-3}$ \\
\bottomrule
\end{tabular}
\end{table}

\begin{table}[h]
\centering
\small
\caption{Hyperparameter sweeps and selected values for activation-space baselines on Ministral 8B.}
\label{tab:hp-baselines-ministral}
\begin{tabular}{llll}
\toprule
Method & Hyperparameter & Sweep range & Selected \\
\midrule
Mahalanobis
 & Regularization & $\{10^{-5}, 10^{-4}, 10^{-3}, 10^{-2}, 10^{-1}\}$ & $10^{-5}$ \\
 & Score type & \{percentile, sigmoid\} & percentile \\
\midrule
GMM
 & $n_{\text{components}}$ & $\{3, 5, 10, 50, 100\}$ & $10$ \\
 & Initialization & \{k-means++, random\} & k-means++ \\
 & Score normalization & \{percentile, sigmoid\} & sigmoid \\
\midrule
OC-SVM
 & $\nu$ & $\{0.01, 0.05, 0.1, 0.2\}$ & $0.1$ \\
 & $\gamma$ & $\{10^{-4}, 10^{-3}, 10^{-2}, 10^{-1}, 1\}$ & $10^{-3}$ \\
 & Kernel & \{RBF, linear\} & RBF \\
\midrule
Deep SVDD
 & Representation dim. & $\{16, 32, 64, 128\}$ & $128$ \\
 & $\nu$ & $\{0.05, 0.1, 0.2\}$ & $0.2$ \\
 & Learning rate & $\{10^{-4}, 10^{-3}, 10^{-2}\}$ & $10^{-3}$ \\
\midrule
CVDD
 & $n_{\text{contexts}}$ & $\{10, 50, 200, 500\}$ & $50$ \\
 & Orthogonality weight & $\{10^{-4}, 10^{-3}, 10^{-2}\}$ & $10^{-3}$ \\
\bottomrule
\end{tabular}
\end{table}

\begin{table}[h]
\centering
\small
\caption{Hyperparameter sweeps and selected values for activation-space baselines on LLaMA3 8B.}
\label{tab:hp-baselines-llama3}
\begin{tabular}{llll}
\toprule
Method & Hyperparameter & Sweep range & Selected \\
\midrule
Mahalanobis
 & Regularization & $\{10^{-5}, 10^{-4}, 10^{-3}, 10^{-2}, 10^{-1}\}$ & $10^{-5}$ \\
 & Score type & \{percentile, sigmoid\} & percentile \\
\midrule
GMM
 & $n_{\text{components}}$ & $\{3, 5, 10, 50, 100\}$ & $10$ \\
 & Initialization & \{k-means++, random\} & k-means++ \\
 & Score normalization & \{percentile, sigmoid\} & percentile \\
\midrule
OC-SVM
 & $\nu$ & $\{0.01, 0.05, 0.1, 0.2\}$ & $0.1$ \\
 & $\gamma$ & $\{10^{-4}, 10^{-3}, 10^{-2}, 10^{-1}, 1\}$ & $10^{-2}$ \\
 & Kernel & \{RBF, linear\} & RBF \\
\midrule
Deep SVDD
 & Representation dim. & $\{16, 32, 64, 128\}$ & $128$ \\
 & $\nu$ & $\{0.05, 0.1, 0.2\}$ & $0.2$ \\
 & Learning rate & $\{10^{-4}, 10^{-3}, 10^{-2}\}$ & $10^{-3}$ \\
\midrule
CVDD
 & $n_{\text{contexts}}$ & $\{10, 50, 200, 500\}$ & $500$ \\
 & Orthogonality weight & $\{10^{-4}, 10^{-3}, 10^{-2}\}$ & $10^{-3}$ \\
\bottomrule
\end{tabular}
\end{table}

\begin{table}[h]
\centering
\small
\caption{Hyperparameter sweeps and selected values for activation-space baselines on Qwen3 8B.}
\label{tab:hp-baselines-qwen3}
\begin{tabular}{llll}
\toprule
Method & Hyperparameter & Sweep range & Selected \\
\midrule
Mahalanobis
 & Regularization & $\{10^{-5}, 10^{-4}, 10^{-3}, 10^{-2}, 10^{-1}\}$ & $10^{-5}$ \\
 & Score type & \{percentile, sigmoid\} & percentile \\
\midrule
GMM
 & $n_{\text{components}}$ & $\{3, 5, 10, 50, 100\}$ & $100$ \\
 & Initialization & \{k-means++, random\} & k-means++ \\
 & Score normalization & \{percentile, sigmoid\} & sigmoid \\
\midrule
OC-SVM
 & $\nu$ & $\{0.01, 0.05, 0.1, 0.2\}$ & $0.05$ \\
 & $\gamma$ & $\{10^{-4}, 10^{-3}, 10^{-2}, 10^{-1}, 1\}$ & $10^{-2}$ \\
 & Kernel & \{RBF, linear\} & RBF \\
\midrule
Deep SVDD
 & Representation dim. & $\{16, 32, 64, 128\}$ & $32$ \\
 & $\nu$ & $\{0.05, 0.1, 0.2\}$ & $0.1$ \\
 & Learning rate & $\{10^{-4}, 10^{-3}, 10^{-2}\}$ & $10^{-2}$ \\
\midrule
CVDD
 & $n_{\text{contexts}}$ & $\{10, 50, 200, 500\}$ & $50$ \\
 & Orthogonality weight & $\{10^{-4}, 10^{-3}, 10^{-2}\}$ & $10^{-4}$ \\
\bottomrule
\end{tabular}
\end{table}

\begin{table}[h]
\centering
\small
\caption{Hyperparameter sweeps and selected values for activation-space baselines on GPT-oss 20B.}
\label{tab:hp-baselines-gpt}
\begin{tabular}{llll}
\toprule
Method & Hyperparameter & Sweep range & Selected \\
\midrule
Mahalanobis
 & Regularization & $\{10^{-5}, 10^{-4}, 10^{-3}, 10^{-2}, 10^{-1}\}$ & $10^{-5}$ \\
 & Score type & \{percentile, sigmoid\} & sigmoid \\
\midrule
GMM
 & $n_{\text{components}}$ & $\{3, 5, 10, 50, 100\}$ & $10$ \\
 & Initialization & \{k-means++, random\} & random \\
 & Score normalization & \{percentile, sigmoid\} & sigmoid \\
\midrule
OC-SVM
 & $\nu$ & $\{0.01, 0.05, 0.1, 0.2\}$ & $0.1$ \\
 & $\gamma$ & $\{10^{-4}, 10^{-3}, 10^{-2}, 10^{-1}, 1\}$ & $10^{-4}$ \\
 & Kernel & \{RBF, linear\} & RBF \\
\midrule
Deep SVDD
 & Representation dim. & $\{16, 32, 64, 128\}$ & $16$ \\
 & $\nu$ & $\{0.05, 0.1, 0.2\}$ & $0.05$ \\
 & Learning rate & $\{10^{-4}, 10^{-3}, 10^{-2}\}$ & $10^{-4}$ \\
\midrule
CVDD
 & $n_{\text{contexts}}$ & $\{10, 50, 200, 500\}$ & $500$ \\
 & Orthogonality weight & $\{10^{-4}, 10^{-3}, 10^{-2}\}$ & $10^{-4}$ \\
\bottomrule
\end{tabular}
\end{table}

\begin{table}[h]
\centering
\small
\caption{Hyperparameter sweeps and selected values for activation-space baselines on Gemma 4 26B.}
\label{tab:hp-baselines-gemma}
\begin{tabular}{llll}
\toprule
Method & Hyperparameter & Sweep range & Selected \\
\midrule
Mahalanobis
 & Regularization & $\{10^{-5}, 10^{-4}, 10^{-3}, 10^{-2}, 10^{-1}\}$ & $10^{-5}$ \\
 & Score type & \{percentile, sigmoid\} & sigmoid \\
\midrule
GMM
 & $n_{\text{components}}$ & $\{3, 5, 10, 50, 100\}$ & $100$ \\
 & Initialization & \{k-means++, random\} & random \\
 & Score normalization & \{percentile, sigmoid\} & sigmoid \\
\midrule
OC-SVM
 & $\nu$ & $\{0.01, 0.05, 0.1, 0.2\}$ & $0.1$ \\
 & $\gamma$ & $\{10^{-4}, 10^{-3}, 10^{-2}, 10^{-1}, 1\}$ & $10^{-2}$ \\
 & Kernel & \{RBF, linear\} & RBF \\
\midrule
Deep SVDD
 & Representation dim. & $\{16, 32, 64, 128\}$ & $128$ \\
 & $\nu$ & $\{0.05, 0.1, 0.2\}$ & $0.1$ \\
 & Learning rate & $\{10^{-4}, 10^{-3}, 10^{-2}\}$ & $10^{-4}$ \\
\midrule
CVDD
 & $n_{\text{contexts}}$ & $\{10, 50, 200, 500\}$ & $500$ \\
 & Orthogonality weight & $\{10^{-4}, 10^{-3}, 10^{-2}\}$ & $10^{-4}$ \\
\bottomrule
\end{tabular}
\end{table}

\subsubsection{Framework instantiations (SAE feature space)}

All three framework instantiations share KMeans clustering and nearest-centroid
cluster assignment; they differ in how the per-cluster sparse subspace is
identified. The centroid is used for scoring only by the two \textsc{FreqMask-KM}
variants ($L_1$ distance to the assigned centroid); \textsc{LearnedMask-LoRA}
uses it only to assign a test point to a cluster and scores by the SAE
reconstruction residual under that cluster's learned mask.

\textsc{FreqMask-KM} only has two hyperparameters: Cluster count $N$ and mask size $k$. Selected values are listed in Table~\ref{tab:hp-framework}. The local version of it (Local \textsc{FreqMask-KM}) uses the same hyperaprameters, but the top-$k$ mask is applied locally by selecting the most frequently activated neurons.

\begin{table}[h]
\centering
\small
\caption{Selected hyperparameters for \textsc{FreqMask-KM}, swept over
$N \in \{50, 100, 200\}$ and $k \in \{25, 50, 100, 200, 400\}$. GPT-oss and
Gemma select small masks ($k=50$ and $k=25$) from the sweep extended to
smaller masks (Appendix~\ref{app:freq-abl-k}).}
\label{tab:hp-framework}
\setlength{\tabcolsep}{6pt}
\begin{tabular}{lcc}
\toprule
Model & $N$ (clusters) & $k$ (mask size) \\
\midrule
Qwen 1.5B    & $200$ & $100$ \\
Ministral 8B & $50$  & $200$ \\
LLaMA3 8B    & $100$ & $400$ \\
Qwen3 8B     & $100$ & $400$ \\
GPT-oss 20B  & $50$  & $50$  \\
Gemma 4 26B  & $50$  & $25$  \\
\bottomrule
\end{tabular}
\end{table}

\textsc{LearnedMask-LoRA} is trained in two stages, both operating
under the SAE space.

\paragraph{Stage 1: per-cluster mask training.}
A binary mask is learned per cluster using a straight-through
estimator (\texttt{mask\_type = binary\_ste}). Concretely, each
cluster's mask is parameterized by a vector of learnable logits
$\ell \in \mathbb{R}^{d_2}$ (one per SAE feature, initialized to
$0.5$); the forward pass applies the hard mask
$m = \mathbbm{1}[\ell > 0]$ to the SAE code via $m \odot z$, while
the backward pass replaces the non-differentiable indicator with
the derivative of $\sigma(\ell)$, so the logits can be trained
end-to-end with SGD despite $m$ being discrete~\citep{bengio2013estimating}.
In practice this is implemented with the standard detach
trick:
\begin{lstlisting}[language=Python]
soft = torch.sigmoid(logits)
hard = (logits > 0).float()
m = soft + (hard - soft).detach()  # forward: hard; backward: through soft
\end{lstlisting}
Sparsity is induced by adding $\lambda_{\text{sparsity}} \,
\|\sigma(\ell)\|_1$ to the reconstruction loss, with the coefficient
calibrated per cluster from the sweep $\{10^{-4}, 3\!\times\!10^{-4},
10^{-3}, 3\!\times\!10^{-3}, 10^{-2}, 3\!\times\!10^{-2}, 10^{-1}\}$.
We use $N = 100$ clusters with a minimum cluster size of $50$. The
clusters are created with an $L_2$ distance to centroids.

\paragraph{Stage 2: per-cluster LoRA adapter.}
A LoRA adapter is trained against an SAE-reconstruction objective on the
cluster's masked subspace. LoRA rank scales with cluster size:
$r = 2$ (small), $4$ (medium), $8$ (large). Training uses $30$ epochs at
learning rate $10^{-3}$ with sparsity weight $10^{-3}$ on the LoRA output.
Stage 2 is gated by a model-specific minimum cluster size: $500$ for
LLaMA3, $20$ for Ministral, $50$ for Qwen3, $50$ for Gemma, and disabled
for Qwen~1.5B and GPT-oss (Stage 1 mask only, no LoRA adaptation), as we
observe that LoRA hurts performance on those two models (Appendix~\ref{app:abl-lora}).

\subsubsection{Supervised reference}

The supervised reference reported in Section~\ref{sec:exp-main} is a linear
probe (with bias) trained on SAE features against binary safe/unsafe labels.
Several BCE variants were evaluated for completeness; their configurations
are summarized in Table~\ref{tab:hp-bce}.

\begin{table}[h]
\centering
\small
\caption{Supervised BCE variants and their 3-category performance across the
six models. AUROC and TPR@$5\%$FPR are each the mean of BeaverTails, ToxiGen,
and pooled-HarmBench. All variants use the BCE loss. The SAE + linear (with
bias) row is the \emph{linear probe} reported in Table~\ref{tab:main-results}
and gives its mean over $5$ random seeds (standard deviations are in
Table~\ref{tab:main-results}). The other variants are single-seed runs (best of
learning rate $\in \{10^{-3}, 10^{-4}\}$), reported here for completeness.}
\label{tab:hp-bce}
\setlength{\tabcolsep}{4pt}
\resizebox{\textwidth}{!}{%
\begin{tabular}{ll cccccc cccccc}
\toprule
 & & \multicolumn{6}{c}{AUROC ($\uparrow$)} & \multicolumn{6}{c}{TPR@$5\%$FPR ($\uparrow$)} \\
\cmidrule(lr){3-8}\cmidrule(lr){9-14}
Input & Classifier & Ministral & LLaMA3 & Qwen3 & GPT-oss & Gemma & Qwen 1.5B & Ministral & LLaMA3 & Qwen3 & GPT-oss & Gemma & Qwen 1.5B \\
\midrule
SAE features    & 2-layer MLP        & 0.993 & 0.959 & 0.992 & 0.987 & 0.978 & 0.971 & 0.966 & 0.902 & 0.958 & 0.938 & 0.937 & 0.887 \\
Raw activations & 2-layer MLP        & 0.993 & 0.979 & 0.992 & 0.988 & 0.969 & 0.975 & 0.966 & 0.913 & 0.962 & 0.939 & 0.915 & 0.802 \\
Raw activations & Linear             & 0.988 & 0.972 & 0.985 & 0.985 & 0.961 & 0.959 & 0.928 & 0.748 & 0.905 & 0.910 & 0.751 & 0.688 \\
SAE features    & Linear (with bias) & 0.993 & 0.972 & 0.993 & 0.990 & 0.983 & 0.982 & 0.964 & 0.905 & 0.957 & 0.948 & 0.942 & 0.876 \\
SAE features    & Linear (no bias)   & 0.985 & 0.965 & 0.979 & 0.973 & 0.947 & 0.954 & 0.898 & 0.693 & 0.847 & 0.858 & 0.635 & 0.664 \\
\bottomrule
\end{tabular}%
}
\end{table}

\subsection{Compute}
\label{app:compute}

All experiments are run on a SLURM-managed cluster with single NVIDIA V100 or Titan-RTX
GPU per job. Approximate wall-clock times per training run are listed in
Table~\ref{tab:compute}. We pre-cached the model activations from each dataset, and these methods load the cached activations for training.

\begin{table}[h]
\centering
\small
\caption{Approximate wall-clock per run on a single V100 or Titan-RTX.}
\label{tab:compute}
\begin{tabular}{lc}
\toprule
Method & Wall-clock per run \\
\midrule
Training SAE & $<5$ min \\
Mahalanobis, GMM, OC-SVM & $5$--$20$ min \\
Deep SVDD, CVDD & $\approx 30$ min \\
\textsc{FreqMask-KM}, Local \textsc{FreqMask-KM} & $\approx 10$ min \\
\textsc{LearnedMask-LoRA} & $\approx 60$ min \\
\bottomrule
\end{tabular}
\end{table}

\section{Effective Dimension: Definition and Per-Setting Specialization}
\label{app:effective-dim}

Let $X \in \mathbb{R}^{n \times d}$ denote the data matrix in the feature
space of interest, with $n$ samples and $d$ features. Let
$\bar{x} = \frac{1}{n} \sum_{i=1}^{n} x_i$ be the empirical mean, computed
independently within each setting's feature space. The empirical covariance is
\begin{equation}
    \Sigma = \frac{1}{n} X^\top X - \bar{x}\,\bar{x}^\top \in \mathbb{R}^{d \times d}.
\end{equation}
Let $\lambda_1 \geq \lambda_2 \geq \cdots \geq \lambda_d \geq 0$ denote its
eigenvalues in decreasing order; negative eigenvalues arising from
finite-precision arithmetic are clamped to zero. Define the cumulative
explained-variance ratio
\begin{equation}
    \rho(k) = \frac{\sum_{j=1}^{k} \lambda_j}{\sum_{j=1}^{d} \lambda_j}.
\end{equation}
The local effective dimension at the $90\%$ variance threshold is
\begin{equation}
    d_{90} = \min\!\left\{ k \in \{1, \dots, d\} : \rho(k) \geq 0.90 \right\}.
\end{equation}
Equivalently, $d_{90}$ is the smallest number of principal components
required to retain at least $90\%$ of the total variance in $X$.

\begin{figure}[!ht]
\centering
\includegraphics[width=0.67\textwidth]{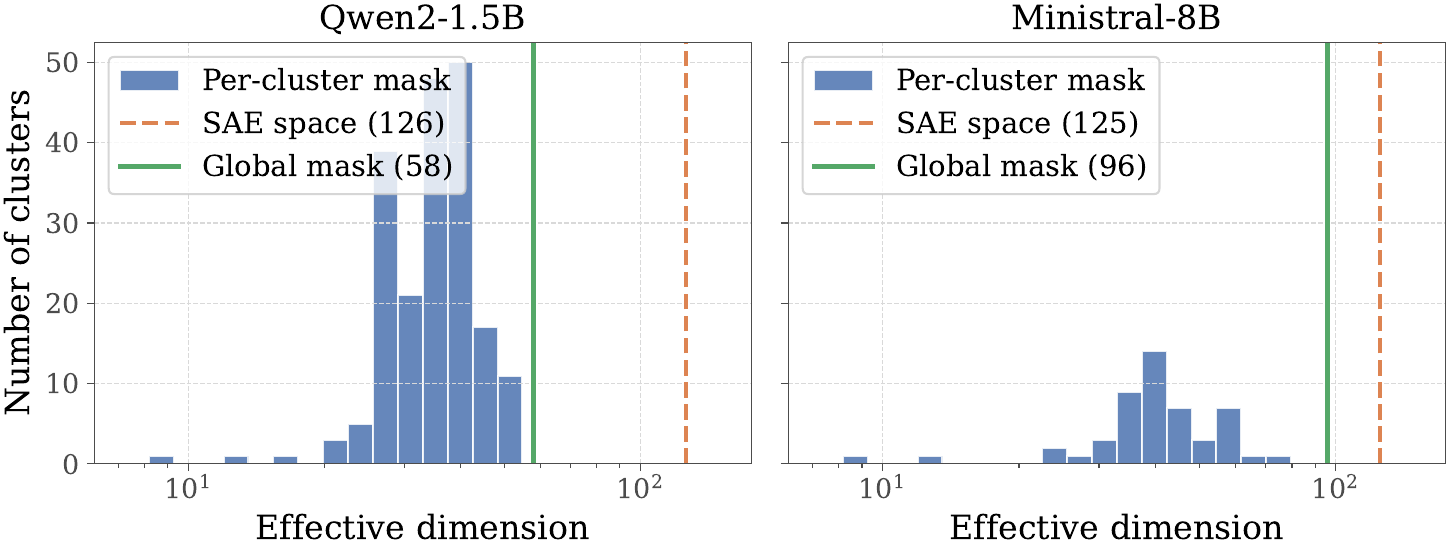}
\caption{Effective dimension ($d_{90}$) for two model families not shown
in the main text (Figure~\ref{fig:effective-dim}): Qwen2-1.5B and
Ministral-8B. We omit Gemma-4-26B because its selected mask has only
$k=25$ features, which already bounds its masked $d_{90}$ values by $25$.
As in the main text, blue bars show the per-cluster $d_{90}$ distribution,
computed on each cluster's members within its own top-$k$ mask
$\mathcal{M}_c$; vertical lines mark $d_{90}$ in the raw SAE feature space
(dashed orange) and the global-mask subspace (solid green). Per-cluster
$d_{90}$ again sits well below the global-mask and raw-SAE values,
confirming the low local effective dimension across the five model
families shown.}
\label{fig:effective-dim-appendix}
\end{figure}

The settings differ only in the choice of $X$ (and hence $d$):
\begin{itemize}
    \item \textbf{Raw SAE:} $X \in \mathbb{R}^{n \times d_{\text{SAE}}}$, rows
    are $z_i = \mathrm{ReLU}(W_{\text{enc}} h_i + b_{\text{enc}})$.
    \item \textbf{Global mask:} $X = Z_{:,\,\mathcal{M}} \in \mathbb{R}^{n \times k}$,
    where $\mathcal{M} \subset \{1, \dots, d_{\text{SAE}}\}$, $|\mathcal{M}| = k$,
    is the global mask's active feature index set.
    \item \textbf{Per-cluster local mask:} for cluster $c$ with members
    $\mathcal{I}_c = \{i : \ell_i = c\}$, define the local active set
    \begin{equation}
        \mathcal{M}_c = \mathrm{top}\text{-}k\!\left(
            \sum_{i \in \mathcal{I}_c} \mathbf{1}[z_i > 0]
        \right),
    \end{equation}
    i.e.\ the $k$ SAE features that fire most often within cluster $c$. Then
    $X^{(c)} = Z_{\mathcal{I}_c,\,\mathcal{M}_c} \in \mathbb{R}^{n_c \times k}$,
    and we report the distribution of $d_{90}^{(c)}$ over clusters with
    $n_c \geq \max(20, 2k)$ to avoid rank-limited estimates.
\end{itemize}

In all settings $d_{90} \leq \min(n - 1, d)$, with the per-cluster bound
$d_{90}^{(c)} \leq \min(n_c - 1, k)$. Centering is performed independently
within each setting's feature space; in particular, for the per-cluster
analysis, $\bar{x}^{(c)}$ is the mean over $\mathcal{I}_c$ within the local
masked subspace $\mathcal{M}_c$.

\section{Ablations: Stage-by-Stage Variations on \textsc{FreqMask-KM}}
\label{app:freqmask-km-ablations}

This section reports six ablations on \textsc{FreqMask-KM}
(\Cref{sec:framework-instantiations}, Appendix~\ref{app:hyperparameters})
targeting the design choices left open by the framework definition,
ordered to follow the stage sequence of \Cref{sec:framework}: the
clustering metric (Appendix~\ref{app:freq-abl-cluster}), the number of
clusters $\ncluster$ (Appendix~\ref{app:freq-abl-K}), the mask size
$\sparsity$ (Appendix~\ref{app:freq-abl-k}), the scoring metric
(Appendix~\ref{app:freq-abl-scoring}), the within-cluster
overlap of safe and unsafe active supports
(Appendix~\ref{app:freq-abl-overlap}), which gives a mechanistic reading of
the clustering-metric and scoring-metric results, and finally the
choice of extraction layer (Appendix~\ref{app:freq-abl-layer}).

\paragraph{Goal of these ablations.} The framework of
\Cref{sec:framework} decomposes the pipeline into four stages
(embedding, clustering, local subspace, scoring), each of which
admits several principled choices. The ablations below are intended
as a worked example of how those choices can be varied one stage at
a time within a single instantiation, and what the resulting
qualitative orderings look like. They are not a hyperparameter
optimization study: the goal is to compare design choices under a
fixed setting, not to identify a per-cell optimum. Two consequences
follow.

First, all runs in this section use a single seed. The only source
of randomness in \textsc{FreqMask-KM} is the KMeans cluster
assignment under uniform random centroid initialization
(\Cref{alg:freqmask-km}); the global frequency mask, centroid
estimation, and $\ell_1$ scoring are deterministic given the
clustering. The seed-to-seed standard deviation of \textsc{FreqMask-KM}
in \Cref{tab:main-results} is empirically zero at the precision
reported, and we do not expect multi-seed sweeps to change the
qualitative orderings reported below.

Second, the sweep grids are narrower than the per-model selection in
Appendix~\ref{app:hyperparameters}: the clustering and scoring ablations use
$\ncluster \in \{50, 100\}$ and $\sparsity \in \{100, 200, 400\}$,
the scaling ablations sweep one of the two over a wider range. The
grid omits $\ncluster = 200$, which is the value selected for
Qwen 1.5B in the main results, so the Qwen entries below peak at a
slightly lower AUROC than the main-table figure. We accept this gap
because the goal is the cross-method ordering at a fixed setting,
not the optimum for any one method.

Performance is reported as mean AUROC across the three evaluation
slices (BeaverTails, ToxiGen, pooled HarmBench adversarial attacks).

\subsection{Clustering Metric}
\label{app:freq-abl-cluster}

\paragraph{Setup.} The clustering stage in
\Cref{sec:framework-clustering} permits magnitude-aware and binarized
metrics. We compare three KMeans variants in SAE feature space:
$\rho_{L_2}(x, x') = \lVert z(x) - z(x') \rVert_2$ (the default in
\Cref{sec:framework-instantiations}),
$\rho_{L_1}(x, x') = \lVert z(x) - z(x') \rVert_1$, and
$\rho_{L_0}(x, x') = \lVert b(x) - b(x') \rVert_1$ (Hamming distance
on the binarized active-feature indicator $b(x)$ from
\Cref{sec:framework-clustering}). All other components are held fixed:
global frequency mask, $\ell_1$ centroid scoring, and a sweep over
$(\ncluster, \sparsity) \in \{50, 100\} \times \{100, 200, 400\}$
with the best cell reported per model. The $\rho_{L_2}$ row is each
model's default \textsc{FreqMask-KM} configuration
(Appendix~\ref{app:hyperparameters}); for GPT-oss this configuration uses a
smaller mask ($\sparsity = 50$) than the binarized-metric grid above.

\paragraph{Results.} \Cref{tab:freq-abl-cluster} reports the best
3-slice mean AUROC per clustering metric and per model. Two
observations.

First, $\rho_{L_2}$ is the best metric on four of the five $\geq$8B
models. The margin over the next-best metric ranges
from $0.10$ AUROC on Ministral down to $0.013$ on Qwen3 and $0.003$ on
GPT-oss. The exception is Gemma, where $\rho_{L_1}$ edges ahead
($0.776$ vs.\ $0.762$ for $\rho_{L_2}$, a $0.014$ gap). On Qwen 1.5B,
$\rho_{L_2}$ places third, behind $\rho_{L_0}$
(gap $0.014$) and $\rho_{L_1}$ (gap $0.003$).

Second, the binarized-clustering penalty is model-dependent.
$\rho_{L_0}$ depends on $z(x)$ only through $b(x)$ and so discards the
magnitude information in $z(x)$; on Ministral, LLaMA3, and Gemma it is the
worst metric, consistent with adversarial inputs being separable from
safe inputs primarily by activation magnitude on shared active
features rather than by the active support itself
(Appendix~\ref{app:freq-abl-overlap} presents direct evidence). On Qwen3 and
GPT-oss, by contrast, $\rho_{L_0}$ is the second-best metric (within
$0.013$ and $0.003$ of $\rho_{L_2}$): the magnitude information that
$\rho_{L_2}$ clustering exploits is less pivotal for these two models'
cluster assignments, even though magnitude still carries the
discriminative signal at scoring time.

\begin{table}[h]
\centering
\small
\caption{Clustering metric for \textsc{FreqMask-KM}, holding the
global frequency mask and $\ell_1$ scoring fixed. AUROC is the
3-slice mean; the $\rho_{L_2}$ row is each model's default
configuration and the other rows report the best of a
$(\ncluster, \sparsity)$ sweep over $\{50,100\} \times \{100, 200, 400\}$.
Single seed; bold marks the best metric per
column.}
\label{tab:freq-abl-cluster}
\setlength{\tabcolsep}{4pt}
\begin{tabular}{lcccccc}
\toprule
Metric & Qwen 1.5B & Ministral 8B & LLaMA3 8B & Qwen3 8B & GPT-oss 20B & Gemma 4 26B \\
\midrule
$\rho_{L_2}$ (default) & $0.745$          & $\mathbf{0.914}$ & $\mathbf{0.937}$ & $\mathbf{0.915}$ & $\mathbf{0.900}$ & $0.762$ \\
$\rho_{L_1}$           & $0.748$          & $0.817$          & $0.914$          & $0.882$          & $0.839$          & $\mathbf{0.776}$ \\
$\rho_{L_0}$           & $\mathbf{0.759}$ & $0.785$          & $0.860$          & $0.902$          & $0.897$          & $0.731$ \\
\bottomrule
\end{tabular}
\end{table}

\subsection{Number of Clusters}
\label{app:freq-abl-K}

\paragraph{Setup.} We sweep $\ncluster \in \{10, 25, 50, 100, 200,
500\}$ at \emph{each model's default mask size} $\sparsity$
(Qwen 1.5B $100$, Ministral $200$, LLaMA3 $400$, Qwen3 $400$,
GPT-oss $50$, Gemma $25$; \Cref{tab:hp-framework}), holding
$\rho_{L_2}$ clustering, global mask, and $\ell_1$ scoring fixed.
Fixing $\sparsity$ per model (rather than a single shared value) keeps
each model in the mask regime it was selected for: a shared mask far
above the small-mask models' optima collapses their adversarial slice
at every $\ncluster$, so a shared setting would understate those rows.

\paragraph{Results.} \Cref{tab:freq-abl-K} reports the resulting
3-slice mean AUROC. With each model held at its own mask size, all six
rows peak inside the sweep and fall off toward both ends, with
$\ncluster = 500$ uniformly the worst or near-worst entry. The peak
$\ncluster$ tracks the per-model selection in
Appendix~\ref{app:hyperparameters}: Qwen 1.5B at $25$, Ministral and Gemma at
$50$, LLaMA3 and Qwen3 at $100$, and GPT-oss at the small end ($10$--$25$,
$0.927$). The two small-mask models (GPT-oss, Gemma) are no longer
depressed---at their own $\sparsity$ they reach $0.927$ and $0.856$, in
line with their main-table \textsc{FreqMask-KM} values, rather than the
much lower figures a shared large mask would produce.
The Ministral row still shows large fluctuation outside its $\ncluster = 50$
peak, driven by the adversarial slice collapsing at
$\ncluster \in \{25, 100, 200, 500\}$; we do not have a mechanistic
explanation and read it as a sensitivity of Ministral's representation
geometry to $\ncluster$ rather than as noise.

\begin{table}[h]
\centering
\small
\caption{Number of clusters for \textsc{FreqMask-KM}, sweeping
$\ncluster$ at \emph{each model's default mask size} $\sparsity$
(Qwen 1.5B $100$, Ministral $200$, LLaMA3 $400$, Qwen3 $400$,
GPT-oss $50$, Gemma $25$; \Cref{tab:hp-framework}), with $\rho_{L_2}$
clustering, global mask, and $\ell_1$ scoring. AUROC is the 3-slice
mean. Single seed; bold marks the best $\ncluster$ per row. The
complementary no-mask analysis (mask disabled, $\ncluster$ swept) is in
\Cref{tab:freq-abl-nomask}.}
\label{tab:freq-abl-K}
\setlength{\tabcolsep}{6pt}
\begin{tabular}{lcccccc}
\toprule
$\ncluster$ & $10$ & $25$ & $50$ & $100$ & $200$ & $500$ \\
\midrule
Qwen 1.5B    & $0.745$ & $\mathbf{0.770}$ & $0.729$ & $0.745$ & $0.762$ & $0.720$ \\
Ministral 8B & $0.879$ & $0.767$ & $\mathbf{0.914}$ & $0.768$ & $0.774$ & $0.768$ \\
LLaMA3 8B    & $0.884$ & $0.916$ & $0.933$ & $\mathbf{0.936}$ & $0.898$ & $0.876$ \\
Qwen3 8B     & $0.880$ & $0.903$ & $0.848$ & $\mathbf{0.914}$ & $0.885$ & $0.880$ \\
GPT-oss 20B  & $\mathbf{0.927}$ & $0.926$ & $0.900$ & $0.875$ & $0.835$ & $0.832$ \\
Gemma 4 26B  & $0.841$ & $0.841$ & $\mathbf{0.856}$ & $0.801$ & $0.802$ & $0.783$ \\
\bottomrule
\end{tabular}
\end{table}

\paragraph{Does the mask do real work? A best-case no-mask comparison.}
Table~\ref{tab:freq-abl-K} varies $\ncluster$ with the global frequency
mask \emph{on}. To test whether the mask is load-bearing rather than
cosmetic, we give the no-mask condition its own best case: we disable the
mask entirely (\texttt{use\_masking=false}, so $\ell_1$ scoring runs over
\emph{all} SAE features with no top-$\sparsity$ restriction) and re-sweep
$\ncluster \in \{10, 25, 50, 100, 200, 500\}$, holding $\rho_{L_2}$
clustering and $\ell_1$ scoring fixed, so the comparison does not depend
on a cluster count chosen for the masked configuration. \Cref{tab:freq-abl-nomask} gives the full sweep; the best no-mask
entry per model, against the deployed masked configuration, is the
table discussed in \Cref{sec:exp-decomposition}.

\begin{table}[h]
\centering
\small
\caption{\textsc{FreqMask-KM} with the global mask \emph{disabled}
(\texttt{use\_masking=false}, $\ell_1$ over all SAE features), sweeping
$\ncluster$ with $\rho_{L_2}$ clustering and $\ell_1$ scoring. AUROC is
the 3-slice mean; single seed; bold marks the best $\ncluster$ per row.
``Masked best'' is each model's main-table \textsc{FreqMask-KM} AUROC
(\Cref{tab:main-results}), where the mask is on. These are the two
columns compared in \Cref{tab:mask-ablation}.}
\label{tab:freq-abl-nomask}
\setlength{\tabcolsep}{6pt}
\begin{tabular}{lcccccc!{\vrule}c}
\toprule
$\ncluster$ & $10$ & $25$ & $50$ & $100$ & $200$ & $500$ & Masked best \\
\midrule
Qwen 1.5B    & $0.769$ & $0.766$ & $0.747$ & $0.759$ & $\mathbf{0.776}$ & $0.736$ & $0.762$ \\
Ministral 8B & $0.875$ & $0.843$ & $\mathbf{0.904}$ & $0.864$ & $0.865$ & $0.859$ & $0.914$ \\
LLaMA3 8B    & $0.895$ & $0.919$ & $0.935$ & $\mathbf{0.937}$ & $0.902$ & $0.880$ & $0.937$ \\
Qwen3 8B     & $0.883$ & $0.904$ & $0.865$ & $\mathbf{0.914}$ & $0.896$ & $0.898$ & $0.915$ \\
GPT-oss 20B  & $\mathbf{0.781}$ & $0.780$ & $0.752$ & $0.748$ & $0.735$ & $0.745$ & $0.900$ \\
Gemma 4 26B  & $0.736$ & $0.718$ & $\mathbf{0.766}$ & $0.707$ & $0.737$ & $0.733$ & $0.856$ \\
\bottomrule
\end{tabular}
\end{table}

\subsection{Mask Size}
\label{app:freq-abl-k}

\paragraph{Setup.} We sweep $\sparsity \in \{25, 50, 100, 200, 400,
800, 1600\}$ at fixed $\ncluster = 100$, holding $\rho_{L_2}$
clustering, global mask, and $\ell_1$ scoring fixed.

\paragraph{Results.} In \Cref{tab:freq-abl-k},
Three of the five $\geq$8B models (Ministral, LLaMA3, Qwen3)
gain as $\sparsity$ grows and then saturate. Ministral rises
monotonically: $0.654 \to 0.817 \to 0.858$ across
$\sparsity = 25, 400, 1600$. LLaMA3 has a non-monotone region at small
$\sparsity$ ($0.805$ at $\sparsity=25$, dropping to $0.701$--$0.702$ at
$\sparsity \in \{50, 100\}$), then recovers and saturates at
$\sparsity \geq 200$ with AUROC $\approx 0.93$--$0.94$. Qwen3 collapses
at small $\sparsity$ ($0.38$ at $\sparsity \in \{25, 50\}$) but jumps
to $0.793$ at $\sparsity = 100$ and saturates at $0.915$ for
$\sparsity \geq 200$. GPT-oss and Gemma run in the opposite direction:
both peak at the smallest masks ($0.875$ at $\sparsity = 50$ for GPT-oss,
$0.820$ at $\sparsity = 50$ for Gemma) and decay as the mask grows---GPT-oss
to a plateau of $0.748$ for $\sparsity \geq 200$ and Gemma to $0.55$--$0.70$,
the two models for which a larger mask hurts. These small-mask peaks are
taken at the table's fixed $\ncluster = 100$; each model's main-table value
($0.900$ for GPT-oss, $0.856$ for Gemma; \Cref{tab:main-results}) is reached
only when the small mask is paired with $\ncluster = 50$, the joint
corner this $\ncluster = 100$ slice does not reach. Qwen 1.5B is essentially flat: AUROC fluctuates
within $0.717$--$0.759$ across the full range and shows no monotone
trend. For the models that benefit from a larger mask, the saturation
point ($\sparsity \approx 400$ on LLaMA3 and Qwen3,
$\sparsity \approx 800$ on Ministral) sets a soft upper bound on how
much the global mask helps before adding features stops contributing
discriminative signal.

\begin{table}[h]
\centering
\small
\caption{Mask size for \textsc{FreqMask-KM} at $\ncluster = 100$,
$\rho_{L_2}$ clustering, global mask, $\ell_1$ scoring. AUROC is the
3-slice mean. Single seed; bold marks the best
$\sparsity$ per row. Most $\geq$8B models favor a larger mask and then
plateau; GPT-oss is the exception, peaking at a small mask.}
\label{tab:freq-abl-k}
\setlength{\tabcolsep}{4pt}
\begin{tabular}{lccccccc}
\toprule
$\sparsity$  & $25$    & $50$    & $100$   & $200$   & $400$   & $800$   & $1600$ \\
\midrule
Qwen 1.5B    & $0.717$ & $0.741$ & $0.745$ & $0.734$ & $0.736$ & $\mathbf{0.759}$ & $\mathbf{0.759}$ \\
Ministral 8B & $0.654$ & $0.667$ & $0.729$ & $0.768$ & $0.817$ & $0.845$ & $\mathbf{0.858}$ \\
LLaMA3 8B    & $0.805$ & $0.701$ & $0.702$ & $0.929$ & $\mathbf{0.937}$ & $\mathbf{0.937}$ & $\mathbf{0.937}$ \\
Qwen3 8B     & $0.383$ & $0.378$ & $0.793$ & $\mathbf{0.915}$ & $\mathbf{0.915}$ & $\mathbf{0.915}$ & $\mathbf{0.915}$ \\
GPT-oss 20B  & $0.863$ & $\mathbf{0.875}$ & $0.832$ & $0.748$ & $0.748$ & $0.748$ & $0.748$ \\
Gemma 4 26B  & $0.801$ & $\mathbf{0.820}$ & $0.729$ & $0.549$ & $0.684$ & $0.697$ & $0.701$ \\
\bottomrule
\end{tabular}
\end{table}

\subsection{Scoring Metric: $\ell_1$ vs.\ $\ell_2$ vs.\ Mahalanobis}
\label{app:freq-abl-scoring}

\paragraph{Setup.} \textsc{FreqMask-KM} as defined in
\Cref{sec:framework-instantiations} uses $\ell_1$ centroid distance
on the masked SAE features. We compare $\ell_1$, $\ell_2$, and a
per-cluster Mahalanobis distance on the same masked features, holding
$\rho_{L_2}$ KMeans clustering and the global frequency mask fixed.
The Mahalanobis variant fits a per-cluster covariance on safe samples
and adds a regularizer $\lambda I$ before inversion, swept over
$\lambda \in \{10^{-3}, 10^{-2}, 10^{-1}\}$. For each scoring rule we
sweep $(\ncluster, \sparsity) \in \{50, 100\} \times \{100, 200, 400\}$
and report the best cell per model.

\paragraph{Results.} \Cref{tab:freq-abl-scoring} reports best 3-slice
mean AUROC per scoring rule on each model. $\ell_1$ and $\ell_2$ are
within $0.012$ AUROC of each other on every model except GPT-oss and
Gemma, where the gap widens to $0.032$ and $0.035$ respectively;
$\ell_1$ wins on Ministral, LLaMA3, and Gemma, while $\ell_2$ wins on
Qwen3, GPT-oss, and Qwen 1.5B. The two
magnitude-aware rules thus remain close, with the
choice between them model-dependent. Mahalanobis trails both
magnitude-aware rules on every $\geq$8B model except Gemma, where it is
in fact the best scoring rule ($0.785$, against $\ell_1$'s $0.762$).
Across models it benefits from a large swept regularizer---rising
monotonically to the largest value $\lambda = 10^{-1}$ on the other
models, and peaking one step earlier at $\lambda = 10^{-2}$ on Gemma
($0.785$ vs.\ $0.783$ at $10^{-1}$)---while at $\lambda = 10^{-3}$
AUROC drops to $0.56$--$0.81$ depending on model. This is
consistent with the per-cluster covariance being poorly conditioned on
the masked SAE features at the cluster sample sizes available on most
models, and substantiates the choice of a magnitude-aware $\ell_1$
scoring in \Cref{alg:freqmask-km} on those models; Gemma is the lone
exception, with a better-conditioned per-cluster covariance whose origin
we do not fully diagnose.

\begin{table}[h]
\centering
\small
\caption{Scoring metric for \textsc{FreqMask-KM}, holding $\rho_{L_2}$
KMeans clustering and the global frequency mask fixed. AUROC is the
3-slice mean; each cell is the best of a $(\ncluster, \sparsity)$
sweep over $\{50,100\} \times \{100, 200, 400\}$. The Mahalanobis row
additionally sweeps $\lambda \in \{10^{-3}, 10^{-2}, 10^{-1}\}$ and
selects the best $\lambda$ (the largest, $10^{-1}$, except Gemma, which
selects $10^{-2}$). Bold marks the best scoring rule per column.}
\label{tab:freq-abl-scoring}
\setlength{\tabcolsep}{4pt}
\begin{tabular}{lcccccc}
\toprule
Scoring & Qwen 1.5B & Ministral 8B & LLaMA3 8B & Qwen3 8B & GPT-oss 20B & Gemma 4 26B \\
\midrule
$\ell_1$        & $0.745$          & $\mathbf{0.914}$ & $\mathbf{0.937}$ & $0.915$          & $0.851$          & $0.762$ \\
$\ell_2$        & $\mathbf{0.749}$ & $0.902$          & $0.931$          & $\mathbf{0.917}$ & $\mathbf{0.883}$ & $0.727$ \\
Mahalanobis     & $0.743$          & $0.885$          & $0.848$          & $0.770$          & $0.860$          & $\mathbf{0.785}$ \\
\bottomrule
\end{tabular}
\end{table}

\subsection{Within-Cluster Active-Support Overlap of Safe and Unsafe Points}
\label{app:freq-abl-overlap}

\paragraph{Setup.} The framework assumes that within a cluster $C_c$,
safe samples occupy a small active-feature support, and the scoring
rule operates on the projection onto that support. The remaining
question is whether unsafe samples assigned to $C_c$ activate a
\emph{different} support from safe samples, or the \emph{same} support
at different magnitudes. We measure this directly. Let
$C_c^{\mathrm{safe}}$ and $C_c^{\mathrm{unsafe}}$ denote the safe and
unsafe samples assigned to cluster $c$ at evaluation time. For each
cluster, we define the population active set
\begin{equation}
S_c^{\mathrm{safe}} = \bigl\{ j \in [d_2] : \tfrac{1}{|C_c^{\mathrm{safe}}|} \!\!\sum_{x \in C_c^{\mathrm{safe}}}\!\! \mathbf{1}[z(x)_j > 0] > 0.5 \bigr\},
\qquad
S_c^{\mathrm{unsafe}} \text{ analogously.}
\end{equation}
The inner expression is the within-cluster firing frequency of SAE
feature $j$ on the safe population: the fraction of safe samples in
$C_c$ for which feature $j$ is active. The set $S_c^{\mathrm{safe}}$
collects the features that fire on a strict majority ($>50\%$) of
safe samples in the cluster; we treat these as the typical active
features for the safe population in that cluster.
$S_c^{\mathrm{unsafe}}$ is the analogous set for unsafe samples
assigned to the same cluster. The framework's local-sparsity
assumption is that $|S_c^{\mathrm{safe}}|$ is small, by construction
of the clustering; the question this section asks is whether
$S_c^{\mathrm{unsafe}}$ differs from $S_c^{\mathrm{safe}}$.

\paragraph{Results.} \Cref{tab:freq-abl-overlap} reports the
restricted Jaccard distribution. The median is $\geq 0.94$ on all six
models, and reaches $0.98$ on Qwen3: within a cluster, safe and unsafe
points populate nearly identical masked active supports. The anomaly
score therefore is not driven by which features are active in unsafe
inputs, but by how strongly they are active. This is the structural
reason that magnitude-aware scoring ($\ell_1$, $\ell_2$) beats the
binarized alternative in Appendix~\ref{app:freq-abl-scoring}: discarding the
magnitude information loses the discriminative signal. At the
clustering stage the magnitude penalty is more
model-dependent---$\rho_{L_0}$ is competitive on Qwen3 and
GPT-oss---as discussed in Appendix~\ref{app:freq-abl-cluster}.

\begin{table}[h]
\centering
\small
\caption{Within-cluster active-support overlap of safe and unsafe
points under the \textsc{FreqMask-KM} global mask. Counts are masked
features that fire in $>50\%$ of the corresponding population within
each cluster, summed across clusters; ``Jaccard median'' is the
per-cluster restricted Jaccard, taken over clusters with at least
one safe and one unsafe sample assigned.}
\label{tab:freq-abl-overlap}
\setlength{\tabcolsep}{4pt}
\begin{tabular}{lcccccc}
\toprule
Model & $\sparsity$ & Both & Safe only & Unsafe only & Neither & Jaccard median \\
\midrule
Qwen 1.5B    & $100$ & $98$  & $0$  & $0$ & $1$   & $1.00$ \\
Ministral 8B & $200$ & $139$ & $20$ & $2$ & $40$  & $0.94$ \\
LLaMA3 8B    & $400$ & $67$  & $3$  & $2$ & $329$ & $0.95$ \\
Qwen3 8B     & $400$ & $53$  & $1$  & $1$ & $345$ & $0.98$ \\
GPT-oss 20B  & $50$  & $36$  & $1$  & $2$ & $10$  & $0.95$ \\
Gemma 4 26B  & $25$  & $15$  & $1$  & $1$ & $8$   & $0.95$ \\
\bottomrule
\end{tabular}
\end{table}

\paragraph{Summary.} Reading the four ablations in stage order:
$\rho_{L_2}$ is the best clustering metric on four of the five $\geq$8B
models (Gemma is the exception, narrowly preferring $\rho_{L_1}$),
though its margin over the binarized
$\rho_{L_0}$ shrinks to $0.003$--$0.013$ on Qwen3 and GPT-oss; the best
$\ncluster$ and $\sparsity$ vary widely across models, within the
ranges $\ncluster \in \{10, 25, 50, 100\}$ and
$\sparsity \in \{25, \dots, 1600\}$, with no single setting dominating
and GPT-oss and Gemma preferring the small-$\ncluster$, small-$\sparsity$
corner that the others avoid; $\ell_1$ and $\ell_2$ scoring perform within
$0.012$ AUROC of each other on every model except GPT-oss ($0.032$) and
Gemma ($0.035$), with the better of the two model-dependent,
while per-cluster Mahalanobis trails both on every
$\geq$8B model except Gemma, and is
best at the largest swept regularizer on every model but Gemma (which
selects $10^{-2}$), broadly consistent with ill-conditioned per-cluster
covariance estimates on most models. The within-cluster Jaccard analysis in
Appendix~\ref{app:freq-abl-overlap} gives a direct mechanistic reason for why
magnitude-aware scoring ($\ell_1$, $\ell_2$) outperforms its binary
counterpart: safe and unsafe points within a cluster activate the same
masked features, and the discriminative signal is on magnitudes.

These orderings are specific to \textsc{FreqMask-KM}, and several of
them follow from choices made elsewhere in the pipeline rather than
from the stage in question alone: the $\ell_1$ scoring rule is paired
with a global frequency mask on SAE features whose entries are
nonnegative, the per-cluster covariance estimates that disadvantage
Mahalanobis are computed on a $\sparsity$-dimensional masked subspace
with $|C_c| / \sparsity$ in the low-tens range, and the $\ncluster$
peak depends on the safe sample size $N$ and the global sparsity $s$
(\Cref{sec:framework-clustering}). Other instantiations of the
framework, with different embedding spaces, mask constructions, or
learned per-cluster subspaces, may legitimately prefer different
choices at the same stage. \Cref{tab:decomposition} already shows
that \textsc{LearnedMask-LoRA} reaches the same performance band
under a different scoring rule (reconstruction residual rather than
$\ell_1$ centroid distance), and this section is intended as an ablation study of how to read the framework's design space, not as a
prescription that transfers to those instantiations unchanged.

\subsection{Extraction Layer}
\label{app:freq-abl-layer}

The main results read activations from a single mid-to-late layer per
model, chosen by the heuristic described in \Cref{sec:exp-setup}. To
test that choice, we sweep the extraction layer on three model
families, training a fresh SAE at each layer and running
\textsc{FreqMask-KM} with every other hyperparameter held at the
model's published setting. AUROC in this subsection is computed
against pooled HarmBench.

\begin{table}[h]
\centering
\small
\caption{\textsc{FreqMask-KM} AUROC against pooled HarmBench as a
function of the extraction layer. A fresh SAE is trained at each
layer; all other hyperparameters are held at each model's published
setting. Bold marks the layer used in the main results. The layer
used in the main results is the strongest of those swept on
LLaMA3-8B and GPT-oss-20B; on Gemma-4-26B the shallow layers score
higher.}
\label{tab:freq-abl-layer}
\setlength{\tabcolsep}{6pt}

\begin{tabular}{lcccccccc}
\toprule
\multicolumn{9}{c}{\textit{LLaMA3-8B}} \\
\midrule
Layer & $8$ & $12$ & $16$ & $20$ & $24$ & $\mathbf{25}$ & $28$ & $30$ \\
AUROC & $0.924$ & $0.913$ & $0.800$ & $0.802$ & $0.799$
      & $\mathbf{0.956}$ & $0.785$ & $0.921$ \\
\bottomrule
\end{tabular}

\vspace{4pt}

\begin{tabular}{lcccccccc}
\toprule
\multicolumn{9}{c}{\textit{GPT-oss-20B}} \\
\midrule
Layer & $4$ & $8$ & $12$ & $16$ & $18$ & $\mathbf{20}$ & $22$ & $23$ \\
AUROC & $0.924$ & $0.932$ & $0.927$ & $0.896$ & $0.924$
      & $\mathbf{0.966}$ & $0.842$ & $0.861$ \\
\bottomrule
\end{tabular}

\vspace{4pt}

\begin{tabular}{lcccccc}
\toprule
\multicolumn{7}{c}{\textit{Gemma-4-26B}} \\
\midrule
Layer & $8$ & $12$ & $16$ & $20$ & $\mathbf{25}$ & $29$ \\
AUROC & $0.940$ & $0.963$ & $0.934$ & $0.807$ & $\mathbf{0.828}$
      & $0.930$ \\
\bottomrule
\end{tabular}
\end{table}

On LLaMA3-8B and GPT-oss-20B the layer used in the main results gives
the highest AUROC of those swept ($0.956$ and $0.966$), with the
remaining layers spanning $0.785$--$0.924$ and $0.842$--$0.932$
respectively. On Gemma-4-26B the shallow layers are the strongest
($0.940$ at layer $8$ and $0.963$ at layer $12$) and the layer used in
the main results reaches $0.828$.

\paragraph{Takeaway.} The mid-to-late heuristic is a reasonable
default on all three models: it selects the best available layer on
two of them, and a workable one on the third. Which layer is best is
model-dependent, and it does not correspond to a fixed depth or a
fixed fraction of depth, so it appears to follow each model's
architecture rather than a rule that transfers across families. A
per-model sweep can therefore recover a small amount of additional
AUROC, at the cost of extra computations.

\subsection{SAE Sparsity Penalty}
\label{app:freq-abl-sae-lambda}

The framework operates in the feature space of a pretrained SAE, so its
performance could in principle depend on how that SAE was trained. We
test this by sweeping the SAE's main training hyperparameter, the
sparsity penalty $\lambda_{\ell_1}$ of Appendix~\ref{app:sae-training}, over
three orders of magnitude on four model families, and rerunning
\textsc{FreqMask-KM} unchanged on each resulting encoder: only the SAE
checkpoint changes, while $\ncluster$, $\sparsity$, the extraction
layer and every other setting stay at the model's published value.
\Cref{tab:freq-abl-sae-lambda} reports, for each encoder, the density
it reaches (the measured average number of active features per input,
$L_0$, out of $d_{\text{SAE}} = 16384$) together with the AUROC it
yields. AUROC here is the three-category mean of
\Cref{tab:main-results}, and the deployed rows reproduce that table.

Detection is stable over a wide range of penalties rather than
depending on one setting. On LLaMA3-8B, AUROC stays between $0.858$ and
$0.936$ while $\lambda_{\ell_1}$ moves by a factor of $1000$ and $L_0$
by a factor of seven, and Qwen2-1.5B stays between $0.643$ and $0.762$
at every setting, consistent with its representation quality rather than
its SAE bounding what unsupervised detection can recover on that model
(\Cref{sec:exp-main}). The clear losses occur where the encoder fails to
sparsify: the two Gemma-4-26B encoders trained for twenty epochs reach
$L_0 = 81$ and $103$ and fall to $0.627$ and $0.665$, well below the
other three settings on that model.

The penalty $\lambda_{\ell_1}$ does not determine sparsity on its own,
since training length also affects the density reached: the same
$\lambda_{\ell_1} = 3$ on Gemma-4-26B gives $L_0 = 54$ after ten epochs
and $L_0 = 81$ after twenty. The measured $L_0$, not $\lambda_{\ell_1}$,
is therefore the quantity on which SAE sparsity is comparable, which is
also why $\lambda_{\ell_1}$ is not shared across models: the two largest
models need a much larger penalty to reach a comparable density because
their residual-stream activations have a larger RMS
(Appendix~\ref{app:sae-training}).

\begin{table}[h]
\centering
\small
\caption{\textsc{FreqMask-KM} AUROC as a function of the SAE sparsity
penalty $\lambda_{\ell_1}$. Only the SAE checkpoint changes; all other
hyperparameters are held at each model's published setting. $L_0$ is
the measured average number of active SAE features per input, out of
$16384$. Epochs are listed because $\lambda_{\ell_1}$ and training
length jointly determine the density reached. AUROC is the
three-category mean of \Cref{tab:main-results}. Bold marks the encoder
used in the main results.}
\label{tab:freq-abl-sae-lambda}
\setlength{\tabcolsep}{6pt}
\begin{tabular}{lcccc}
\toprule
Model & $\lambda_{\ell_1}$ & Epochs & $L_0$ & AUROC \\
\midrule
LLaMA3-8B  & $\mathbf{0.01}$ & $\mathbf{5}$  & $\mathbf{66}$ & $\mathbf{0.936}$ \\
           & $1$    & $10$ & $58$  & $0.861$ \\
           & $3$    & $10$ & $22$  & $0.858$ \\
           & $10$   & $10$ & $9$   & $0.910$ \\
\midrule
Qwen2-1.5B & $\mathbf{0.01}$ & $\mathbf{10}$ & $\mathbf{227}$ & $\mathbf{0.762}$ \\
           & $0.1$  & $10$ & $213$ & $0.758$ \\
           & $1$    & $10$ & $153$ & $0.719$ \\
           & $3$    & $10$ & $91$  & $0.643$ \\
           & $10$   & $10$ & $53$  & $0.696$ \\
\midrule
GPT-oss-20B & $0.01$ & $20$ & $393$ & $0.846$ \\
            & $300$  & $20$ & $103$ & $0.832$ \\
            & $\mathbf{700}$ & $\mathbf{20}$ & $\mathbf{48}$ & $\mathbf{0.900}$ \\
\midrule
Gemma-4-26B & $1$    & $10$ & $55$  & $0.823$ \\
            & $1$    & $20$ & $103$ & $0.665$ \\
            & $\mathbf{3}$ & $\mathbf{10}$ & $\mathbf{54}$ & $\mathbf{0.856}$ \\
            & $3$    & $20$ & $81$  & $0.627$ \\
            & $10$   & $10$ & $31$  & $0.835$ \\
\bottomrule
\end{tabular}
\end{table}

\section{Ablations for \textsc{LearnedMask-LoRA}}
\label{app:learnedmask-lora-ablations}

This section reports three ablations on \textsc{LearnedMask-LoRA}
(\Cref{sec:framework-instantiations},
Appendix~\ref{app:hyperparameters}) targeting choices that are not pinned
down by the framework definition: whether the per-cluster LoRA stage
is necessary (Appendix~\ref{app:abl-lora}), whether Mahalanobis scoring on
the LoRA-adapted reconstruction outperforms the $\ell_2$ whitened
residual used in the main results (Appendix~\ref{app:abl-scoring}), and
whether the binary straight-through mask in Stage~1 can be replaced
by a continuous mask without loss (Appendix~\ref{app:abl-mask-type}).

Unless otherwise noted, all runs use the configuration of
Appendix~\ref{app:hyperparameters}: cosine binarized clustering with
$\ncluster = 50$ clusters for Ministral 8B and LLaMA3 8B and
$\ncluster = 100$ for the other four models, whitened residual
normalization with shrinkage $\alpha = 0.3$, Stage~1 logit
initialization at $0.5$, LoRA rank schedule $r \in \{2, 4, 8\}$ keyed
to cluster size, $30$ LoRA epochs at learning rate $10^{-3}$, and LoRA
sparsity weight $10^{-3}$. The minimum cluster size for LoRA Stage~2 is
$20$ for Qwen 1.5B and Ministral 8B, $500$ for LLaMA3 8B, and $50$
for Qwen3 8B and Gemma 4 26B. Performance is reported as mean AUROC across the three
evaluation slices (BeaverTails, ToxiGen, pooled HarmBench adversarial
attacks).

\paragraph{Relation to \Cref{tab:main-results}.} The Stage~1$+$Stage~2
row of \Cref{tab:abl-lora} matches the main table for every model
except Qwen3 8B, where it reads $0.914$ against $0.917$. The ablation
applies the rank schedule above, whereas the Qwen3 main-table run used
a flat $r = 8$. The gap is $0.003$ AUROC and does not change the
conclusion that Stage~2 helps on Qwen3: $+0.046$ with the schedule,
$+0.049$ without.

\subsection{LoRA Stage Contribution}
\label{app:abl-lora}

\paragraph{Setup.} We compare the Stage~1 binary-mask checkpoint
against the full two-stage pipeline that adds a per-cluster LoRA
adapter on top. Both rows reuse the same Stage~1 checkpoints and
clustering; the only difference is whether Stage~2 is applied. Five
seeds per row.

\paragraph{Results.} \Cref{tab:abl-lora} reports mean and standard
deviation of 3-slice AUROC across $5$ seeds. The LoRA stage helps on
four of the five $\geq$8B models---$+0.062$ AUROC on Ministral 8B,
$+0.096$ on LLaMA3 8B, $+0.046$ on Qwen3 8B, and $+0.051$ on Gemma 4
26B---but degrades two models: Qwen 1.5B ($-0.033$) and GPT-oss 20B
($-0.101$). GPT-oss thus joins Qwen 1.5B as a model on which the Stage~2
adapter hurts, so we report it Stage~1-only in the main results and
disable Stage~2 on both.

\begin{table}[h]
\centering
\small
\caption{Effect of the per-cluster LoRA stage on
\textsc{LearnedMask-LoRA}. AUROC is the 3-slice mean (BeaverTails,
ToxiGen, pooled HarmBench), reported as mean $\pm$ standard
deviation across $5$ random seeds.}
\label{tab:abl-lora}
\setlength{\tabcolsep}{4pt}
\resizebox{\textwidth}{!}{%
\begin{tabular}{lcccccc}
\toprule
Method & Qwen 1.5B & Ministral 8B & LLaMA3 8B & Qwen3 8B & GPT-oss 20B & Gemma 4 26B \\
\midrule
\textsc{LearnedMask} (Stage~1 only)        & $0.794 \pm 0.001$ & $0.834 \pm 0.001$ & $0.816 \pm 0.001$ & $0.868 \pm 0.000$ & $0.919 \pm 0.001$ & $0.853 \pm 0.000$ \\
\textsc{LearnedMask + LoRA} (Stage~1 + Stage~2) & $0.761 \pm 0.001$ & $0.896 \pm 0.001$ & $0.912 \pm 0.001$ & $0.914 \pm 0.001$ & $0.818 \pm 0.002$ & $0.904 \pm 0.000$ \\
\midrule
$\Delta$ from LoRA & $-0.033$ & $+0.062$ & $+0.096$ & $+0.046$ & $-0.101$ & $+0.051$ \\
\bottomrule
\end{tabular}%
}
\end{table}

\subsection{Scoring Function: Mahalanobis vs.\ Whitened $\ell_2$}
\label{app:abl-scoring}

\paragraph{Setup.} The main \textsc{LearnedMask-LoRA} results in
\Cref{tab:decomposition} use the $\ell_2$ norm of the whitened (masked)
SAE-reconstruction residual under the assigned cluster's adapted
encoder/decoder pair. We test whether replacing this with a per-cluster
Mahalanobis distance on the same residual improves performance. The
Mahalanobis variant fits a per-cluster covariance on the safe residuals
and adds a regularizer $\lambda I$ before inversion; we sweep
$\lambda \in \{10^{-4}, 10^{-2}, 10^{-1}, 1.0\}$, single seed per cell,
on Stage~1 checkpoints with hyperparameters matched to the multi-seed
configuration of Appendix~\ref{app:abl-lora}.

\paragraph{Results.} \Cref{tab:abl-scoring} reports 3-slice mean AUROC
for the full $\lambda$ sweep, the best Mahalanobis cell per model, and
the whitened $\ell_2$ reference from \Cref{tab:abl-lora}. Best $\lambda$
is large for the $\geq$8B models ($\lambda = 1.0$) and intermediate for
Qwen 1.5B ($\lambda = 10^{-1}$), with the exception of Gemma, whose best
$\lambda$ is the smallest ($10^{-4}$); at small $\lambda$ ($10^{-4}$) the
per-cluster covariance is otherwise ill-conditioned and AUROC collapses. At its
best $\lambda$, Mahalanobis reaches $0.793/0.890/0.904/0.879/0.577/0.785$ on
the six models, against the whitened $\ell_2$ reference of
$0.794/0.896/0.912/0.914/0.818/0.904$. The two scoring rules are within
$0.008$ AUROC on the original three models, but the gap widens on the
newer models ($0.035$ on Qwen3, $0.119$ on Gemma, and $0.241$ on GPT-oss);
on no model does Mahalanobis exceed the whitened $\ell_2$ reference. We retain
whitened $\ell_2$ as the default, since it requires no per-cluster
covariance fit and no regularization sweep.

\begin{table}[h]
\centering
\small
\caption{Scoring function comparison on \textsc{LearnedMask-LoRA}.
Mahalanobis scoring on the per-cluster reconstruction residual is
swept over the regularizer $\lambda$ ($1$ seed per cell); the
``best Mahal.''\ column reports the best $\lambda$ per model. The
``whitened $\ell_2$'' column is the multi-seed reference from
\Cref{tab:abl-lora}. AUROC is the 3-slice mean; bold marks the best
$\lambda$ within each model.}
\label{tab:abl-scoring}
\setlength{\tabcolsep}{4pt}
\begin{tabular}{lcccccc}
\toprule
 & \multicolumn{4}{c}{Mahalanobis, $\lambda$} & best Mahal. & whitened $\ell_2$ \\
\cmidrule(lr){2-5}
Model & $10^{-4}$ & $10^{-2}$ & $10^{-1}$ & $1.0$ & (selected $\lambda$) & (reference) \\
\midrule
Qwen 1.5B    & $0.602$ & $0.738$ & $\mathbf{0.793}$ & $0.786$ & $0.793$ ($\lambda{=}10^{-1}$) & $0.794$ \\
Ministral 8B & $0.426$ & $0.653$ & $0.885$ & $\mathbf{0.890}$ & $0.890$ ($\lambda{=}1.0$) & $0.896$ \\
LLaMA3 8B    & $0.576$ & $0.785$ & $0.897$ & $\mathbf{0.904}$ & $0.904$ ($\lambda{=}1.0$) & $0.912$ \\
Qwen3 8B     & $0.691$ & $0.703$ & $0.710$ & $\mathbf{0.879}$ & $0.879$ ($\lambda{=}1.0$) & $0.914$ \\
GPT-oss 20B  & $0.544$ & $0.546$ & $0.552$ & $\mathbf{0.577}$ & $0.577$ ($\lambda{=}1.0$) & $0.818$ \\
Gemma 4 26B  & $\mathbf{0.785}$ & $0.782$ & $0.776$ & $0.749$ & $0.785$ ($\lambda{=}10^{-4}$) & $0.904$ \\
\bottomrule
\end{tabular}
\end{table}

\subsection{Mask Parameterization: Continuous vs.\ Binary STE}
\label{app:abl-mask-type}

\paragraph{Setup.} Stage~1 of \textsc{LearnedMask-LoRA} learns a
binary mask $m_c \in \{0,1\}^{d_2}$ via a straight-through estimator
(Appendix~\ref{app:hyperparameters}). We replace this with a continuous mask
$m_c \in [0,1]^{d_2}$ trained against the same Stage~1 reconstruction
objective, then apply the standard Stage~2 LoRA on top. Hyperparameters
match Appendix~\ref{app:abl-lora} otherwise; single seed.

\paragraph{Results.} \Cref{tab:abl-mask-type} reports 3-slice mean
AUROC and the continuous-variant FPR@$95$. On all five $\geq$8B models
the two parameterizations are within $0.007$ AUROC of each other
($+0.004$ on Ministral, $-0.007$ on LLaMA3, $-0.003$ on Qwen3,
$+0.005$ on GPT-oss, $-0.001$ on Gemma); the continuous-mask row is a single seed, so we
cannot resolve whether these gaps reflect a real difference or seed
variation, and we treat the two parameterizations as indistinguishable
at this granularity. On Qwen 1.5B the continuous variant underperforms
the Stage~1 binary baseline by $0.046$ AUROC. Note that for GPT-oss
both rows sit near $0.82$ rather than near its main-table AUROC: this
ablation holds the Stage~2 LoRA adapter \emph{on} while varying the
mask type, and LoRA degrades GPT-oss (\Cref{tab:abl-lora}), so both
mask types fall below the Stage~1-only configuration GPT-oss actually
deploys ($0.919$, \Cref{tab:main-results}); the within-row
binary-vs-continuous gap ($+0.005$) is unaffected. Binary STE is retained
as the default since it does not meaningfully underperform the
continuous variant on any model and yields a sparser mask.

\begin{table}[h]
\centering
\small
\caption{Mask parameterization in Stage~1 of
\textsc{LearnedMask-LoRA}, with Stage~2 LoRA applied on both rows.
``Binary STE + LoRA'' reproduces the multi-seed reference from
\Cref{tab:abl-lora} (mean $\pm$ std over $5$ seeds). ``Continuous + LoRA'' is a single seed.}
\label{tab:abl-mask-type}
\setlength{\tabcolsep}{4pt}
\resizebox{\textwidth}{!}{%
\begin{tabular}{lcccccc}
\toprule
Method & Qwen 1.5B & Ministral 8B & LLaMA3 8B & Qwen3 8B & GPT-oss 20B & Gemma 4 26B \\
\midrule
\multicolumn{7}{l}{\textit{3-slice mean AUROC}} \\
\midrule
Binary STE + LoRA (multi-seed)  & $0.794 \pm 0.001$ & $0.896 \pm 0.001$ & $0.912 \pm 0.001$ & $0.914 \pm 0.001$ & $0.818 \pm 0.002$ & $0.904 \pm 0.000$ \\
Continuous + LoRA ($1$ seed)    & $0.764$            & $0.900$           & $0.905$           & $0.911$           & $0.823$           & $0.903$ \\
$\Delta$                        & $-0.046$          & $+0.004$          & $-0.007$          & $-0.003$          & $+0.005$          & $-0.001$ \\
Continuous + LoRA FPR@$95$ ($\downarrow$) & $0.594$ & $0.253$ & $0.250$ & $0.433$ & $0.456$ & $0.229$ \\
\bottomrule
\end{tabular}%
}
\end{table}

\paragraph{Summary.} The LoRA stage improves the performance of
\textsc{LearnedMask-LoRA} on Ministral, LLaMA3, Qwen3, and Gemma, and is
correctly disabled on the two models it hurts (Qwen 1.5B and
GPT-oss). The whitened-$\ell_2$ scoring rule used in the main results
is not improved upon by per-cluster Mahalanobis at any swept
regularizer on any model. The binary STE Stage~1 mask is retained over
a continuous-mask variant since the binary method yields a sparser
mask and does not perform worse than a continuous one.

\section{Adaptive Attacks Against the Deployed Detector}
\label{app:adaptive-attack}

A static detector exposes a single fixed decision rule, a score and one
operating threshold, and a white-box attacker who optimizes against
that rule can almost always find an input that stays on the safe side
of the threshold yet remains harmful. Detection-based defenses for LLMs
have repeatedly been broken this way
\citep{andriushchenko2024jailbreaking, nasr2025attacker}, and we
expected the same here. We tested it directly: with
\textsc{FreqMask-KM} used as an adversarial defense detector at its deployment
threshold, adaptive attack \citep{andriushchenko2024jailbreaking} succeeds on $98\%$ of $50$
AdvBench behaviors on LLaMA3-8B and $94\%$ on Qwen2-1.5B, and the
detector flags none of them. That attacker optimizes only the model's
refusal signal and never the anomaly score, so a detector-aware
attacker would be at least as strong. We therefore do not claim
robustness to detector-aware adaptive attacks; what we claim is
unsupervised detection of new, non-adaptive unsafe inputs, which is a
different threat model.

\begin{table}[h]
\centering
\small
\caption{A white-box adversarial-suffix jailbreak against the guarded
model, with \textsc{FreqMask-KM} detecting harmful generation at its deployment
threshold.}
\label{tab:adaptive-attack}
\setlength{\tabcolsep}{8pt}
\begin{tabular}{lcc}
\toprule
Model & Attack success rate & Detector block rate \\
\midrule
LLaMA3-8B  & $98\%$ ($49/50$) & $0\%$ ($0/50$) \\
Qwen2-1.5B & $94\%$ ($47/50$) & $0\%$ ($0/50$) \\
\bottomrule
\end{tabular}
\end{table}

The outcome is a property of static
defenses rather than of this detector in particular: the defense is
fixed first and the attacker moves second. It also does not contradict
the detection results of \Cref{sec:exp-main}. AUROC measures how well
the score ranks unsafe inputs above safe ones on a fixed evaluation
distribution, whereas the attack here moves the input distribution
until individual points fall below one chosen threshold, so a detector
can rank well and still be evaded at a fixed operating point.
Designing a defense robust to a sufficiently strong white-box adaptive
attacker remains an open problem and is beyond the scope of this
paper.

\section{Proofs for the Theory Section}
\label{ss:Pf}

We use the notation of \Cref{sec:theory}. Throughout, the representation
$z$, the population distribution $P_0$, the selected population centers,
and $K,k$ are fixed before drawing $\mathcal D_n\sim P_0^n$. An SAE
fitted on an independent pilot sample is allowed by conditioning on that
sample. Merely freezing an SAE after fitting it on $\mathcal D_n$ does
not give this independence. The results below do not analyze that
additional reuse of data or adaptive hyperparameter selection.

All probabilities conditional on $\mathcal D_n$ concern a fresh test
point. After a permutation of the empirical center labels, we transport
the assignment labels by the same permutation; we do not recompute a
different tie-breaking rule. Empty oracle clusters are assigned mean
zero by convention and are excluded on the count event used below.

\subsection{Population centers and oracle sample means}
\label{app:theory-centroids}

We first justify the conditional-mean identity used in the main text.
Let $V_c=\{z:c^\star(z)=c\}$. For any $u\in\R^{d_2}$,
\[
 \E[\|Z-u\|_2^2\mid Z\in V_c]
 =\E[\|Z-\E[Z\mid Z\in V_c]\|_2^2\mid Z\in V_c]
  +\|u-\E[Z\mid Z\in V_c]\|_2^2.
\]
If a center with $\pi_c>0$ differed from this conditional mean, replacing
it would strictly reduce the objective with the old assignments held
fixed. Reassigning to the nearest new centers cannot increase that
objective. This contradicts population optimality and proves
$\mu_c^\star=\E[Z\mid c^\star(Z)=c]$.

\begin{lemma}\label{Lem:cluster-counts}
If $n\pi_{\min}\geq8\log(4K/\delta)$, then with probability at least
$1-\delta/4$,
\begin{equation}
\label{eq:cluster-count-event}
 N_c\geq n\pi_c/2\geq n\pi_{\min}/2\qquad(c\in[K]).
\end{equation}
\end{lemma}
\begin{proof}
For each $c$, $N_c\sim\operatorname{Bin}(n,\pi_c)$. The multiplicative
Chernoff bound gives
$\Pr(N_c<n\pi_c/2)\leq\exp(-n\pi_c/8)$.
Summing over $K$ clusters proves the claim. For completeness, this
Chernoff bound follows by applying Markov's inequality to $e^{-tN_c}$,
using $\E e^{-tN_c}\leq\exp(n\pi_c(e^{-t}-1))$, and taking
$t=\log 2$: the exponent is
$-n\pi_c(1-\log 2)/2\leq-n\pi_c/8$.
\end{proof}

\begin{lemma}\label{Lem:masked-centroid}
Under \Cref{Ass:local-safe-geometry} and the count condition of
\Cref{Lem:cluster-counts}, with probability at least $1-\delta/2$,
\begin{align}
\label{eq:masked-centroid-concentration}
 \max_c\|M^\star(\overline\mu_c-\mu_c^\star)\|_2
 &\leq\sqrt{\frac{2}{n\pi_{\min}}}
       \left(v+B\sqrt{2\log\frac{4K}{\delta}}\right),\\
\label{eq:masked-centroid-l1}
 \max_c\|M^\star(\overline\mu_c-\mu_c^\star)\|_1
 &\leq\sqrt{\frac{2s_{\max}}{n\pi_{\min}}}
       \left(v+B\sqrt{2\log\frac{4K}{\delta}}\right).
\end{align}
\end{lemma}
\begin{proof}
Fix $c$ and condition on the indicators
$(\mathbf 1\{c^\star(Z_i)=c\})_{i=1}^n$, with $N_c=r\geq1$.
The observations in $I_c$ are then independent with distribution
$P_0(\,\cdot\mid c^\star(Z)=c)$. Index them as $Z_{c,1},\ldots,Z_{c,r}$
and set $Y_j=M^\star(Z_{c,j}-\mu_c^\star)$. They satisfy
\[
 \E Y_j=0,\qquad \E\|Y_j\|_2^2\leq v^2,\qquad
 \|Y_j\|_2\leq B\quad\text{almost surely}.
\]
For $F=\|r^{-1}\sum_jY_j\|_2$, independence and centering give
\[
 \E F\leq(\E F^2)^{1/2}
 =\left(\frac1{r^2}\sum_j\E\|Y_j\|_2^2\right)^{1/2}
 \leq\frac v{\sqrt r}.
\]
Replacing one $Y_j$ changes $F$ by at most $2B/r$. The bounded-differences (Mcdiarmid's)
inequality therefore implies
\[
 \Pr\left(F>\frac v{\sqrt r}+t
       \ \middle|\ \text{membership indicators}\right)
 \leq\exp\left(-\frac{rt^2}{2B^2}\right).
\]
Take $t=B\sqrt{2\log(4K/\delta)/r}$. The bound holds for every
membership pattern with $r\geq1$, so it may be averaged over these
patterns without conditioning on any empirical clustering event.
A union bound over $c$ and \Cref{Lem:cluster-counts} cost at most
$\delta/4+\delta/4$. On their intersection, $r\geq n\pi_{\min}/2$,
which proves \eqref{eq:masked-centroid-concentration}.

If $j\in S^\star\setminus J_c$, then $Z_j=(\mu_c^\star)_j$ almost
surely in cluster $c$. Hence $M^\star(\overline\mu_c-\mu_c^\star)$ is
supported on $J_c$, simultaneously for every nonempty oracle cluster,
with probability one. Cauchy--Schwarz gives
\[
 \|M^\star(\overline\mu_c-\mu_c^\star)\|_1
 \leq\sqrt{|J_c|}\,
       \|M^\star(\overline\mu_c-\mu_c^\star)\|_2,
\]
and proves \eqref{eq:masked-centroid-l1}.
\end{proof}

\subsection{The global frequency mask and the main theorem}
\label{app:theory-mask-recovery}

\begin{lemma}\label{Lem:mask-recovery}
Let $\widehat p_j=n^{-1}\sum_i\mathbf 1\{(Z_i)_j>0\}$ and suppose
$\gamma=p_{(k)}-p_{(k+1)}>0$. If
$n\geq8\gamma^{-2}\log(4d_2/\delta)$, then
$\widehat M=M^\star$ with probability at least $1-\delta/2$.
\end{lemma}
\begin{proof}
Hoeffding's inequality  and a union bound give
\[
 \Pr\left(\max_{j\in[d_2]}|\widehat p_j-p_j|>\gamma/4\right)
 \leq2d_2e^{-n\gamma^2/8}\leq\delta/2.
\]
On the complementary event, for every $j\in S^\star$ and
$\ell\notin S^\star$,
$\widehat p_j-\widehat p_\ell\geq p_j-p_\ell-\gamma/2\geq\gamma/2>0$.
Thus the selected supports agree. Independence between different
coordinates is unnecessary; independence across observations suffices.
\end{proof}

\begin{proof}[Proof of \Cref{Thm:One}]
Intersect the events in \Cref{Lem:masked-centroid,Lem:mask-recovery}.
Their joint probability is at least $1-\delta$. No independence between
these events or the fitted $K$-means  solution is required.
For every $c$, the definition of $\rho_n$ gives
\begin{align}
\label{eq:empirical-population-centroid}
 \|M^\star(\widehat\mu_c-\mu_c^\star)\|_1
 &\leq\|M^\star(\widehat\mu_c-\overline\mu_c)\|_1
       +\|M^\star(\overline\mu_c-\mu_c^\star)\|_1
 \leq\varepsilon_n.
\end{align}
Consequently, for every $z$ whose empirical and population assignments
agree, the reverse triangle inequality implies
\begin{equation}
\label{eq:pointwise-score-control}
 |\widehat s_n(z)-s^\star(z)|
 \leq\|M^\star(\widehat\mu_{c^\star(z)}-\mu_{c^\star(z)}^\star)\|_1
 \leq\varepsilon_n.
\end{equation}
This statement is pointwise on the whole representation space, not
only on the safe support. The use of local sparsity concerned the
\emph{centroid estimation error}, not the support of the test point.

Write $A_n=\{z:\widehat c_n(z)\neq c^\star(z)\}$.
Outside $A_n$, disagreement between the two strict-threshold detectors
implies $|s^\star(z)-\tau_q|\leq\varepsilon_n$. Conditional on the
training sample, $\varepsilon_n$ is fixed and therefore
\[
 P_0(\widehat f_{n,q}\neq f_q^\star\mid\mathcal D_n)
 \leq P_0(A_n\mid\mathcal D_n)
       +P_0(|s^\star(Z)-\tau_q|\leq\varepsilon_n)
 \leq\eta_{0,n}+L_0\varepsilon_n.
\]
This proves the theorem. The last step remains valid for random
$\rho_n$ because anti-concentration holds for every $t\geq0$.
\end{proof}

\paragraph{Meaning of the rate.}
The theorem is an oracle inequality with explicit clustering remainders.
It applies to any measurable fitted centers and assignments. To conclude
consistency, one must separately show or assume that $\rho_n$ and
$\eta_{0,n}$ vanish. In a sequence of problems, $\gamma$, $\pi_{\min}$,
$v$, $B$, and $L_0$ must also be tracked. In particular, a small frequency
gap can make mask identification expensive even when $\log d_2$ is small.
For balanced clusters and bounded $L_0v,L_0B$, obtaining sampling error
at most $e$ requires $n$ of order $Ks_{\max}/e^2$, up to logarithms,
in addition to mask recovery and clustering control.
Multiplying all scores by a constant changes $L_0$ inversely, so score
normalization alone cannot improve this rate.

\subsection{Two elementary controls on the clustering remainders}
\label{app:theory-clustering}

These controls clarify the terms left explicit by the main theorem;
they do not claim that local sparsity alone guarantees $K$-means recovery.
Let
\[
 \eta_{\mathrm{tr},n}
 =\frac1n\sum_i\mathbf 1\{\widehat c_n(Z_i)\neq c^\star(Z_i)\}.
\]

\begin{lemma}\label{Lem:clustering-perturbation}
Suppose each fitted center is the mean of the points assigned to it,
$\max_i\|M^\star Z_i\|_1\leq H$, and the count event
\eqref{eq:cluster-count-event} holds. If
$\eta_{\mathrm{tr},n}\leq\pi_{\min}/4$, then
\begin{equation}\label{eq:rho-from-labels}
 \rho_n\leq\frac{8H\eta_{\mathrm{tr},n}}{\pi_{\min}}.
\end{equation}
Moreover, if $a_n=\max_c\|\widehat\mu_c-\mu_c^\star\|_2$, then for
$y\in\{0,1\}$,
\begin{equation}\label{eq:assignment-margin}
 P_y(\widehat c_n(Z)\neq c^\star(Z)\mid\mathcal D_n)
 \leq P_y(D_2(Z)-D_1(Z)\leq2a_n),
\end{equation}
where $D_1,D_2$ are the smallest and second-smallest population-center
distances. For $K=1$, the assignment error is zero.
\end{lemma}
\begin{proof}
Put $\widehat I_c=\{i:\widehat c_n(Z_i)=c\}$.
For each $c$, $|I_c\mathbin{\triangle}\widehat I_c|
\leq n\eta_{\mathrm{tr},n}$, and hence
$|\widehat I_c|\geq n\pi_{\min}/4$.
For two index sets $A,B$ of sizes $a,b>0$ and vectors $u_i$ of norm at
most $H$, subtraction of their means gives
\[
 \left\|\frac1b\sum_{i\in B}u_i-\frac1a\sum_{i\in A}u_i\right\|_1
 \leq\frac{H|A\mathbin{\triangle}B|}{b}
       +\frac{H|a-b|}{b}
 \leq\frac{2H|A\mathbin{\triangle}B|}{b}.
\]
Apply this with $u_i=M^\star Z_i$ to obtain \eqref{eq:rho-from-labels}.
For \eqref{eq:assignment-margin}, each center distance changes by at
most $a_n$. A gap greater than $2a_n$ therefore preserves the unique
nearest center. Taking probabilities proves the claim.
\end{proof}

The quantity $a_n$ is a \emph{full-space} center error. The masked error
in \Cref{Thm:One} does not control it. Thus a dimension dependence in
full-space clustering cannot be removed merely by masking the score.

\subsection{Masking: estimation complexity and population separation}
\label{app:theory-masking}

\begin{corollary}\label{Cor:fixed-mask-complexity}
Let $M$ be a deterministic binary coordinate mask, fixed before drawing
$\mathcal D_n$, and keep the same full-space centers and assignments.
Define
\begin{align*}
 J_c(M)&=\{j:M_{jj}=1,\ \operatorname{Var}(Z_j\mid c^\star(Z)=c)>0\},
 \qquad s(M)=\max_c|J_c(M)|,\\
 v(M)^2&=\max_c\E[\|M(Z-\mu_c^\star)\|_2^2\mid c^\star(Z)=c],\\
 B(M)&=\max_c\operatorname*{ess\,sup}_{Z\mid c^\star(Z)=c}
                         \|M(Z-\mu_c^\star)\|_2,
\end{align*}
where $B(M)<\infty$, and put
$\rho_n(M)=\max_c\|M(\widehat\mu_c-\overline\mu_c)\|_1$.
Under the count condition, with probability at least $1-\delta/2$,
\begin{equation}\label{eq:fixed-mask-envelope}
 \max_c\|M(\widehat\mu_c-\mu_c^\star)\|_1
 \leq\varepsilon_n(M):=\rho_n(M)
 +\sqrt{\frac{2s(M)}{n\pi_{\min}}}
       \left(v(M)+B(M)\sqrt{2\log\frac{4K}{\delta}}\right).
\end{equation}
Let $s_M^\star,\widehat s_{n,M}$ be the corresponding scores and let
$\tau_{M,q}$ be the population safe quantile. If
$P_0(|s_M^\star(Z)-\tau_{M,q}|\leq t)\leq L_Mt$, their detector
disagreement is at most $\eta_{0,n}+L_M\varepsilon_n(M)$ on this event.
\end{corollary}
\begin{proof}
Apply the proof of \Cref{Lem:masked-centroid} to $M$ and use
$\|u\|_1\leq\sqrt{s(M)}\|u\|_2$ on each $J_c(M)$.
The triangle inequality adds $\rho_n(M)$. The decision bound follows
from \eqref{eq:pointwise-score-control} and the stated margin condition
for this particular score. No frequency gap or mask recovery is needed
because the mask is fixed in advance.
\end{proof}

If $M_1$ retains a subset of the coordinates retained by $M_2$, then
$s(M_1)\leq s(M_2)$, $v(M_1)\leq v(M_2)$, $B(M_1)\leq B(M_2)$, and
$\rho_n(M_1)\leq\rho_n(M_2)$. Thus the envelope in
\eqref{eq:fixed-mask-envelope} is monotone under removing coordinates,
for fixed centers and matching. This is not a classification-dominance
statement: the population score, threshold, and margin constant change
with the mask, and useful information may be discarded. The unmasked
case is $M=I_{d_2}$. The result holds for each fixed mask, not
simultaneously for all data-selected masks without an additional argument.

\paragraph{What effective dimension does and does not imply.}
For $\Sigma_{c,M}=M\operatorname{Cov}(Z\mid c^\star(Z)=c)M$,
\[
 v(M)^2=\max_c\operatorname{tr}(\Sigma_{c,M}).
\]
For a nonzero positive-semidefinite matrix $\Sigma$, the PCA definition
of $d_{90}$ implies
$\operatorname{tr}(\Sigma)/\|\Sigma\|_{\mathrm{op}}
\leq d_{90}(\Sigma)/0.9$.
Indeed, the sum of the largest $d_{90}$ eigenvalues is at least $90\%$
of the trace and at most $d_{90}\|\Sigma\|_{\mathrm{op}}$.
But $d_{90}$ does not bound the absolute trace without an eigenvalue
scale, and a rank-one covariance can involve every coordinate.
Furthermore, Figure~\ref{fig:effective-dim} reports empirical quantities,
not a population concentration bound. We therefore use it as evidence
about relative variance concentration, not as a verification of every
hypothesis in \Cref{Thm:One}.

The next example separates population-score quality from finite-sample
estimation. It is a mathematical illustration, not a fitted model of
the reported activations.

\begin{proposition}[Weak sparse noise can obscure an unmasked score]
\label[proposition]{Prop:mask-nuisance}
There exist safe and unsafe distributions on $\R_+^{D+1}$ with
$1+O(\sqrt D)$ nonzero coordinates per input with high probability such
that the following holds for $K=k=1$. The global population frequency
mask retains a single coordinate and its centroid-distance score has
AUROC one. The unmasked distance to the \emph{known population safe
centroid} has AUROC tending to $1/2$ as $D\to\infty$.
\end{proposition}
\begin{proof}
Take $D\geq9$, $p=D^{-1/2}$ and $b=D^{-1/8}$. Let
$U_0\sim\mathrm{Unif}[1,2]$, $U_1\sim\mathrm{Unif}[3,4]$, and let
$B_{y,j}\sim\mathrm{Bern}(p)$ be independent of these variables and
of one another. Define, for $y\in\{0,1\}$,
\[
 Z_y=(U_y,bB_{y,1},\ldots,bB_{y,D}).
\]
The expected support size is $1+Dp=1+\sqrt D$; a binomial Chernoff
bound gives $\|Z_y\|_0\leq1+2\sqrt D$ with probability at least
$1-e^{-\sqrt D/3}$. Each nuisance activation has amplitude $b\to0$.
The first safe coordinate has activation frequency one, whereas every
other coordinate has frequency $p$. The global top-one mask therefore
selects the first coordinate, with gap $1-p$.

The population safe center is $\mu=(3/2,bp,\ldots,bp)$.
The masked scores are $A_0=|U_0-3/2|\in[0,1/2]$ and
$A_1=|U_1-3/2|\in[3/2,5/2]$, so the masked ranking is perfect.
For $T_y=\sum_jB_{y,j}\sim\operatorname{Bin}(D,p)$, the unmasked
scores are
\[
 S_y=A_y+b\{Dp+(1-2p)T_y\}.
\]
Let $r_D=\sqrt{2Dp(1-p)}$. The standardized difference
$(T_1-T_0)/r_D$ converges to $N(0,1)$. This can be seen directly from
its characteristic function:
\[
 \E\exp\!\left(it\frac{T_1-T_0}{r_D}\right)
 =\left[1+2p(1-p)\{\cos(t/r_D)-1\}\right]^D
 \longrightarrow e^{-t^2/2},
\]
using $Dp(1-p)\to\infty$ and the expansion of cosine at zero.
Since $b(1-2p)r_D\asymp D^{1/8}\to\infty$ and $A_1-A_0$ is bounded,
\[
 \frac{S_1-S_0}{b(1-2p)r_D}\ \Longrightarrow\ N(0,1).
\]
The score difference has no atom at zero because $U_1$ is continuous
and independent of the other variables. Hence
$\operatorname{AUROC}(S)=\Pr(S_1>S_0)\to1/2$.
The class label has not changed the nuisance distribution; its increasing
fluctuation alone obscures the bounded signal in the unmasked score.
\end{proof}

\subsection{Independent safe-threshold calibration}
\label{app:theory-calibration}

Appendix~\ref{app:datasets} reports a held-out ROC evaluation and states
that no separate threshold-calibration step was performed. The following
results concern a possible deployment procedure. They do not attribute
an additional split to the experiments.

\begin{proposition}\label[proposition]{Prop:safe-quantile}
Condition on $\mathcal D_n$ and let
$Z'_1,\ldots,Z'_m\sim P_0^m$ be independent safe calibration points.
Let $\widehat F_m$ be the empirical distribution function of
$\widehat s_n(Z'_i)$, and define
$\widehat\tau_q=\inf\{t:\widehat F_m(t)\geq1-q\}$.
For $u_m=\sqrt{\log(2/\delta_q)/(2m)}$, with conditional probability
at least $1-\delta_q$,
\begin{equation}\label{eq:safe-quantile-bound}
 P_0(\widehat s_n(Z)>\widehat\tau_q
       \mid\mathcal D_n,Z'_1,\ldots,Z'_m)\leq q+u_m.
\end{equation}
The assertion allows atoms and uses the strict inequality in the detector.
\end{proposition}
\begin{proof}
Let $\widehat F(t)=P_0(\widehat s_n(Z)\leq t\mid\mathcal D_n)$.
The Dvoretzky--Kiefer--Wolfowitz inequality gives
$\sup_t|\widehat F_m(t)-\widehat F(t)|\leq u_m$ with the stated
probability. Since $\widehat F_m(\widehat\tau_q)\geq1-q$,
$\widehat F(\widehat\tau_q)\geq1-q-u_m$.
Taking complements proves the result.
\end{proof}

False-positive control alone does not bound the displacement of the
threshold or the unsafe false-negative rate. The next result states the
additional quantile regularity needed for such a comparison.

\begin{proposition}\label[proposition]{Prop:estimated-quantile-disagreement}
Let $F^\star(t)=P_0(s^\star(Z)\leq t)$ and suppose, for some
$\lambda_0,r_0>0$ and all $0\leq u\leq r_0$,
\begin{equation}\label{eq:quantile-growth}
 F^\star(\tau_q+u)\geq1-q+\lambda_0u,
 \qquad F^\star(\tau_q-u)\leq1-q-\lambda_0u.
\end{equation}
On the training event of \Cref{Thm:One}, put
$a_m=2(\eta_{0,n}+u_m)/\lambda_0$ and assume $a_m\leq r_0$.
With conditional calibration probability at least $1-\delta_q$,
\begin{align}
\label{eq:quantile-displacement}
 |\widehat\tau_q-\tau_q|&\leq\varepsilon_n+a_m,\\
\label{eq:calibrated-disagreement}
 P_0(\mathbf 1\{\widehat s_n(Z)>\widehat\tau_q\}
         \neq f_q^\star(Z)\mid\mathcal D_n,Z'_1,\ldots,Z'_m)
 &\leq\eta_{0,n}+L_0(2\varepsilon_n+a_m).
\end{align}
\end{proposition}
\begin{proof}
The deterministic score comparison \eqref{eq:pointwise-score-control}
shows, for every $t$,
\[
 F^\star(t-\varepsilon_n)-\eta_{0,n}
 \leq\widehat F(t)
 \leq F^\star(t+\varepsilon_n)+\eta_{0,n}.
\]
On the DKW event, set $w=\eta_{0,n}+u_m>0$. Applying
\eqref{eq:quantile-growth} at $u=a_m$ yields
\[
 \widehat F_m(\tau_q+\varepsilon_n+a_m)\geq1-q+w,
 \qquad
 \widehat F_m(\tau_q-\varepsilon_n-a_m)\leq1-q-w.
\]
The empirical quantile lies between these two arguments, proving
\eqref{eq:quantile-displacement}. Outside changed assignments,
\[
 |(\widehat s_n(z)-\widehat\tau_q)-(s^\star(z)-\tau_q)|
 \leq\varepsilon_n+|\widehat\tau_q-\tau_q|
 \leq2\varepsilon_n+a_m.
\]
The same threshold-band argument as in \Cref{Thm:One} proves
\eqref{eq:calibrated-disagreement}.
\end{proof}

\subsection{Training and held-out cluster occupancy}
\label{app:theory-occupancy}

\begin{lemma}\label{Lem:multinomial-l1}
For $r$ independent observations in $[K]$ with common law $p$ and
empirical law $\widehat p_r$, with probability at least $1-\delta$,
\begin{equation}\label{eq:multinomial-l1}
 \|\widehat p_r-p\|_1\leq b_r(\delta):=
 \sqrt{K/r}+\sqrt{2\log(1/\delta)/r}.
\end{equation}
\end{lemma}
\begin{proof}
Cauchy--Schwarz and the multinomial variances give
\[
 \E\|\widehat p_r-p\|_1
 \leq\sqrt{K\,\E\|\widehat p_r-p\|_2^2}
 =\sqrt{\frac K r\left(1-\sum_cp_c^2\right)}\leq\sqrt{K/r}.
\]
Changing one observation changes the $\ell_1$ deviation by at most
$2/r$. The bounded-differences inequality therefore bounds its upper
tail above the mean by $\exp(-rt^2/2)$, proving the claim.
\end{proof}

\begin{proposition}\label[proposition]{Prop:occupancy-generalization}
Let $\widehat\pi_{\mathrm{tr}}$ be the fitted-cluster histogram of
$\mathcal D_n$ and let $\widehat\pi_{\mathrm{ho}}$ be the corresponding
histogram of $m$ independent held-out safe observations. With probability
at least $1-\delta$ over both samples,
\begin{equation}\label{eq:occupancy-generalization}
 \|\widehat\pi_{\mathrm{ho}}-\widehat\pi_{\mathrm{tr}}\|_1
 \leq2\eta_{\mathrm{tr},n}+2\eta_{0,n}
       +b_n(\delta/2)+b_m(\delta/2).
\end{equation}
Here the assignment-error quantities are their actual sample-dependent
values. The result requires neither mask recovery nor small centroid error.
\end{proposition}
\begin{proof}
Let $\pi_c=P_0(c^\star(Z)=c)$, let
$\widehat\pi_{\mathrm{tr},c}^\star
=n^{-1}\sum_i\mathbf 1\{c^\star(Z_i)=c\}$, and let
$\widetilde\pi_{0,c}=P_0(\widehat c_n(Z)=c\mid\mathcal D_n)$.
Changing one assignment moves empirical mass between at most two cells,
so
\[
 \|\widehat\pi_{\mathrm{tr}}-\widehat\pi_{\mathrm{tr}}^\star\|_1
 \leq2\eta_{\mathrm{tr},n}.
\]
Coupling $c^\star(Z)$ and $\widehat c_n(Z)$ with the same fresh $Z$
gives $\|\widetilde\pi_0-\pi\|_1\leq2\eta_{0,n}$.
Since the population partition is fixed, \Cref{Lem:multinomial-l1}
applies to the oracle training histogram. Conditional on $\mathcal D_n$,
it also applies to the held-out histogram with law $\widetilde\pi_0$.
These two concentration events cost $\delta/2$ each; their intersection
satisfies
\[
 \|\widehat\pi_{\mathrm{tr}}^\star-\pi\|_1\leq b_n(\delta/2),
 \qquad
 \|\widehat\pi_{\mathrm{ho}}-\widetilde\pi_0\|_1\leq b_m(\delta/2).
\]
The triangle inequality proves the result. In particular, we have not
incorrectly treated the fitted training histogram as multinomial
conditional on a partition learned from those same observations.
\end{proof}

Similar occupancy does not imply similar partitions: distinct partitions
may have identical cell probabilities. Figure~\ref{fig:cluster-occupancy}
is therefore an empirical check of one aspect of generalization, not a
converse to \Cref{Prop:occupancy-generalization} or a measurement of
$\eta_{0,n}$.

\subsection{Optimal readout and the information retained by the detector}
\label{app:theory-information}

Condition on the fitted detector throughout this subsection. Let
$\mathcal R_\alpha(f)=\alpha P_0(f(Z)=1)+(1-\alpha)P_1(f(Z)=0)$ for
$\alpha\in(0,1)$. Define $\mathcal R_{\mathrm{Bayes}}$ as the infimum
over measurable classifiers of $Z$, $\mathcal R_{\mathrm{cs}}^\star$
as the infimum over measurable functions of
$(\widehat c_n(Z),\widehat s_n(Z))$, and
$\mathcal R_{\mathrm{cl}}^\star$ as the infimum over functions of
$\widehat c_n(Z)$ alone. These are population oracle risks, not risks
of fitted calibration models.

\begin{proposition}\label[proposition]{Prop:cluster-only-risk}
With $\widetilde\pi_{y,c}=P_y(\widehat c_n(Z)=c\mid\mathcal D_n)$,
\begin{align}
\label{eq:cluster-risk-exact}
 \mathcal R_{\mathrm{cl}}^\star
 &=\sum_c\min\{\alpha\widetilde\pi_{0,c},
                  (1-\alpha)\widetilde\pi_{1,c}\}\\
\label{eq:cluster-risk-weighted-tv}
 &=\frac{1-\|\alpha\widetilde\pi_0
                   -(1-\alpha)\widetilde\pi_1\|_1}{2}.
\end{align}
For equal priors this becomes
\begin{equation}\label{eq:cluster-risk-tv}
 \mathcal R_{\mathrm{cl}}^\star
 =\frac12(1-\operatorname{TV}(\widetilde\pi_0,\widetilde\pi_1)),
 \qquad \operatorname{TV}(p,q)=\tfrac12\|p-q\|_1.
\end{equation}
Let $\mathcal R_{\mathrm{aff}}^\star$ be the infimum over classifiers
$\mathbf 1\{a_{\widehat c_n(Z)}\widehat s_n(Z)
                  +b_{\widehat c_n(Z)}>0\}$ with arbitrary real
coefficients. Then
\begin{equation}\label{eq:cluster-score-bayes-sandwich}
 \mathcal R_{\mathrm{Bayes}}
 \leq\mathcal R_{\mathrm{cs}}^\star
 \leq\mathcal R_{\mathrm{aff}}^\star
 \leq\mathcal R_{\mathrm{cl}}^\star.
\end{equation}
In particular,
$0\leq\mathcal R_{\mathrm{cs}}^\star-\mathcal R_{\mathrm{Bayes}}
\leq\mathcal R_{\mathrm{cl}}^\star$.
\end{proposition}
\begin{proof}
Declaring cluster $c$ unsafe incurs error
$\alpha\widetilde\pi_{0,c}$, whereas declaring it safe incurs error
$(1-\alpha)\widetilde\pi_{1,c}$. Choosing the better label separately
in each cluster gives \eqref{eq:cluster-risk-exact}. The identity
$\min(a,b)=(a+b-|a-b|)/2$ and the normalization of the two occupancy
vectors give \eqref{eq:cluster-risk-weighted-tv} and
\eqref{eq:cluster-risk-tv}.

Every cluster-and-score classifier is a measurable classifier of $Z$.
Every cluster-affine classifier is a cluster-and-score classifier.
Finally, setting $a_c=0$ and $b_c\in\{-1,1\}$ realizes any constant
label in each cluster. These inclusions prove the risk ordering.
\end{proof}

For a fitted affine-threshold classifier $\widehat g$, the excess risk is exactly
\begin{align}
\label{eq:readout-decomposition}
 \mathcal R_\alpha(\widehat g)-\mathcal R_{\mathrm{Bayes}}
 ={}&[\mathcal R_\alpha(\widehat g)-\mathcal R_{\mathrm{aff}}^\star]
   +[\mathcal R_{\mathrm{aff}}^\star-\mathcal R_{\mathrm{cs}}^\star]
   +[\mathcal R_{\mathrm{cs}}^\star-\mathcal R_{\mathrm{Bayes}}].
\end{align}
The three terms are nonnegative: fitting the readout, restricting it to
cluster-affine maps, and losing information in the cluster-and-score
representation. The occupancy formula bounds the last term, not the
first two. The supervised probe is another fitted classifier, not the
Bayes rule. Similar measured AUROCs therefore provide empirical evidence
but not a bound on these classification-risk gaps.

\begin{proposition}\label[proposition]{Prop:cluster-risk-estimation}
Condition on $\mathcal D_n$. For $y\in\{0,1\}$, let
$Z_{y,1},\ldots,Z_{y,m_y}\sim P_y^{m_y}$ be independent evaluation
samples not used to fit or select the detector. Let
$\widehat\pi_y^{\,\mathrm{ev}}$ be their fitted-cluster histograms and
put
\[
 \widehat{\mathcal R}_{\mathrm{cl}}
 =\sum_c\min\{\alpha\widehat\pi_{0,c}^{\,\mathrm{ev}},
                 (1-\alpha)\widehat\pi_{1,c}^{\,\mathrm{ev}}\}.
\]
With conditional probability at least $1-\delta$,
\begin{equation}\label{eq:cluster-risk-estimation}
 |\widehat{\mathcal R}_{\mathrm{cl}}-\mathcal R_{\mathrm{cl}}^\star|
 \leq\tfrac12\{\alpha b_{m_0}(\delta/2)
                  +(1-\alpha)b_{m_1}(\delta/2)\}.
\end{equation}
\end{proposition}
\begin{proof}
Use \eqref{eq:cluster-risk-weighted-tv} for both normalized histogram
pairs. The reverse triangle inequality bounds the difference by
\[
 \tfrac12\{\alpha\|\widehat\pi_0^{\,\mathrm{ev}}-\widetilde\pi_0\|_1
 +(1-\alpha)\|\widehat\pi_1^{\,\mathrm{ev}}-\widetilde\pi_1\|_1\}.
\]
Apply \Cref{Lem:multinomial-l1} to each independent evaluation sample
with failure probability $\delta/2$.
\end{proof}

Figure~\ref{fig:cluster-occupancy} overlays each evaluation histogram
with the \emph{training} histogram. The overlap formula instead compares
safe and unsafe \emph{population} occupancies under the same fitted
partition; the preceding proposition estimates them with independent
safe and unsafe evaluations. The training overlay must not be
substituted without accounting for its data dependence.
For the benign-OOD evaluation, the same formulas apply after replacing
$P_0$ in the risk by the evaluation's mixture of in-distribution safe and
benign-OOD data. A conclusion for one negative distribution does not
automatically transfer to the other.

\subsection{Generalization on a safe--unsafe mixture}
\label{app:theory-mixture}

\begin{corollary}\label{Cor:mixture-generalization}
On the event of \Cref{Thm:One}, define
$\eta_{1,n}=P_1(\widehat c_n(Z)\neq c^\star(Z)\mid\mathcal D_n)$.
Suppose $P_1(|s^\star(Z)-\tau_q|\leq t)\leq L_1t$ for all $t\geq0$.
Then
\begin{equation}\label{eq:mixture-risk-proof}
 |\mathcal R_\alpha(\widehat f_{n,q})-\mathcal R_\alpha(f_q^\star)|
 \leq\alpha\eta_{0,n}+(1-\alpha)\eta_{1,n}
       +\{\alpha L_0+(1-\alpha)L_1\}\varepsilon_n.
\end{equation}
For the independently calibrated detector of
\Cref{Prop:estimated-quantile-disagreement}, the same bound holds with
$\varepsilon_n$ replaced by $2\varepsilon_n+a_m$ on its calibration
event.
\end{corollary}
\begin{proof}
The score comparison \eqref{eq:pointwise-score-control} holds for any
test point. The threshold-band argument therefore bounds disagreement
under $P_y$ by $\eta_{y,n}+L_y\varepsilon_n$, for $y=0,1$.
For binary classifiers $f,g$, the absolute difference in either
class-conditional error is at most $P_y(f\neq g)$. Weight by the class
priors and add. The calibrated statement follows from the last display
in the proof of \Cref{Prop:estimated-quantile-disagreement}.
\end{proof}

Consequently, total excess risk is bounded by the estimation term in
\eqref{eq:mixture-risk-proof} plus
$\mathcal R_\alpha(f_q^\star)-\mathcal R_{\mathrm{Bayes}}$.
The latter includes the choice of a global threshold and the false-positive
budget $q$; it is not solely representation loss. In fact,
\eqref{eq:safe-anticoncentration} excludes an atom at $\tau_q$, so
$P_0(s^\star(Z)>\tau_q)=q$ and
$\mathcal R_\alpha(f_q^\star)\geq\alpha q$.
For fixed $q>0$, vanishing estimation error therefore need not imply
vanishing safety-classification error.

\end{document}